\documentclass[twoside]{article}

\usepackage[preprint]{aistats2026}
\usepackage[utf8]{inputenc} 
\usepackage[T1]{fontenc}    
\usepackage[hyphens]{url}            
\usepackage{hyperref}
\usepackage[hyphenbreaks]{breakurl}
\usepackage{booktabs}       
\usepackage{amsfonts}       
\usepackage{nicefrac}       
\usepackage{microtype}      
\usepackage{xcolor}         
 \usepackage{natbib} 
 
\usepackage{amsmath,amssymb,amsfonts,amsthm}

\usepackage{bm}
\usepackage{makecell}
\usepackage{algorithm}
\usepackage{algorithmic}
\usepackage{multirow}

\usepackage{microtype}
\usepackage{graphicx}
\usepackage{subcaption}
\usepackage{booktabs}
\usepackage{cancel}

\DeclareMathOperator*{\argmax}{arg\,max}

\newtheorem{theorem}{Theorem}[section]
\newtheorem{lemma}[theorem]{Lemma}
\newtheorem{proposition}[theorem]{Proposition}
\newtheorem{corollary}[theorem]{Corollary}

\begin{document}

%

%

\twocolumn[

\aistatstitle{Multi-View Weighted Majority Vote Learning: Direct Minimization of PAC-Bayesian Bounds}


\aistatsauthor{
Mehdi Hennequin$^{1,4,\ast}$ \hspace{5pt}
Abdelkrim Zitouni$^{2,3,\ast}$ \hspace{5pt} 
Khalid Benabdeslem$^{4}$ \\
\textbf{Haytham Elghazel}$^{4}$ \hspace{5pt}
\textbf{Yacine Gaci}$^{4}$ 
\\[0.8em]
$^{1}$Omundu, Lyon, France \\
$^{2}$ESI-SBA, Algeria \\
$^{3}$Université Lyon 2, ERIC, France \\
$^{4}$Université Lyon 1, LIRIS, UMR CNRS 5205, France
\\[0.5em]
\texttt{mehdi.hennequin@omundu.ai} \\
\texttt{abdelkrim.zitouni@univ-lyon2.fr} \\
\texttt{\{khalid.benabdeslem,haytham.elghazel,yacine.gaci\}@univ-lyon1.fr}}
\vspace{0.8em}]


\def\thefootnote{*}\footnotetext{These authors contributed equally to this work. Correspondence to: Mehdi Hennequin}\def\thefootnote{\arabic{footnote}}

\begin{abstract}
    The PAC-Bayesian framework has significantly advanced the understanding of statistical learning, particularly for majority voting methods. Despite its successes, its application to multi-view learning—a setting with multiple complementary data representations—remains underexplored.  In this work, we extend PAC-Bayesian theory to multi-view learning, introducing novel generalization bounds based on Rényi divergence. These bounds provide an alternative to traditional Kullback-Leibler divergence-based counterparts, leveraging the flexibility of Rényi divergence. Furthermore, we propose first- and second-order oracle PAC-Bayesian bounds and extend the C-bound to multi-view settings. To bridge theory and practice, we design efficient self-bounding optimization algorithms that align with our theoretical results.
\end{abstract}

\section{Introduction}

Multi-view learning leverages multiple sets of features—\emph{views}—to improve the performance and robustness of learning algorithms \citep{Sun2013, XU013, ZHAO2017, Fang2023}. For instance, in image processing, combining visual, depth, or thermal data enhances object recognition accuracy. A prevalent approach in this context involves training view-specific classifiers and aggregating their predictions through a weighted majority vote, capitalizing on the complementary information across views. Although these strategies enhance predictive performance, ensuring reliable generalization across diverse views remains challenging. To address this, researchers have explored generalization bounds for multi-view learning, initially within the Probably Approximately Correct (PAC) framework \citep{Blum1998, Dasgupta2001}, and more recently using Rademacher complexity \citep{Farquhar2005, SZEDMAK2007, rosenberg07a, Sindhwani2008, Rosenberg2009, Shiliang2010, Sun2011, TIAN2021, TANG2023, Ma2024}. However, these methods can yield bounds that are either too loose or computationally intractable in practice \citep{truong2025rademachercomplexitybasedgeneralizationbounds}. In contrast, the PAC-Bayesian framework offers tighter and more adaptable generalization bounds, making it a promising alternative for practical applications in majority vote learning \citep{germain2015a,Perez2021,ZAN2021,WU2021,Viallard2011}.

While the PAC-Bayesian framework has provided tight generalization bounds for single-view majority voting, its extension to multi-view settings faces key challenges. \cite{SUN17} laid the theoretical groundwork by introducing PAC-Bayes bounds for multi-view learning through the combination of weight vectors from different views. However, their method is constrained to two views, limiting its applicability in scenarios with multiple data sources. \cite{Goyal17} addressed this by proposing a more flexible two-level PAC-Bayes approach suitable for multiple views, showing how to bound the generalization error of multi-view majority vote classifiers through hierarchical distributions over views and voters. Nevertheless, their approach lacked an explicit optimization procedure. In later work \citep{Goyal19}, they partially addressed this by optimizing the empirical multi-view C-Bound (Lemma 1, Equation 3). However, their general formulation (Theorem 2) remains theoretically sound but practically unused due to difficulties in optimizing the C-Bound while maintaining PAC-Bayesian guarantees \citep{Viallard2011}.

The theoretical framework of \cite{Goyal17} employs Kullback-Leibler (KL) divergence uniformly across all views and provides in-expectation bounds, and focuses primarily on binary classification. Meanwhile, \cite{begin16} demonstrated that Rényi divergence\footnote{A parametric family indexed by $\alpha > 0$} can provide tighter bounds than KL in single-view settings (e.g., $\alpha = 1.1$, as shown in our experiments and those of \cite{begin16}). However, the potential of Rényi divergence for multi-view hierarchical structures remains unexplored, particularly the question of whether different views might benefit from different divergence parameters when they have heterogeneous characteristics.

This paper bridges the practical gap in multi-view PAC-Bayesian learning by providing complete optimization algorithms and extending the theoretical framework with Rényi divergence. Our key insight is that the hierarchical structure of multi-view learning—with distributions over both views and voters within views—naturally accommodates view-specific regularization through Rényi divergence parameters $\alpha_v$. As $\alpha$ approaches 1 from above, the Rényi divergence recovers the familiar Kullback-Leibler divergence; at $\alpha = 2$, it connects to $\chi^2$-style divergence. This parametric flexibility allows us to experiment with different divergence measures (regularizers) within a unified framework. Unlike earlier multi-view bounds \citep{Goyal17,Goyal19}, which lack practical optimization procedures, we provide explicit algorithms that directly minimize our Rényi-based PAC-Bayesian objectives. We provide the following contributions:

\paragraph{General Multi-view PAC-Bayesian Bounds.}  
We derive new \emph{in-probability} PAC-Bayes bounds for multi-view learning using Rényi divergence with fixed, user-defined $\alpha$ as well as view-specific parameters $\alpha_v$. While \cite{Goyal17} utilized techniques from Lemma 3 in \cite{begin16} to establish \emph{in-expectation} PAC-Bayes bounds, they did not explore the specific application of Rényi divergence proposed in \cite{begin16}.

\paragraph{Extension to First/Second-Order Oracle Bounds.} 
We propose first- and second-order oracle multi-view PAC-Bayes bounds grounded in Rényi divergence. Moreover, we extend the multi-view C-Bound \citep{Goyal17,Goyal19} to incorporate Rényi divergence, and adapt the C-Tandem Oracle Bound formulation \citep{Masegosa2020} to multi-view settings.

\paragraph{Optimization Algorithms for Multi-View Learning.}  
We introduce complete optimization procedures specifically designed for multi-view learning within the PAC-Bayes framework. This includes reimplementing \emph{in-probability} versions of prior multi-view formulations \citep{Goyal17,Goyal19} and single-view methods \citep{Masegosa2020,Viallard2011} within a unified framework, enabling the first systematic empirical comparison among multi-view PAC-Bayesian bounds and with single-view baselines such as concatenated views.

\section{Multi-view PAC-Bayesian Learning}

We stand in the context of learning a weighted majority vote for multiclass classification. Consider a $d$-dimensional input space $\mathcal{X} \subseteq \mathbb{R}^d$ and a finite label space $\mathcal{Y}\subseteq N$. We assume an unknown data distribution $\mathcal{D}$ on $\mathcal{X} \times \mathcal{Y}$, with $\mathcal{D}_{\mathcal{X}}$ denoting the marginal distribution on $\mathcal{X}$. A learning sample $S = \{(\bm{x}_i, y_i)\}_{i=1}^m \sim \mathcal{D}^m$ is provided to the learning algorithm. Let $\mathcal{H}$ be a hypothesis set consisting of so-called voters $h: \mathcal{X} \rightarrow \mathcal{Y}$. The learner aims to find a weighted combination of the voters in $\mathcal{H}$, where the weights are represented by a distribution over $\mathcal{H}$. In the PAC-Bayes framework, we postulate a prior distribution $\mathcal{P}$ over $\mathcal{H}$. After observing $S$, the goal is to learn a posterior distribution $\mathcal{Q}$ over $\mathcal{H}$ used to construct a $\mathcal{Q}$-weighted majority vote classifier, $\mathcal{B}_{\mathcal{Q}}(\bm{x}) \triangleq \argmax_{y \in \mathcal{Y}} \left[\mathbb{E}_{h \sim \mathcal{Q}}\left[\mathbb{I} (h(\bm{x}) = y)\right]\right]$ (a.k.a. Bayes classifier), that minimizes the true risk $ R_{\mathcal{D}} \triangleq \mathbb{E}_{(\bm{x},y)\sim \mathcal{D}}\left[\ell(\mathcal{B}_{\mathcal{Q}}(\bm{x}),y)\right]$, with the 0-1 loss $\ell(h(\bm{x}),y) = \mathbb{I}(h(\bm{x}) \neq y)$, where $\mathbb{I}(.)$ is the indicator function. Since $\mathcal{D}$ is unknown, a common way to try to minimize the true risk is the minimization of its empirical counterpart defined as $\hat{R}_{S} \triangleq \frac{1}{m}\sum_{i=1}^m\ell(\mathcal{B}_{\mathcal{Q}}(\bm{x}_i),y_i)$. 

In multi-view learning, data instances are represented or partitioned across $V \geq 2$ different views, where each view $v \in [\![V]\!]$ (which denotes the set \(\{1, 2, \ldots, V\}\)) contains elements from $\mathcal{X}^{v} \subset \mathbb{R}^{d_v}$. The combined dimensions of all views are represented by $d = d_1 \times \cdots \times d_V$. Each view contributes to the labeled sample as $S = \{(\bm{x}_i^{v}, y_i)\}_{i=1}^m \sim \mathcal{D}$. For each view $v \in [\![V]\!]$, we consider a view-specific set $\mathcal{H}_v$ of voters
$h : \mathcal{X}^{v} \to \mathcal{Y}$, with an associated prior distribution $\mathcal{P}_v$ for each
view. Additionally, a hyper-prior distribution $\pi$ is defined over the set of views. The learner's dual objective is to optimize both the view-specific posterior distributions $\mathcal{Q}_v$ and the hyper-posterior distribution $\rho$ over the views. This strategy aims to minimize the true risk $R^{\mathcal{V}}_{\mathcal{D}}$ and its empirical counterpart $\hat{R}^{\mathcal{V}}_{S}$ of the multi-view weighted majority vote, defined as, $
\mathcal{B}_{\rho}(\bm{x}^v) \triangleq \argmax_{y \in \mathcal{Y}} \mathbb{E}_{v \sim \rho} \left[\mathbb{E}_{h \sim \mathcal{Q}_v} \left[ \mathbb{I}(h(\bm{x}^v) = y) \right]\right]$. Here, the weighted majority vote is computed by taking the expectation over both the hyper-posterior $\rho$ on the views and the posterior $\mathcal{Q}_v$ on the voters within each view.

To simplify the following sections, we introduce several abbreviations (all notation tables are provided in Appendix~\ref{App:notation}). In particular, we use $\mathbb{E}\mathbb{E}[\cdot]$ to denote $\mathbb{E}[\mathbb{E}[\cdot]]$, abbreviate $\mathbb{E}_{(\bm{x}^v,y) \sim \mathcal{D}}[\cdot]$ to $\mathbb{E}_{\mathcal{D}}[\cdot]$, $\mathbb{E}_{\bm{x} \sim \mathcal{D}_{\mathcal{X}}}[\cdot]$ to $\mathbb{E}_{\mathcal{D}_{\mathcal{X}}}[\cdot]$, represent $\mathbb{E}_{S \sim \mathcal{D}^m}[\cdot]$ by $\mathbb{E}_{S}[\cdot]$, simplify $\mathbb{E}_{v \sim \rho}$ to $\mathbb{E}_{\rho}[\cdot]$, $\mathbb{E}_{h \sim \mathcal{Q}}[\cdot]$ to $\mathbb{E}_{\mathcal{Q}}[\cdot]$, $\mathbb{E}_{(v,v') \sim \rho^2}[\cdot]$ to $\mathbb{E}_{\rho^2}[\cdot]$ and $\mathbb{E}_{(h,h') \sim \mathcal{Q}^2}[\cdot]$ to $\mathbb{E}_{\mathcal{Q}^2}[\cdot]$.

\subsection{General Multi-view PAC-Bayesian bounds}\label{General Multi-view PAC-Bayesian bounds}

The risk of $\mathcal{B}_{\mathcal{Q}}$ is known to be NP-hard \citep{Lacasse06,Redko2019}; therefore, PAC–Bayes generalization bounds do not directly focus on the risk of $\mathcal{B}_{\mathcal{Q}}$. Instead, it provides an upper bound on the expectation of the true risks of all individual hypotheses under $\mathcal{Q}$, which is known as the Gibbs risk $\mathfrak{R}_{\mathcal{D}} \triangleq \mathbb{E}_{\mathcal{D}}\mathbb{E}_{\mathcal{Q}}\left[\ell(h(\bm{x}),y)\right]$. We propose PAC-Bayesian analysis in a multi-view setting to estimate the Gibbs risk \( \mathfrak{R}_{\mathcal{D}}^{\mathcal{V}} \triangleq \mathbb{E}_{\mathcal{D}}\mathbb{E}_{\rho}\mathbb{E}_{\mathcal{Q}_{v}}\left[\ell(h(\bm{x}^{v}),y)\right]\) from the empirical Gibbs risk \( \mathfrak{\hat{R}}_{S}^{\mathcal{V}} \triangleq \frac{1}{m}\sum^{m}_{i=1}\mathbb{E}_{\rho}\mathbb{E}_{\mathcal{Q}_{v}}\left[\ell(h(\bm{x}_{i}^{v}),y_i)\right]\) building on the work of \cite{begin16}, who employed Rényi divergence for PAC-Bayesian bounds, by extending it to multi-view learning. Rényi divergence offers a broader, more adaptable measure compared to the traditionally used Kullback-Leibler divergence, thereby enhancing the flexibility of divergence measures between distributions \citep{Erven212, begin16,Viallard24}. We derive three foundational PAC-Bayesian approaches: \cite{McAllester99}, \cite{Catoni07}, and \cite{Seeger03, Langford05a}, to formulate bounds that are specifically tailored for multi-view settings using Rényi divergence. Specifically, we present the Seeger/Langford bound \citep{Seeger03,Langford05a}, known as the tightest bound \citep{Germain15a,Foong21}, in detail within the main text. Additional bounds, based on the works of Catoni and McAllester, are discussed in the Appendix~\ref{proof of General Multiview PAC-Bayesian Theorem}.
\begin{corollary}[PAC-Bayes-kl Inequality based on Rényi Divergence, in the idea of Seeger/Langford's theorem \citep{Seeger03, Langford05a}]\label{PAC-Bayes-kl Inequality based on Rényi Divergence}
Let $V \geq 2$ be the number of views. For any distribution $\mathcal{D}$ on $\mathcal{X} \times \mathcal{Y}$, for any set of prior distributions $\{\mathcal{P}_{v}\}_{v=1}^{V}$, and for any hyper-prior distribution $\pi$  over $[\![V]\!]$,  with probability at least $1-\delta$ over a random draw of
a sample $S$, we have:
\begin{align}
&\textnormal{KL}\left(\mathfrak{\hat{R}}_{S}^{\mathcal{V}}\middle\| \mathfrak{R}_{\mathcal{D}}^{\mathcal{V}}\right)
    \\  & \nonumber \leq \underbrace{\frac{\mathbb{E}_{\rho}[D_{\alpha_v}(\mathcal{Q}_{v} \| \mathcal{P}_v)] + D_\alpha(\rho \| \pi)  + \ln{\frac{2\sqrt{m}}{\delta}}}{m}}_{\psi_r}
\end{align}
\end{corollary}
The Kullback-Leibler (KL) divergence between \( Q \) and \( P \) is defined as ${\text{KL}(Q \| P) \triangleq \mathbb{E}_{h \sim Q} \left[\ln \frac{Q(h)}{P(h)}\right]}$, and the Rényi divergence as ${D_\alpha(Q \| P) \triangleq \frac{1}{\alpha - 1} \ln \left( \mathbb{E}_{h \sim P} \left[\left( \frac{Q(h)}{P(h)} \right)^\alpha \right] \right)}$ for \( \alpha > 1 \). 

Compared to the PAC-Bayes-kl inequality proposed by Seeger and Langford, this approach relies on the introduction of a hyper-prior \(\pi\) and a hyper-posterior \(\rho\) distribution over the views, leading to the additional term \(D_\alpha(\rho \| \pi)\). This term measures the deviation between the hyper-prior and the hyper-posterior distributions on \([\![V]\!]\) through the Rényi divergence. Moreover, the view-specific prior and posterior distributions contribute an additional term $\mathbb{E}_{\rho}\left[D_{\alpha_v}(\mathcal{Q}_{v} \| \mathcal{P}_v)\right]$, expressed as the expectation of the view-specific Rényi divergence over the views \([\![V]\!]\) according to the hyper-posterior distribution \(\rho\). Compared with \cite{Goyal17}’s PAC-Bayes bounds, our results provide \emph{high‑probability} deviation bounds: with probability at least \(1-\delta\) over the draw of the sample, the true risk is bounded by a term that still contains the confidence factor \(\ln(2\sqrt{m}/\delta\bigr)\). Because the inequalities are \emph{expectation} bounds, the extra outer expectation removes the confidence parameter, so their corresponding logarithmic factor reduces to \(\ln 2\sqrt{m}\) and an additional expectation operator appears instead. 
 Moreover, in subsequent work, \cite{Goyal19} derive a probabilistic bound following Catoni’s approach (see Appendix~\ref{Multi-view Bounds in Expectation} for a clearer explanation of the distinction between expectation and probabilistic bounds). Additionally, our framework enables view-specific $\alpha_v$ parameters—a key innovation absent in both \cite{Goyal17}'s KL-based approach and \cite{begin16}'s single-view Rényi bounds.
\subsection{First Order Multi-view PAC-Bayesian Bounds}\label{First Order Multi-view PAC-Bayesian Bounds}
When \( \mathcal{B}_{\mathcal{Q}}(\cdot) \) misclassifies an instance \( \bm{x}  \), it implies that at least half of the classifiers (according to the distribution \( \mathcal{Q} \)) have made an error on that instance. As a result, we can bound the true risk \( R_{\mathcal{D}} \) by twice the Gibbs risk \( \mathfrak{R}_{\mathcal{D}} \), i.e., $R_{\mathcal{D}} \leq 2 \, \mathfrak{R}_{\mathcal{D}}$. This is commonly referred to as the first-order oracle bound \citep{Germain15a, Masegosa2020}. This relationship can also be generalized to the multi-view learning framework, yielding the inequality:
\begin{theorem}[First Order Multi-view Oracle Bound \citep{Goyal17}]\label{First Order Multi-view Oracle Bound}
\begin{equation}
R_{\mathcal{D}}^{\mathcal{V}} \leq 2 \, \mathfrak{R}_{\mathcal{D}}^{\mathcal{V}}.
\end{equation}

\end{theorem} In this section, we extend the first-order multi-view oracle bound to empirical bounds by leveraging the PAC-Bayes-kl inequality, with the Rényi Divergence previously introduced. The next theorem provides a relaxation of the PAC-Bayes-kl inequality, which is more convenient for optimization. The upper bound is due to \cite{Thiemann17}, while the lower bound was proposed by \cite{Masegosa2020}. Therefore, we propose adapting Thiemann et al.'s approch to the multi-view PAC-Bayes. See the Appendix~\ref{PAC_Bayes_lambda_proof} for the proof. 
\begin{theorem}{Multi-view PAC-Bayes-$\lambda$ Inequality, in the idea of  \cite{Thiemann17}'s theorem.}\label{multi_view_thieman} Under the same assumption of Corollary~\ref{PAC-Bayes-kl Inequality based on Rényi Divergence} and for all $\lambda \in (0, 2)$ and $\gamma> 0$ we have:
\begin{align}\label{PAC-Bayes-lambda_Inequality}
    & \mathfrak{R}_{\mathcal{D}}^{\mathcal{V}} \leq \frac{\mathfrak{\hat{R}}_{ S}^{\mathcal{V}}}{1-\frac{\lambda}{2}} + \frac{\psi_r}{\lambda(1-\frac{\lambda}{2})},  \mathfrak{R}_{\mathcal{D}}^{\mathcal{V}} \geq \left(1-\frac{\gamma}{2}\right)\mathfrak{\hat{R}}_{S}^{\mathcal{V}} - \frac{\psi_r}{\gamma}.
\end{align}

\end{theorem}

We propose the following corollary to bound the Bayes risk $R_{\mathcal{D}}^{\mathcal{V}}$, utilizing the multi-view PAC-Bayes-$\lambda$ inequality presented in the above theorem. However, it's important to note that the Gibbs risk may not fully reflect the efficiency of voter combination in ensemble methods, as it overlooks the necessity to compensate for individual voter errors. This aspect is articulated through the decomposition of \(\mathfrak{R}_{\mathcal{D}}^{\mathcal{V}}\) into the expected disagreement  \({d^{\mathcal{V}}_{\mathcal{D}_{\mathcal{X}}} \triangleq \mathbb{E}_{\mathcal{D}_{\mathcal{X}}}\mathbb{E}_{\rho^2}\mathbb{E}_{{\mathcal{Q}}^{2}_{v}}\big[\ell(h(\bm{x}^v),h'(\bm{x}^{v'}))\big]}\) and the expected joint error \({e^{\mathcal{V}}_{\mathcal{D}} \triangleq \mathbb{E}_{\mathcal{D}}\mathbb{E}_{\rho^2}\mathbb{E}_{\mathcal{Q}^{2}_{v}}\big[\ell(h(\bm{x}^v),y) \times \ell(h'(\bm{x}^{v'}),y)\big]}\) (due to \cite{Lacasse06} for single view and \cite{Goyal17} for multi-view), $\mathfrak{R}_{\mathcal{D}}^{\mathcal{V}} = \frac{1}{2}d^{\mathcal{V}}_{\mathcal{D}_{\mathcal{X}}} + e^{\mathcal{V}}_{\mathcal{D}}$. We denote by \({\hat{e}^{\mathcal{V}}_{S} \triangleq \frac{1}{m}\sum^{m}_{i=1} \mathbb{E}_{\rho^2}\mathbb{E}_{\mathcal{Q}^{2}_{v}}\big[\ell(h(\bm{x}^{v}_{i}),y_{i}) \times \ell(h'(\bm{x}^{v'}_{i}),y_{i})\big]}\) and \({\hat{d}^{\mathcal{V}}_{S}\triangleq \frac{1}{m}\sum^{m}_{i=1} \mathbb{E}_{\rho^2}\mathbb{E}_{{\mathcal{Q}}^{2}_{v}}\big[\ell(h(\bm{x}^{v}_{i}),h'(\bm{x}^{v'}_{i}))\big]}\) their empirical counterparts. With this, we derive the following corollary (see Appendix~\ref{Proofs of Multi-view Oracle Bounds Inequalities} for the proof):

\begin{corollary}{First Order Multi-view Bounds with Pac-bayes-$\lambda$ Inequality}\label{Pac-bayes-kl-mv-FO}
Under the same assumption of Corollary~\ref{PAC-Bayes-kl Inequality based on Rényi Divergence} and for all $\lambda, \lambda_1, \lambda_2 \in (0, 2)$, we have: 
\begin{align}
R_{\mathcal{D}}^{\mathcal{V}} &\leq \underbrace{2\left(\frac{\hat{\mathfrak{R}}_{S}^\mathcal{V}}{1 - \frac{\lambda}{2}} + \frac{\psi_{r}}{\lambda \left(1 - \frac{\lambda}{2}\right)}\right)}_{\mathcal{R}}, \\
R_{\mathcal{D}}^{\mathcal{V}} &\leq 2\left(\frac{\hat{e}_{S}^\mathcal{V}}{1 - \frac{\lambda_1}{2}}
+ \frac{\psi_e}{\lambda_1 \left(1 - \frac{\lambda_1}{2}\right)}\right) \nonumber \\
&\quad \underbrace{+\left(\frac{\hat{d}_{S}^\mathcal{V}}{1 - \frac{\lambda_2}{2}} + \frac{\psi_d}{\lambda_2 \left(1 - \frac{\lambda_2}{2}\right)}\right)}_{\mathcal{E}},
\end{align}

$ \textnormal{with} 
\left\{\!\!
\begin{array}{rl}
\psi_{e} \!\!\!\! &= \frac{2 [\mathbb{E}_{\rho}[D_{\alpha_v}(\mathcal{Q}_{v} \| \mathcal{P}_v)] + D_\alpha(\rho \| \pi)] + \ln\Big(\frac{4\sqrt{m}}{\delta}\Big)}{m}, \\
\psi_{d} \!\!\!\! &= \frac{2 [\mathbb{E}_{\rho}[D_{\alpha_v}(\mathcal{Q}_{v} \| \mathcal{P}_v)] + D_\alpha(\rho \| \pi)] + \ln\Big(\frac{4\sqrt{n}}{\delta}\Big)}{n}.
\end{array}
\right.
$
\label{Eq-Pac-bayes-joint-dis-mv-FO}
\end{corollary}

The bounds presented above, have the advantage of controlling the trade-off between empirical risk and divergence (parametrized bounds \cite{Catoni07,Viallard24}). Although this is interesting for optimization, it represents a relaxation of the PAC-Bayes-kl inequality. Specifically, \cite{Masegosa2020} leveraged this relaxation to find an optimal posterior distribution by minimizing the bound with respect to $\lambda$, after which they substituted the result  posterior distribution back into the PAC-Bayes-kl formula for a slightly tighter bound than the original PAC-Bayes-$\lambda$ bound. Therefore, we propose to redefine the bound using the inverted KL as suggested by \cite{Gintare}. This approach reinterprets Seeger/Langford's bound by applying the inverted KL. We derive the following bound, with probability at least \(1-\delta\), for any posterior distribution \(\mathcal{Q}_v \in \mathcal{H}_v\) and any hyper-posterior distribution \( \rho \in [\![V]\!] \) (see Appendix~\ref{Proofs of Multi-view Oracle Bounds Inequalities} for the proof),

\begin{corollary}{First Order Multi-view Bounds with Inverted KL.}\label{Pac-bayes-kl-inv-mv-FO}
Under the same assumptions of Corollary~\ref{PAC-Bayes-kl Inequality based on Rényi Divergence}, we have:
    \begin{align}
        &R_{\mathcal{D}}^{\mathcal{V}} \leq \underbrace{2\, \overline{\textnormal{KL}}\left(\hat{\mathfrak{R}}_{S}^{\mathcal{V}} \middle\| \psi_{r} \right)}_{\mathcal{K}}, 
        \\ \nonumber &R_{\mathcal{D}}^{\mathcal{V}} \leq \underbrace{2\,\overline{\textnormal{KL}}\left({\hat{e}}^{\mathcal{V}}_{S}\middle\| \psi_{e}\right) + 
        \overline{\textnormal{KL}}\left({\hat{d}}^{\mathcal{V}}_{S}\middle\| \psi_{d}\right)}_{\mathcal{K}^{u}}, 
    \end{align}\label{Eq-Pac-bayes-kl-inv-FO}
\(\textnormal{where}\, \overline{\textnormal{KL}}(q \| \psi) = \max\left\{ p \in (0,1) \,\middle|\, \textnormal{KL}(q \| p) \leq \psi \right\}, \\ \underline{\textnormal{KL}}(q \| \psi) = \min\left\{ p \in (0,1) \,\middle|\, \textnormal{KL}(q \| p) \leq \psi \right\}\)
\end{corollary}

\subsection{Second Order Multi-view PAC-Bayesian Bounds}

The first order oracle bound ignores the correlation of errors, which is the main power of the majority vote. Furthermore, this bound is tight only when the Gibbs risk is low \citep{Langford02}. In order to take correlation of errors into account, \cite{Lacasse06} derived the C-Bound, which is based on the Chebyshev-Cantelli inequality. The concept was further developed by \cite{Laviolette11,Laviolette17}, \cite{Germain15a}, and extended to multi-view learning by \cite{Goyal17}. \cite{Masegosa2020} extended this idea with a second-order oracle bound, based on the second-order Markov’s inequality, positing that \(R_{\mathcal{D}} \leq 4 \, e_{\mathcal{D}}\). For multi-view, we propose the following theorem (a proof of this relation is available in the Appendix~\ref{Second Order Multi-view Oracle Bound}),
\begin{theorem}{Second Order Multi-view Oracle Bound }\label{Second Order Multi-view Oracle Bound th}
    \raggedright
    \begin{equation}
        R_{\mathcal{D}}^{\mathcal{V}} \leq 4 \, e_{\mathcal{D}}^{\mathcal{V}}
    \end{equation}
\end{theorem}

As stated in Section~\ref{First Order Multi-view PAC-Bayesian Bounds}, we propose the following corollary to bound the Bayes risk \(R_{\mathcal{D}}^{\mathcal{V}}\), utilizing the multi-view PAC-Bayes-\(\lambda\) inequality~\ref{multi_view_thieman}.
\begin{corollary}{Second Order Multi-view with PAC-Bayes-$\lambda$ Inequality, in multiclass classification.}\label{Pac-bayes-kl-mv-SO}
      Under the same assumption of Corollary~\ref{Pac-bayes-kl-mv-FO} we have:
\begin{align}\label{Eq-Pac-bayes-joint-mv-SO}
    R_{\mathcal{D}}^{\mathcal{V}} \leq  \underbrace{4\left(\frac{{\hat{e}}^{\mathcal{V}}_{S}}{1-\frac{{\lambda}}{2}} + \frac{\psi_{e}}{{\lambda}(1-\frac{{\lambda}}{2})} \right)}_{\mathcal{E}_{\text{II}}} 
\end{align}
\end{corollary}

Given the possibility of using unlabeled data, we propose the following theorem, which has the potential to provide a stricter bound when a significant amount of unlabeled data is available. This theorem uses the disagreement between voters but is restricted to the binary classification (see Appendix~\ref{Proofs of Multi-view Oracle Bounds Inequalities} for the proof).

\begin{corollary}{Multi-view PAC-Bayes-$\lambda$ Inequality, in binary classification.}\label{PAC-Bayes_lambda_binary classification} Under the same assumption of corollary~\ref{Pac-bayes-kl-mv-FO} and for all ${\gamma}> 0$ we have:\label{Eq-Pac-bayes-dis-mv-SO}
\begin{align}
    \begin{array}{rl}
         R_{\mathcal{D}}^{\mathcal{V}} \leq \underbrace{4 \left( \frac{\hat{\mathfrak{R}}_{S}^{\mathcal{V}}}{1-\frac{{\lambda}}{2}} + \frac{\psi_{r}}{{\lambda}(1-\frac{{\lambda}}{2})}\right)  - 2\left((1-\frac{{\gamma}}{2}){\hat{d}}^{\mathcal{V}}_{S} - \frac{\psi_{d}}{{\gamma}}\right)}_{\mathcal{R}_{\textnormal{II}}}, 
    \end{array}
\end{align}
\end{corollary}

Finally, we extend the two above corollaries using the inverted KL (see Appendix~\ref{Proofs of Multi-view Oracle Bounds Inequalities} for the proof),

\begin{corollary}{Second Order Multi-view with inverted KL, in multiclass classification.}\label{Pac-bayes-kl-inv-mv-SO} Under the same assumption of corollary~\ref{PAC-Bayes-kl Inequality based on Rényi Divergence}:
    \begin{equation}\label{Eq-Pac-bayes-joint-inv-mv-SO}
        R^{\mathcal{V}}_{\mathcal{D}} \leq \underbrace{4 \, \overline{\textnormal{KL}}\left({\hat{e}}^{\mathcal{V}}_{S} \middle\| \psi_{e} \right)}_{\mathcal{K}_{\textnormal{II}}}
    \end{equation}
    \end{corollary}
 \begin{corollary}{Second Order Multi-view bound with inverted KL, in binary classification.}\label{Multi-view PAC-Bayes-KL_Inequality_binary} Under the same assumption of Corollary~\ref{PAC-Bayes-kl Inequality based on Rényi Divergence}:  
    \begin{align} \label{Eq-Pac-bayes-dis-inv-mv-SO}
        & R^{\mathcal{V}}_{\mathcal{D}} \leq \underbrace{4 \, \overline{\textnormal{KL}}\left(\hat{\mathfrak{R}}_{S} \middle\| \psi_{r} \right) - 2 \, \underline{\textnormal{KL}}\left({\hat{d}}_{S}^{\mathcal{V}} \middle\| \psi_{{d}} \right)}_{\mathcal{K}^{u}_{\textnormal{II}}}.
    \end{align}
\end{corollary}

The relationship between first-order, second-order oracle bounds, and the C-Bound \cite{Lacasse06}, particularly in terms of their tightness, is complex. Theorem 2 from \cite{Viallard2011}, which draws on the work of \cite{Germain15a} and \cite{Masegosa2020}, elucidates these connections. The results of Theorem 2 from \cite{Viallard2011} also hold in a multi-view context. The multi-view C-Bound proves tighter than both first and second order terms when \( R_{\mathcal{D}}^{\mathcal{V}} \leq d_{\mathcal{D}}^{\mathcal{V}}\). However, optimization efforts have previously focused on binary classification. While the second-order approach broadened the application of the C-Bound to multiclass settings by integrating the joint error—referred to as the C-Tandem Oracle Bound, a reformulation of PAC-Bound 1 from \cite{Lacasse06}—it does so without directly minimizing the C-Bound itself. We propose, in the following section, an approach to directly optimize the multi-view C-Bound and the multi-view C-Tandem Oracle Bound.

\section{Multi-view PAC-Bayesian C-Bounds}
 
 In this section, we recall PAC-Bayesian generalization bounds on the C-Bound referred to as the \textbf{PAC-Bayesian C-Bound}. The first, is based on the \cite{Seeger03}'s approach that we adapt in multi-view as proposed by \cite{Goyal17}. The second is the C-Tandem Oracle Bound using the form proposed by \cite{Lacasse06} (PAC-bound 1). We adapt this bound to multi-view with \cite{Seeger03}'s approach.  
 
 \begin{theorem}[Multi-view PAC-Bayesian C-Bound]\label{Multi-view PAC-Bayesian C-Bound}
    \raggedright
    Under the same assumption of Corollary~\ref{Pac-bayes-kl-mv-FO}, if \(R_{\mathcal{D}}^{\mathcal{V}} < \frac{1}{2}\) and for any $\delta > 0$, we have:
    \begin{align} \label{Eq-Pac-bayes-mv-C-Bound}
         R_{\mathcal{D}}^{\mathcal{V}} &\leq \underbrace{1 - \frac{(1-2\,\mathfrak{R}_{\mathcal{D}}^{\mathcal{V}}\big)^{2}}{1-2{d}^{\mathcal{V}}_{\mathcal{D}_{\mathcal{X}}}} }_{\mathcal{C}_{\mathcal{D}}^{S}} \\ \nonumber &\leq \underbrace{1 -  \frac{\left(1-2\min\left[\frac{1}{2},\overline{\textnormal{KL}}\left(\hat{\mathfrak{R}}^{\mathcal{V}}_{S}\Big\|\psi_{{r}}\right)\right]\right)^2}{1-2\max\left[0,\underline{\textnormal{KL}}\left( {\hat{d}}_{S}\Big\|\psi_{{d}}\right)\right]}}_{\mathcal{C}_{\rho}^{S}}
    \end{align}
\end{theorem}
\begin{theorem}{Multi-view PAC-Bayesian C-Tandem Oracle Bound}\label{Multi-view PAC-Bayesian C-tandem Oracle Bound}
     Under the same assumption of Theorem~\ref{Multi-view PAC-Bayesian C-Bound}, we have:
    \begin{align} 
          R_{\mathcal{D}}^{\mathcal{V}} &\leq \underbrace{\frac{{e}^{\mathcal{V}}_{\mathcal{D}} - (\mathfrak{R}_{\mathcal{D}}^{\mathcal{V}})^{2}}{{e}^{\mathcal{V}}_{\mathcal{D}} - \mathfrak{R}_{\mathcal{D}}^{\mathcal{V}} + \frac{1}{4}} }_{\mathcal{C}_{\mathcal{D}}^{T}} \\ \nonumber &\leq \underbrace{ \frac{\overline{\textnormal{KL}}\left({\hat{e}}^{\mathcal{V}}_{S}\Big\|\psi_{{e}}\right)- \left[\underline{\textnormal{KL}}\left(\hat{\mathfrak{R}}^{\mathcal{V}}_{S}\Big\|\psi_{{r}}\right)\right]^2}{\overline{\textnormal{KL}}\left({\hat{e}}^{\mathcal{V}}_{S}\Big\|\psi_{{e}}\right) - \overline{\textnormal{KL}}\left(\hat{\mathfrak{R}}^{\mathcal{V}}_{S}\Big\|\psi_{{r}}\right) + \frac{1}{4}}}_{\mathcal{C}_{\rho}^{T}} 
    \end{align}\label{Eq-Pac-bayes-mv-C-tandem-Bound}
\end{theorem}
\section{Self-Bounding Algorithms}\label{Self-Bounding Algorithms}
\textbf{Optimization of PAC-Bayes-$\lambda$ inequality Bounds}. First, we propose an optimization approach based on the PAC-Bayes-$\lambda$ inequality. Unlike the optimization procedure proposed by \cite{Masegosa2020}, we impose constraints based on the work of \cite{Germain15a} and \cite{Viallard2011}. The bounds proposed with \cite{Thiemann17}'s approach parameterize the trade-off between empirical risk and divergence, introducing the $\lambda$ parameter. In the optimization procedure, the choice of $\lambda$ as a gradient parameter can be made in two ways: the primary method is outlined in our Algorithm~\ref{alg:compute_lambda}, with the secondary choice being to calculate it using the methods described by \cite{Thiemann17}. The entire optimization procedure is detailed in Appendix~\ref{Optimization of Inverted Multi-View KL Bounds}. 

We aim to minimize the following constrained objective functions derived from the Pac-Bayes-$\lambda$ inequality bounds from Corollaries~\ref{Pac-bayes-kl-mv-FO},~\ref{Pac-bayes-kl-mv-SO},~\ref{PAC-Bayes_lambda_binary classification}:
\begin{align}
\min_{\mathcal{Q}_v,\rho,\lambda}  \mathcal{R}  \text{ s.t. } &\left\{\left(\frac{\hat{\mathfrak{R}}_{S}^\mathcal{V}}{1-\frac{\lambda}{2}} + \frac{\psi_{r}}{\lambda \left(1-\frac{\lambda}{2}\right)}\right) \leq \frac{1}{2}\right\} \\ \nonumber &\to  \mathbf{B}_t \left(\frac{\hat{\mathfrak{R}}_{S}^\mathcal{V}}{1-\frac{\lambda}{2}} + \frac{\psi_{r}}{\lambda \left(1-\frac{\lambda}{2}\right)} - \frac{1}{2}\right)
\end{align},\label{min_R}

where $\mathbf{B}_t(a) = \begin{cases}
    -\frac{1}{t} \ln(-a), & \text{if } a \leq -\frac{1}{t^2}, \\
    t a - \frac{1}{t} \ln\left( \frac{1}{t^2} \right) + \frac{1}{t}, & \text{otherwise}.
\end{cases}$

 the log-barrier extension introduced by \cite{Kervadec17}. The log-barrier extension plays a dual role: it acts as a soft penalty function that enforces constraints indirectly by integrating them into the objective function as penalty terms. 
\begin{align}\label{min_e}
\min_{\mathcal{Q}_v,\rho,\lambda_1,\lambda_2} \mathcal{E} \;  \textnormal{s.t.} &\left\{
        \begin{array}{l}
            e \leq 1/4, \\
            d \leq 2\left(\sqrt{e}-e\right)
        \end{array}
    \right\} \\ \nonumber &\to 
  \mathbf{B}_t \left(e  - \frac{1}{4}\right) + \mathbf{B}_t \left( d - 2(\sqrt{e}-e)\right).
\end{align}
where $d = \frac{\hat{d}^{\mathcal{V}}_{\rho}}{1-\frac{\lambda_2}{2}} + \frac{\psi_d}{\lambda_2(1-\frac{\lambda_2}{2})} \textnormal{and} \; e = \frac{\hat{e}^{\mathcal{V}}_{\rho}}{1-\frac{\lambda_1}{2}} + \frac{\psi_e}{\lambda_1(1-\frac{\lambda_1}{2})} $
\begin{align}\label{min_e_second order}
    \min_{\mathcal{Q}_v,\rho,\lambda} \mathcal{E}_{\text{II}} \;  \textnormal{s.t.} &\left\{ \left(\frac{\hat{e}^{\mathcal{V}}_{\rho}}{1-\frac{\lambda}{2}} + \frac{\psi_e}{\lambda(1-\frac{\lambda}{2})} \right) \leq \frac{1}{4} \right\} \\ \nonumber &\to \mathbf{B}_t \left(\frac{\hat{e}^{\mathcal{V}}_{\rho}}{1-\frac{\lambda}{2}} + \frac{\psi_e}{\lambda(1-\frac{\lambda}{2})}  - \frac{1}{4}\right).
\end{align}
\begin{align}\label{min_R_second_order}
\min_{\mathcal{Q}_v,\rho,\lambda,\gamma} \mathcal{R}_{\textnormal{II}} \;  \textnormal{s.t.} &\left\{
        \begin{array}{l}
            r \leq 1/2, 
            d \leq 1/2
        \end{array}
    \right\} \\ \nonumber &\to \mathbf{B}_t\left(r-\frac{1}{2}\right) + \mathbf{B}_t\left(d-\frac{1}{2}\right).
\end{align}
where $r = \frac{\hat{\mathfrak{R}}^{\mathcal{V}}_{\rho}}{1-\frac{\lambda}{2}} + \frac{\psi_r}{\lambda(1-\frac{\lambda}{2})}$ and $d = \left(1-\frac{\gamma}{2}\right)\hat{d}_{\rho} - \frac{\psi_d}{\gamma}$.

\textbf{Optimization of Inverted KL Bounds}. The main challenge in optimizing the multi-view first- and second-order inverted KL bounds is to evaluate $\overline{\textnormal{KL}}$ and $\underline{\textnormal{KL}}$ and to compute their derivatives. To achieve this, we employ the bisection method proposed by \cite{Reeb18} for calculating $\textnormal{KL}$. This method is outlined in the functions Compute-$\overline{\textnormal{KL}}(q\|\psi)$, Compute-$\underline{\textnormal{KL}}(q\|\psi)$ of Algorithm~\ref{alg:kl_computation} in the Appendix~\ref{Optimization of Inverted Multi-View KL Bounds}. It involves iteratively refining an interval $[p_{\min}, p_{\max}]$ such that $p \in [p_{\min}, p_{\max}]$ and $\textnormal{KL}(q \parallel p) = \psi$. We aim to minimize the following constrained objective functions derived from the inverted KL bounds from Corollaries~\ref{Pac-bayes-kl-inv-mv-FO}, \ref{Pac-bayes-kl-inv-mv-SO}, \ref{Multi-view PAC-Bayes-KL_Inequality_binary} and \ref{Multi-view PAC-Bayes-KL_Inequality_binary}:
\begin{align}
   \min_{\mathcal{Q}_v,\rho} \mathcal{K}  \text{ s.t. } &\left\{\overline{\textnormal{KL}}\left(\hat{\mathfrak{R}}^{\mathcal{V}}_{\rho} \middle\| \psi_r \right) \leq \frac{1}{2}\right\} \\ \nonumber &\to  \mathbf{B}_t \left(\overline{\textnormal{KL}}\left(\hat{\mathfrak{R}}^{\mathcal{V}}_{\rho} \middle\| \psi_r \right) - \frac{1}{2}\right). 
\label{min_k}\end{align}
\begin{align}
\min_{\mathcal{Q}_v,\rho,\lambda_1,\lambda_2} \mathcal{K}^{u}\text{s.t.} &\left\{
        \begin{array}{l}
            e \leq 1/4, 
            d \leq 2\left(\sqrt{e}-e\right)
        \end{array}
    \right\}  \\ \nonumber &\to \mathbf{B}_t \left(e - \frac{1}{4}\right) + \mathbf{B}_t \left( d - 2(\sqrt{e}-e)\right).
\label{min_K_u}\end{align}

where $e = \overline{\textnormal{KL}}\left(\hat{e}^{\mathcal{V}}_{\rho}\middle\| \psi_e\right)$ and $d = \overline{\textnormal{KL}}\left(\hat{d}^{\mathcal{V}}_{\rho} \middle\| \psi_d\right)$
\begin{align}
\min_{\mathcal{Q}_v,\rho} \mathcal{K}_{\textnormal{II}} \text{ s.t. } &\left\{\overline{\textnormal{KL}}\left(e^{\mathcal{V}}_{\rho} \middle\| \psi_e \right) \leq \frac{1}{4}\right\} \\ \nonumber &\to  \mathbf{B}_t \left(\overline{\textnormal{KL}}\left(e^{\mathcal{V}}_{\rho} \middle\| \psi_e \right) - \frac{1}{4}\right),
\label{min_K_SO}\end{align}
\begin{align}
\min_{\mathcal{Q}_v,\rho} \mathcal{K}^{u}_{\textnormal{II}} \text{s.t.} &\left\{ 
        \begin{array}{l}
            r \leq 1/2, 
            d \leq 1/2
        \end{array}
     \right\} \\ \nonumber &\to \mathbf{B}_t\left(r - \frac{1}{2}\right) + \mathbf{B}_t\left(d - \frac{1}{2}\right).
\end{align}\label{min_K_SO_u}

where $r = \overline{\textnormal{KL}}\left(\hat{\mathfrak{R}}^{\mathcal{V}}_{\rho} \middle\| \psi_r \right)$ and $d = \underline{\textnormal{KL}}\left(\hat{d}^{\mathcal{V}}_{\rho} \middle\| \psi_r \right)$ 

\textbf{Optimization of Multi-view PAC-Bayesian C-Bound}. In this section, we present self-bounding algorithms to directly minimize the PAC-Bayesian C-Bounds. We aim at minimizing the following constraint optimization problem:
\begin{align}\label{min_C_bound}
     \min_{\mathcal{Q}_v,\rho}  \mathcal{C}^{S}_{\rho} \textnormal{s.t.} &\left\{\overline{\textnormal{KL}}\left(\hat{\mathfrak{R}}^{\mathcal{V}}_{\rho}\|\psi_{r}\right) \leq \frac{1}{2}\right\} \\ \nonumber &\to \mathbf{B}_{t}\left( \overline{\textnormal{KL}}\left(\hat{\mathfrak{R}}^{\mathcal{V}}_{\rho}\|\psi_{r}\right) - \frac{1}{2} \right). 
\end{align}
From the equation~\ref{Eq-Pac-bayes-mv-C-tandem-Bound} of Theorem~\ref{Multi-view PAC-Bayesian C-tandem Oracle Bound} we aim at minimizing the following constraint optimization problem:
\begin{align}\label{C_tandem_bound}
\min_{\mathcal{Q}_v,\rho} \mathcal{C}_{T}^{S}  \; \textnormal{s.t.} & \left\{
        \begin{array}{l}
            r \leq 1/2, 
            e \leq 1/4
        \end{array}
     \right\} \\ \nonumber &\to \mathbf{B}_t(r-\frac{1}{2}) + \mathbf{B}_t(e-\frac{1}{4}).
\end{align}
where $r = \overline{\textnormal{KL}}\left(\hat{\mathfrak{R}}^{\mathcal{V}}_{\rho} \middle\| \psi_r \right)$ and $e = \overline{\textnormal{KL}}\left(\hat{e}^{\mathcal{V}}_{\rho}\|\psi_{e}\right)$ 

\section{Experiments}
\label{Experiments}

In this section, we evaluate our proposed algorithms on multi-view datasets. Our experiments focus on two aspects: (1) analyzing the intra- and inter-view PAC-Bayesian bounds and (2) examining how the parameter $\alpha$ and the proportion of labeled data influence the bounds' values.

To assess the effectiveness of our approach, we employ a total of 10 datasets\footnote{Processed datasets are available for download at \url{https://osf.io/xh5qs/?view_only=966ab35b04bd4e478491038941f7c141}.}. While some datasets are inherently multi-view, others were originally mono-view and required transformation and feature extraction to fit our multi-view setting.\footnote{Each dataset contains $|V|$ views, along with a concatenated representation of all views.} Datasets with multiple classes were selected to facilitate the optimization of both multi-classification and binary classification bounds. For a detailed description of each dataset and its source, refer to Table~\ref{tab:datasets} in the Appendix. Following the methodology of \cite{Masegosa2020}, we evaluate the effectiveness of our proposed bounds.\footnote{The codebase and obtained results are available at \url{https://anonymous.4open.science/r/Multi-View-Majority-Vote-Learning-Algorithms-Direct-Minimization-of-PAC-Bayesian-Bounds-4B77}.} Additional details on the experimental setup and hyperparameter choices can be found in Appendix~\ref{sup_experimental}.

\paragraph{Results.} Figures~\ref{figure:mfeat-binary-4-9} and~\ref{figure:mfeat-mult} display the optimized Bayes risk and bound values for each of our proposed self-bounding algorithms, allowing comparisons across individual views, the concatenated view, and the multi-view setting. For single-view experiments, some methods were adapted from previous work \citep{Masegosa2020, Viallard2011}, while others, such as the first- and second-order inverted KL bounds, are newly implemented.

We primarily report results for the “mfeat-large” dataset in both binary and multi-class classification scenarios. This dataset offers the most views and the largest number of samples among those we considered, providing a rich multi-view setting and enhancing the statistical significance of our results. To save space, only the concatenated and multi-view subplots are included for the multi-class plot; Results on the other datasets are included in the Appendix~\ref{App:results}.

We note that the slashed bars (\textbackslash) represent the Bayes risks $R^{\mathcal{V}}_T$ on the test data. Our multi-view method generally outperforms single-view approaches and the concatenated view in terms of Bayes risk. While we obtain tighter bounds than single-view methods, concatenated views frequently achieve tighter bounds than our multi-view approach, except the C-Tandem Oracle bound, particularly on datasets where views are artificially constructed by splitting single data sources (e.g., dividing images into quadrants as in \cite{Goyal19}, or extracting feature sets from images). This advantage likely stems from the additional divergence terms in multi-view bounds: $D_\alpha(\rho \| \pi)$ for the view-level distribution and $\mathbb{E}_{\rho}[D_{\alpha_v}(\mathcal{Q}_v \| \mathcal{P}_v)]$ for within-view distributions. On naturally multi-view data, however, our approach shows comparable or better performance while providing theoretical guarantees that concatenation cannot offer (ALOI dataset, Figure~\ref{figure:aloi-mult-full-plot}).

Across all settings—including single views, the concatenated view, and our multi-view method—the first-order bound yields the tightest results, which aligns with Theorem 2 from \cite{Viallard2011}, stating that when $R_{\mathcal{D}}^{\mathcal{V}} > d_{\mathcal{D}}^{\mathcal{V}}$, the first-order bound is theoretically the tightest compared to the C-bound and the second-order bound. In practice, we observe this condition holds in most of our experiments.

We explore a broader range of configurations, including variations in $\alpha$ and labeled data proportions, as shown in Appendix~\ref{App:analysis}. Figure~\ref{figure:vary-labeled} highlights the effects of varying the proportion of labeled data ($s\_labeled\_size$) on bound values, with a fixed $\alpha = 1.1$. As labeled data increases, bounds improve, with $\mathcal{K}^{u}_{\textnormal{II}}$ (Equation~\ref{Eq-Pac-bayes-dis-inv-mv-SO}) achieving tighter values than $\mathcal{K}_{\textnormal{II}}$ (Equation~\ref{Eq-Pac-bayes-joint-inv-mv-SO}), especially with more unlabeled data and the inclusion of the disagreement term. This suggests that incorporating disagreement enhances the bound's tightness by enabling learning from unlabeled data, which aligns with the theoretical difference between Equations~\ref{Eq-Pac-bayes-dis-inv-mv-SO} and~\ref{Eq-Pac-bayes-joint-inv-mv-SO}.

In contrast, Figure~\ref{figure:vary-alpha} examines the effect of varying $\alpha$ on bound values with a fixed $s\_labeled\_size = 0.5$. The bounds generally tighten around $\alpha = 1.1$, suggesting that this value provides an optimal trade-off for controlling the Rényi divergence. This observation is further supported by the results in Figure~\ref{figure:mfeat-binary-4-9-alpha1.1vsoptim}, where setting $\alpha$ as an optimizable parameter leads to convergence near 1.1. This trend highlights the importance of $\alpha$ in regulating bound tightness. Under this optimization setting, Table~\ref{tab:optim_aplha} demonstrates that different views may converge to different $\alpha_v$ values. These results confirm our hypothesis that heterogeneous views have intrinsically different complexities and characteristics that benefit from view-specific regularization strengths.

\begin{figure}[h!]
\centering
\includegraphics[width=\linewidth]{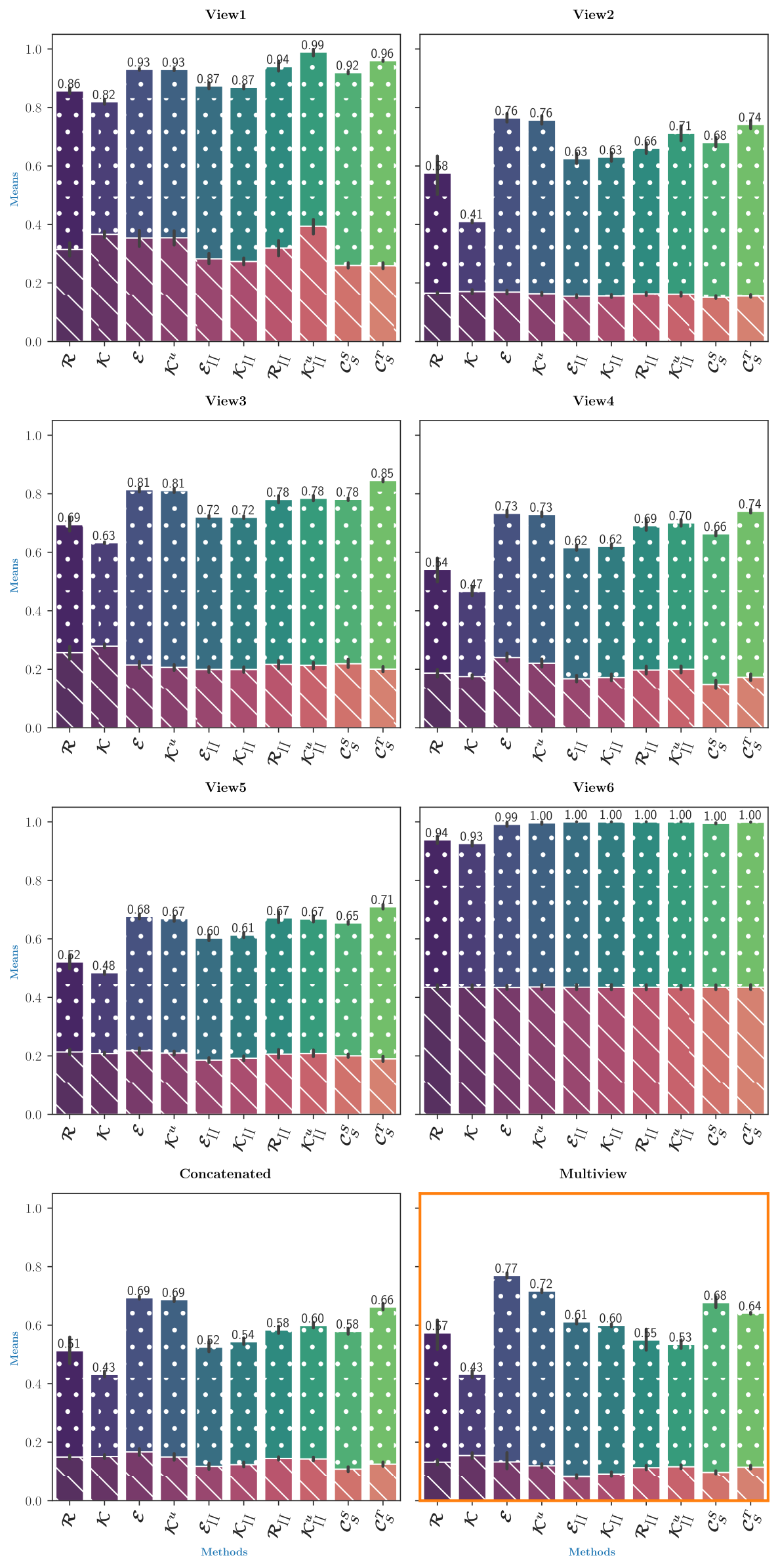}
\caption{Test error rates and PAC-Bayesian bounds for binary classification between labels 4 and 9 on the mfeat-large dataset, averaged over 10 runs. Each subplot represents a different view. Dotted bars ($\bullet$) indicate bounds, while slashed bars (\textbackslash) represent risks. Colors distinguish between bounds, risks, and methods within each subplot. The experiment uses KL divergence for single-view and Rényi divergence ($\alpha=1.1$) for multi-view, with a stump configuration and 50\% labeled data. Multi-view results are highlighted in orange.}
\label{figure:mfeat-binary-4-9}
\end{figure}
\begin{figure}[h!]
\centering
\includegraphics[width=\linewidth]{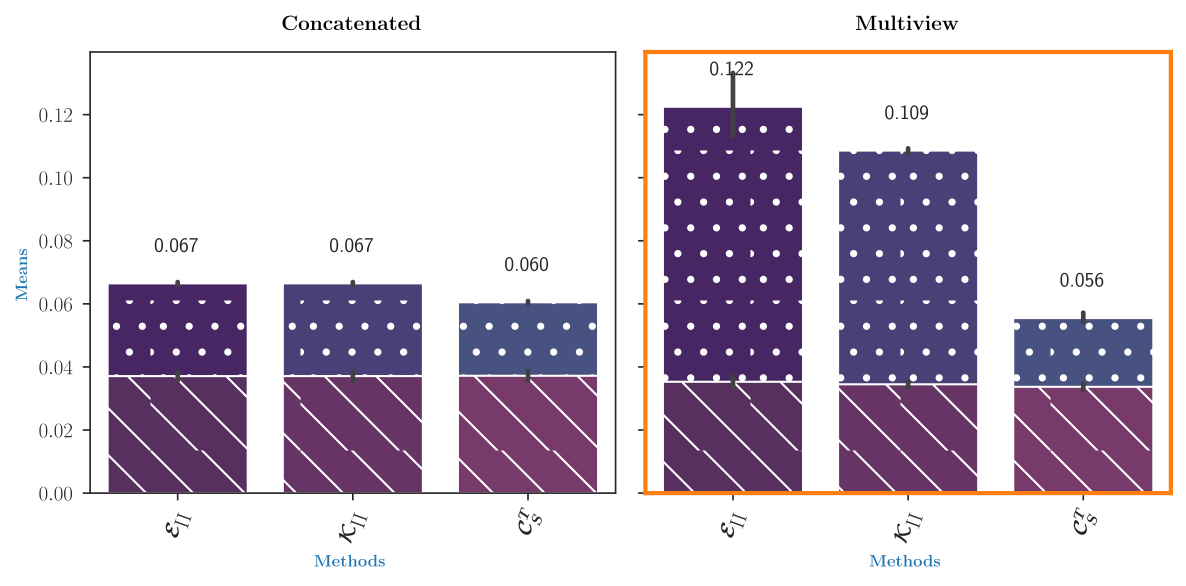}
\caption{Test error rates and PAC-Bayesian bounds for multiclass classification on the mfeat-large dataset, averaged over 10 runs. Only the concatenated view and the multi-view are shown (full plot with all views in Appendix). The experiment uses the same configuration as Figure~\ref{figure:mfeat-binary-4-9} with modifications to aid multi-class learning, strong learners with depth=20, and 100\% labeled data. Multi-view results are highlighted in orange.}
\label{figure:mfeat-mult}
\end{figure}
\begin{table}[h!]
\centering
\caption{The final optimized $\alpha_v$ per view and $\alpha$ for multi-view. Same configuration as Figure~\ref{figure:mfeat-binary-4-9} with optimizable $\alpha$.}
\begin{tabular}{p{0.4cm}p{0.55cm}p{0.55cm}p{0.55cm}p{0.55cm}p{0.55cm}p{0.6cm}|p{0.55cm}}
\toprule
B & $\alpha_1$ & $\alpha_2$ & $\alpha_3$ & $\alpha_4$ & $\alpha_5$ & $\alpha_6$ & $\alpha$\\
\midrule
\boldmath{$\mathcal{R}$} & 1.254 & 1.183 & 1.203 & 1.175 & 1.176 & 1.254 & 1.189 \\
\boldmath{$\mathcal{K}$} & 1.127 & 1.082 & 1.113 & 1.059 & 1.063 & 1.128 & 1.073 \\
\boldmath{$\mathcal{E}$} & 2.134 & 2.179 & 2.148 & 2.168 & 2.152 & 2.093 & 2.215 \\
\boldmath{$\mathcal{K}^{u}$} & 1.062 & 1.062 & 1.061 & 1.059 & 1.056 & 1.059 & 1.124 \\
\boldmath{$\mathcal{E}_{\textnormal{II}}$} & 1.097 & 1.086 & 1.087 & 1.082 & 1.087 & 1.095 & 1.105 \\
\boldmath{$\mathcal{K}_{\textnormal{II}}$} & 1.068 & 1.073 & 1.069 & 1.067 & 1.071 & 1.060 & 1.107 \\
\boldmath{$\mathcal{R}_{\textnormal{II}}$} & 1.425 & 1.464 & 1.470 & 1.435 & 1.467 & 1.431 & 1.439 \\
\boldmath{$\mathcal{K}^{u}_{\textnormal{II}}$} & 1.306 & 1.346 & 1.402 & 1.315 & 1.360 & 1.332 & 1.299 \\
\boldmath{$\mathcal{C}_{\rho}^{S}$} & 1.739 & 1.740 & 1.740 & 1.740 & 1.739 & 1.739 & 1.739 \\
\boldmath{$\mathcal{C}_{S}^{T}$} & 1.000 & 1.000 & 1.000 & 1.000 & 1.000 & 1.000 & 1.000 \\
\bottomrule
\end{tabular}
\label{tab:optim_aplha}
\end{table}

\section{Conclusion}

This paper establishes the first complete optimization framework for PAC-Bayesian multi-view learning, bridging the gap between theory and practice. Our self-bounding algorithms enable direct minimization of multi-view PAC-Bayesian bounds, including the \emph{in-probability} versions of \cite{Goyal17,Goyal19}'s formulations, and single-view methods \citep{Masegosa2020,Viallard2011}, making these theoretical guarantees practically accessible with a unified framework. Beyond implementation, we extended the framework with Rényi divergence and view-specific $\alpha_v$ parameters, allowing more flexibility in experimenting with different divergence measures. These advances are comprehensively summarized in Table~\ref{tab:comparison_detailed} in the Appendix.

Our framework's limitations point to promising research directions. The restriction to $\alpha > 1$ in our multi-view Rényi change of measure inequality (Appendix~\ref{proof of General Multiview PAC-Bayesian Theorem}) yields values exceeding KL divergence as noted by \cite{Erven212}, potentially loosening bounds compared to KL-based alternatives. Exploring $\alpha = \frac{1}{2}$ (corresponding to Hellinger distance) could provide tighter bounds but requires resolving theoretical constraints in the change of measure inequality. Finally, integrating adversarial robustness techniques \citep{SUN22} could strengthen view-specific learning through stability-based approaches, particularly when some views are noisy or corrupted (see Appendix~\ref{App:poisoning}).

\bibliographystyle{apalike}   %
\bibliography{main}            

\begin{thebibliography}{}

\bibitem[Alquier, 2024]{Alquier_2024}
Alquier, P. (2024).
\newblock User-friendly introduction to pac-bayes bounds.
\newblock {\em Foundations and Trends® in Machine Learning}, 17(2):174–303.

\bibitem[Blum and Mitchell, 1998]{Blum1998}
Blum, A. and Mitchell, T. (1998).
\newblock Combining labeled and unlabeled data with co-training.
\newblock In {\em Proceedings of the eleventh annual conference on Computational learning theory}, pages 92--100.

\bibitem[Bégin et~al., 2016]{begin16}
Bégin, L., Germain, P., Laviolette, F., and Roy, J.-F. (2016).
\newblock Pac-bayesian bounds based on the rényi divergence.
\newblock In Gretton, A. and Robert, C.~C., editors, {\em Proceedings of the 19th International Conference on Artificial Intelligence and Statistics}, volume~51 of {\em Proceedings of Machine Learning Research}, pages 435--444, Cadiz, Spain. PMLR.

\bibitem[Catoni et~al., 2007]{Catoni07}
Catoni, O., Euclid, P., Library, C.~U., and Press, D.~U. (2007).
\newblock {\em PAC-Bayesian Supervised Classification: The Thermodynamics of Statistical Learning}.
\newblock Lecture notes-monograph series. Cornell University Library.

\bibitem[Chua et~al., 2009]{nus-wide-civr09}
Chua, T.-S., Tang, J., Hong, R., Li, H., Luo, Z., and Zheng, Y.-T. (July 8-10, 2009).
\newblock Nus-wide: A real-world web image database from national university of singapore.
\newblock In {\em Proc. of ACM Conf. on Image and Video Retrieval (CIVR'09)}, Santorini, Greece.

\bibitem[Cohen et~al., 2017]{cohen2017emnist}
Cohen, G., Afshar, S., Tapson, J., and van Schaik, A. (2017).
\newblock Emnist: an extension of mnist to handwritten letters.

\bibitem[Dalalyan and Tsybakov, 2008]{Dalalyan_2008}
Dalalyan, A. and Tsybakov, A.~B. (2008).
\newblock Aggregation by exponential weighting, sharp pac-bayesian bounds and sparsity.
\newblock {\em Machine Learning}, 72(1–2):39–61.

\bibitem[Dasgupta et~al., 2001]{Dasgupta2001}
Dasgupta, S., Littman, M., and McAllester, D. (2001).
\newblock Pac generalization bounds for co-training.
\newblock In Dietterich, T., Becker, S., and Ghahramani, Z., editors, {\em Advances in Neural Information Processing Systems}, volume~14. MIT Press.

\bibitem[Donsker and Varadhan, 1975]{Donsker1975AsymptoticEO}
Donsker, M.~D. and Varadhan, S. R.~S. (1975).
\newblock Asymptotic evaluation of certain markov process expectations for large time, i.
\newblock {\em Communications on Pure and Applied Mathematics}, 28(1):1--47.

\bibitem[Duin, 1998]{misc_multiple_features_72}
Duin, R. (1998).
\newblock {Multiple Features}.
\newblock UCI Machine Learning Repository.
\newblock {DOI}: https://doi.org/10.24432/C5HC70.

\bibitem[Dziugaite and Roy, 2017]{Gintare}
Dziugaite, G.~K. and Roy, D.~M. (2017).
\newblock Computing nonvacuous generalization bounds for deep (stochastic) neural networks with many more parameters than training data.
\newblock {\em CoRR}, abs/1703.11008.

\bibitem[Fang et~al., 2023]{Fang2023}
Fang, U., Li, M., Li, J., Gao, L., Jia, T., and Zhang, Y. (2023).
\newblock A comprehensive survey on multi-view clustering.
\newblock {\em IEEE Transactions on Knowledge and Data Engineering}, 35(12):12350--12368.

\bibitem[Farquhar et~al., 2005]{Farquhar2005}
Farquhar, J., Hardoon, D., Meng, H., Shawe-taylor, J., and Szedm\'{a}k, S. (2005).
\newblock Two view learning: Svm-2k, theory and practice.
\newblock In Weiss, Y., Sch\"{o}lkopf, B., and Platt, J., editors, {\em Advances in Neural Information Processing Systems}, volume~18. MIT Press.

\bibitem[Foong et~al., 2021]{Foong21}
Foong, A., Bruinsma, W., Burt, D., and Turner, R. (2021).
\newblock How tight can pac-bayes be in the small data regime?
\newblock {\em Advances in Neural Information Processing Systems}, 34:4093--4105.

\bibitem[Germain et~al., 2015a]{germain2015a}
Germain, P., Lacasse, A., Laviolette, F., March, M., and Roy, J.-F. (2015a).
\newblock Risk bounds for the majority vote: From a pac-bayesian analysis to a learning algorithm.
\newblock {\em Journal of Machine Learning Research}, 16(26):787--860.

\bibitem[Germain et~al., 2015b]{Germain15a}
Germain, P., Lacasse, A., Laviolette, F., March, M., and Roy, J.-F. (2015b).
\newblock Risk bounds for the majority vote: From a pac-bayesian analysis to a learning algorithm.
\newblock {\em Journal of Machine Learning Research}, 16(26):787--860.

\bibitem[Germain et~al., 2009]{Germain09}
Germain, P., Lacasse, A., Laviolette, F., and Marchand, M. (2009).
\newblock Pac-bayesian learning of linear classifiers.
\newblock In {\em Proceedings of the 26th Annual International Conference on Machine Learning}, ICML '09, page 353–360, New York, NY, USA. Association for Computing Machinery.

\bibitem[Goyal et~al., 2017]{Goyal17}
Goyal, A., Morvant, E., Germain, P., and Amini, M.-R. (2017).
\newblock Pac-bayesian analysis for a two-step hierarchical multiview learning approach.
\newblock In Ceci, M., Hollm{\'e}n, J., Todorovski, L., Vens, C., and D{\v{z}}eroski, S., editors, {\em Machine Learning and Knowledge Discovery in Databases}, pages 205--221, Cham. Springer International Publishing.

\bibitem[Goyal et~al., 2019a]{Goyal19}
Goyal, A., Morvant, E., Germain, P., and Amini, M.-R. (2019a).
\newblock Multiview boosting by controlling the diversity and the accuracy of view-specific voters.
\newblock {\em Neurocomputing}, 358:81--92.

\bibitem[Goyal et~al., 2019b]{GOYAL201981}
Goyal, A., Morvant, E., Germain, P., and Amini, M.-R. (2019b).
\newblock Multiview boosting by controlling the diversity and the accuracy of view-specific voters.
\newblock {\em Neurocomputing}, 358:81--92.

\bibitem[Grunwald et~al., 2021]{pmlr-v134-grunwald21a}
Grunwald, P., Steinke, T., and Zakynthinou, L. (2021).
\newblock Pac-bayes, mac-bayes and conditional mutual information: Fast rate bounds that handle general vc classes.
\newblock In Belkin, M. and Kpotufe, S., editors, {\em Proceedings of Thirty Fourth Conference on Learning Theory}, volume 134 of {\em Proceedings of Machine Learning Research}, pages 2217--2247. PMLR.

\bibitem[Kervadec et~al., 2019]{Kervadec17}
Kervadec, H., Dolz, J., Yuan, J., Desrosiers, C., Granger, E., and Ayed, I.~B. (2019).
\newblock Log-barrier constrained cnns.
\newblock {\em CoRR}, abs/1904.04205.

\bibitem[Lacasse et~al., 2006]{Lacasse06}
Lacasse, A., Laviolette, F., Marchand, M., Germain, P., and Usunier, N. (2006).
\newblock Pac-bayes bounds for the risk of the majority vote and the variance of the gibbs classifier.
\newblock In Sch\"{o}lkopf, B., Platt, J., and Hoffman, T., editors, {\em Advances in Neural Information Processing Systems}, volume~19. MIT Press.

\bibitem[Langford, 2005]{Langford05a}
Langford, J. (2005).
\newblock Tutorial on practical prediction theory for classification.
\newblock {\em Journal of Machine Learning Research}, 6(10):273--306.

\bibitem[Langford and Shawe-Taylor, 2002]{Langford02}
Langford, J. and Shawe-Taylor, J. (2002).
\newblock Pac-bayes \&amp; margins.
\newblock In Becker, S., Thrun, S., and Obermayer, K., editors, {\em Advances in Neural Information Processing Systems}, volume~15. MIT Press.

\bibitem[Laviolette et~al., 2011]{Laviolette11}
Laviolette, F., Marchand, M., and Roy, J.-F. (2011).
\newblock From pac-bayes bounds to quadratic programs for majority votes.
\newblock In {\em Proceedings of the 28th International Conference on International Conference on Machine Learning}, ICML'11, page 649–656, Madison, WI, USA. Omnipress.

\bibitem[Laviolette et~al., 2017]{Laviolette17}
Laviolette, F., Morvant, E., Ralaivola, L., and Roy, J.-F. (2017).
\newblock Risk upper bounds for general ensemble methods with an application to multiclass classification.
\newblock {\em Neurocomputing}, 219:15--25.

\bibitem[LeCun et~al., 2010]{lecun2010mnist}
LeCun, Y., Cortes, C., and Burges, C. (2010).
\newblock Mnist handwritten digit database.
\newblock {\em ATT Labs [Online]. Available: http://yann.lecun.com/exdb/mnist}, 2.

\bibitem[Li and Wang, 2008]{corel}
Li, J. and Wang, J. (2008).
\newblock Real-time computerized annotation of pictures.
\newblock {\em IEEE transactions on pattern analysis and machine intelligence}, 30:985--1002.

\bibitem[Loshchilov and Hutter, 2017]{Ilya17}
Loshchilov, I. and Hutter, F. (2017).
\newblock Fixing weight decay regularization in adam.
\newblock {\em CoRR}, abs/1711.05101.

\bibitem[Ma et~al., 2024]{Ma2024}
Ma, G., Lu, J., Fang, Z., Liu, F., and Zhang, G. (2024).
\newblock Multiview classification through learning from interval-valued data.
\newblock {\em IEEE Transactions on Neural Networks and Learning Systems}, pages 1--15.

\bibitem[Masegosa et~al., 2020]{Masegosa2020}
Masegosa, A., Lorenzen, S., Igel, C., and Seldin, Y. (2020).
\newblock Second order pac-bayesian bounds for the weighted majority vote.
\newblock In Larochelle, H., Ranzato, M., Hadsell, R., Balcan, M., and Lin, H., editors, {\em Advances in Neural Information Processing Systems}, volume~33, pages 5263--5273. Curran Associates, Inc.

\bibitem[Matkowski and R{\"a}tz, 1997]{Matkowski1997}
Matkowski, J. and R{\"a}tz, J. (1997).
\newblock Convexity of power functions with respect to symmetric homogeneous means.
\newblock In Bandle, C., Everitt, W.~N., Losonczi, L., and Walter, W., editors, {\em General Inequalities 7}, pages 231--247, Basel. Birkh{\"a}user Basel.

\bibitem[Maurer, 2004]{maurer2004notepacbayesiantheorem}
Maurer, A. (2004).
\newblock A note on the pac bayesian theorem.

\bibitem[McAllester, 1998]{McAllester99}
McAllester, D.~A. (1998).
\newblock Some pac-bayesian theorems.
\newblock In {\em Proceedings of the Eleventh Annual Conference on Computational Learning Theory}, COLT' 98, page 230–234, New York, NY, USA. Association for Computing Machinery.

\bibitem[McAllester, 2003]{McAllester03a}
McAllester, D.~A. (2003).
\newblock Pac-bayesian stochastic model selection.
\newblock {\em Machine Learning}, 51:5--21.

\bibitem[Orabona and Tommasi, 2017]{COCOB}
Orabona, F. and Tommasi, T. (2017).
\newblock Backprop without learning rates through coin betting.
\newblock {\em CoRR}, abs/1705.07795.

\bibitem[Padmanabhan et~al., 2016]{reuters}
Padmanabhan, D., Bhat, S., Shevade, S.~K., and Narahari, Y. (2016).
\newblock Topic model based multi-label classification from the crowd.
\newblock {\em CoRR}, abs/1604.00783.

\bibitem[Paszke et~al., 2019]{paszke2019pytorch}
Paszke, A., Gross, S., Massa, F., Lerer, A., Bradbury, J., Chanan, G., Killeen, T., Lin, Z., Gimelshein, N., Antiga, L., et~al. (2019).
\newblock Pytorch: An imperative style, high-performance deep learning library.
\newblock NeurIPS.

\bibitem[P\'{e}rez-Ortiz et~al., 2021]{Perez2021}
P\'{e}rez-Ortiz, M., Rivasplata, O., Shawe-Taylor, J., and Szepesv\'{a}ri, C. (2021).
\newblock Tighter risk certificates for neural networks.
\newblock {\em J. Mach. Learn. Res.}, 22(1).

\bibitem[Redko et~al., 2019]{Redko2019}
Redko, I., Morvant, E., Habrard, A., Sebban, M., and Bennani, Y. (2019).
\newblock {\em Advances in domain adaptation theory}.
\newblock Elsevier.

\bibitem[Reeb et~al., 2018]{Reeb18}
Reeb, D., Doerr, A., Gerwinn, S., and Rakitsch, B. (2018).
\newblock Learning gaussian processes by minimizing pac-bayesian generalization bounds.
\newblock In Bengio, S., Wallach, H., Larochelle, H., Grauman, K., Cesa-Bianchi, N., and Garnett, R., editors, {\em Advances in Neural Information Processing Systems}, volume~31. Curran Associates, Inc.

\bibitem[Rosenberg and Bartlett, 2007]{rosenberg07a}
Rosenberg, D.~S. and Bartlett, P.~L. (2007).
\newblock The rademacher complexity of co-regularized kernel classes.
\newblock In Meila, M. and Shen, X., editors, {\em Proceedings of the Eleventh International Conference on Artificial Intelligence and Statistics}, volume~2 of {\em Proceedings of Machine Learning Research}, pages 396--403, San Juan, Puerto Rico. PMLR.

\bibitem[Rosenberg et~al., 2009]{Rosenberg2009}
Rosenberg, D.~S., Sindhwani, V., Bartlett, P.~L., and Niyogi, P. (2009).
\newblock Multiview point cloud kernels for semisupervised learning [lecture notes].
\newblock {\em IEEE Signal Processing Magazine}, 26(5):145--150.

\bibitem[Schlimmer, 1987]{misc_mushroom_73}
Schlimmer, J. (1987).
\newblock {Mushroom}.
\newblock UCI Machine Learning Repository.
\newblock {DOI}: https://doi.org/10.24432/C5959T.

\bibitem[Schubert and Zimek, 2010]{elki-aloi-dataset_2010}
Schubert, E. and Zimek, A. (2010).
\newblock Elki multi-view clustering data sets based on the amsterdam library of object images (aloi) (1.0) [data set].

\bibitem[Seeger, 2003]{Seeger03}
Seeger, M. (2003).
\newblock Pac-bayesian generalisation error bounds for gaussian process classification.
\newblock {\em J. Mach. Learn. Res.}, 3(null):233–269.

\bibitem[Sindhwani and Rosenberg, 2008]{Sindhwani2008}
Sindhwani, V. and Rosenberg, D.~S. (2008).
\newblock An rkhs for multi-view learning and manifold co-regularization.
\newblock In {\em Proceedings of the 25th International Conference on Machine Learning}, ICML '08, page 976–983, New York, NY, USA. Association for Computing Machinery.

\bibitem[Strodthoff et~al., 2023a]{ptb-xl+}
Strodthoff, N., Mehari, T., Nagel, C., Aston, P., Sundar, A., Graff, C., Kanters, J., Haverkamp, W., Doessel, O., Loewe, A., Bär, M., and Schaeffter, T. (2023a).
\newblock Ptb-xl+, a comprehensive electrocardiographic feature dataset (version 1.0.1).

\bibitem[Strodthoff et~al., 2023b]{ptb-xl}
Strodthoff, N., Mehari, T., Nagel, C., Aston, P.~J., Sundar, A., Graff, C., Kanters, J.~K., Haverkamp, W., D{\"o}ssel, O., Loewe, A., B{\"a}r, M., and Schaeffter, T. (2023b).
\newblock {PTB-XL+}, a comprehensive electrocardiographic feature dataset.
\newblock {\em Scientific Data}, 10(1):279.

\bibitem[Sun, 2011]{Sun2011}
Sun, S. (2011).
\newblock Multi-view laplacian support vector machines.
\newblock In Tang, J., King, I., Chen, L., and Wang, J., editors, {\em Advanced Data Mining and Applications}, pages 209--222, Berlin, Heidelberg. Springer Berlin Heidelberg.

\bibitem[Sun, 2013]{Sun2013}
Sun, S. (2013).
\newblock A survey of multi-view machine learning.
\newblock {\em Neural computing and applications}, 23:2031--2038.

\bibitem[Sun and Shawe-Taylor, 2010]{Shiliang2010}
Sun, S. and Shawe-Taylor, J. (2010).
\newblock Sparse semi-supervised learning using conjugate functions.
\newblock {\em Journal of Machine Learning Research}, 11(84):2423--2455.

\bibitem[Sun et~al., 2017]{SUN17}
Sun, S., Shawe-Taylor, J., and Mao, L. (2017).
\newblock Pac-bayes analysis of multi-view learning.
\newblock {\em Information Fusion}, 35:117--131.

\bibitem[Sun et~al., 2022]{SUN22}
Sun, S., Yu, M., Shawe-Taylor, J., and Mao, L. (2022).
\newblock Stability-based pac-bayes analysis for multi-view learning algorithms.
\newblock {\em Information Fusion}, 86-87:76--92.

\bibitem[Szedmak and Shawe-Taylor, 2007]{SZEDMAK2007}
Szedmak, S. and Shawe-Taylor, J. (2007).
\newblock Synthesis of maximum margin and multiview learning using unlabeled data.
\newblock {\em Neurocomputing}, 70(7):1254--1264.
\newblock Advances in Computational Intelligence and Learning.

\bibitem[Tang et~al., 2023]{TANG2023}
Tang, J., He, H., Fu, S., Tian, Y., Kou, G., and Xu, S. (2023).
\newblock Robust multi-view learning with the bounded linex loss.
\newblock {\em Neurocomputing}, 518:384--400.

\bibitem[Thiemann et~al., 2017]{Thiemann17}
Thiemann, N., Igel, C., Wintenberger, O., and Seldin, Y. (2017).
\newblock A strongly quasiconvex pac-bayesian bound.
\newblock In {\em International Conference on Algorithmic Learning Theory}, pages 466--492. PMLR.

\bibitem[Tian et~al., 2021]{TIAN2021}
Tian, Y., Fu, S., and Tang, J. (2021).
\newblock Incomplete-view oriented kernel learning method with generalization error bound.
\newblock {\em Information Sciences}, 581:951--977.

\bibitem[Truong, 2025]{truong2025rademachercomplexitybasedgeneralizationbounds}
Truong, L.~V. (2025).
\newblock On rademacher complexity-based generalization bounds for deep learning.

\bibitem[van Erven and Harremo{\"e}s, 2012]{Erven212}
van Erven, T. and Harremo{\"e}s, P. (2012).
\newblock R{\'e}nyi divergence and kullback-leibler divergence.
\newblock {\em IEEE Transactions on Information Theory}, 60:3797--3820.

\bibitem[Viallard et~al., 2021]{Viallard2011}
Viallard, P., Germain, P., Habrard, A., and Morvant, E. (2021).
\newblock Self-bounding majority vote learning algorithms by the direct minimization of a tight pac-bayesian c-bound.
\newblock In Oliver, N., P{\'e}rez-Cruz, F., Kramer, S., Read, J., and Lozano, J.~A., editors, {\em Machine Learning and Knowledge Discovery in Databases. Research Track}, pages 167--183, Cham. Springer International Publishing.

\bibitem[Viallard et~al., 2023]{Viallard24}
Viallard, P., Germain, P., Habrard, A., and Morvant, E. (2023).
\newblock A general framework for the practical disintegration of pac-bayesian bounds.
\newblock {\em Mach. Learn.}, 113(2):519–604.

\bibitem[Wu et~al., 2021]{WU2021}
Wu, Y.-S., Masegosa, A., Lorenzen, S., Igel, C., and Seldin, Y. (2021).
\newblock Chebyshev-cantelli pac-bayes-bennett inequality for the weighted majority vote.
\newblock In Ranzato, M., Beygelzimer, A., Dauphin, Y., Liang, P., and Vaughan, J.~W., editors, {\em Advances in Neural Information Processing Systems}, volume~34, pages 12625--12636. Curran Associates, Inc.

\bibitem[Xiao et~al., 2017]{xiao2017/online}
Xiao, H., Rasul, K., and Vollgraf, R. (2017).
\newblock Fashion-mnist: a novel image dataset for benchmarking machine learning algorithms.

\bibitem[Xu et~al., 2013]{XU013}
Xu, C., Tao, D., and Xu, C. (2013).
\newblock A survey on multi-view learning.
\newblock {\em arXiv preprint arXiv:1304.5634}.

\bibitem[Zantedeschi et~al., 2021]{ZAN2021}
Zantedeschi, V., Viallard, P., Morvant, E., Emonet, R., Habrard, A., Germain, P., and Guedj, B. (2021).
\newblock Learning stochastic majority votes by minimizing a pac-bayes generalization bound.
\newblock In {\em Proceedings of the 35th International Conference on Neural Information Processing Systems}, NIPS '21, Red Hook, NY, USA. Curran Associates Inc.

\bibitem[Zhao et~al., 2017]{ZHAO2017}
Zhao, J., Xie, X., Xu, X., Sun, S., and Wang, Y. (2017).
\newblock Multi-view learning overview: Recent progress and new challenges.
\newblock {\em Information Fusion}, 38:43--54.

\end{thebibliography}

\clearpage
\appendix
\thispagestyle{empty}

\onecolumn
\aistatstitle{Supplementary Materials}


\section{Notation Reference}\label{App:notation}

\subsection*{Basic Notation}
\begin{table}[H]
\centering
\caption{Basic learning framework notation}
\label{tab:basic_notation}
\begin{tabular}{@{}p{0.2\textwidth}p{0.75\textwidth}@{}}
\toprule
\textbf{Symbol} & \textbf{Description} \\ 
\midrule
$\mathcal{X} \subseteq \mathbb{R}^d$ & $d$-dimensional input space \\
$\mathcal{Y} \subseteq \mathbb{N}$ & Finite label space \\
$\mathcal{D}$ & Unknown data distribution on $\mathcal{X} \times \mathcal{Y}$ \\
$\mathcal{D}_{\mathcal{X}}$ & Marginal distribution on $\mathcal{X}$ \\
$S = \{(\bm{x}_i, y_i)\}_{i=1}^m$ & Learning sample of size $m$ \\
$\mathcal{H}$ & Hypothesis set (voters $h: \mathcal{X} \rightarrow \mathcal{Y}$) \\
$\mathbb{I}(\cdot)$ & Indicator function \\
$\ell(h(\bm{x}),y)$ & 0-1 loss: $\mathbb{I}(h(\bm{x}) \neq y)$ \\
\bottomrule
\end{tabular}
\end{table}
\subsection*{Multi-View Specific Notation}
\begin{table}[H]
\centering
\caption{Multi-view specific notation}
\label{tab:multiview_notation}
\begin{tabular}{@{}p{0.2\textwidth}p{0.75\textwidth}@{}}
\toprule
\textbf{Symbol} & \textbf{Description} \\ 
\midrule
$V \geq 2$ & Number of views \\
$[\![V]\!]$ & Set $\{1, 2, \ldots, V\}$ \\
$v \in [\![V]\!]$ & View index \\
$\mathcal{X}^v \subset \mathbb{R}^{d_v}$ & Input space for view $v$ with dimension $d_v$ \\
$d = d_1 \times \cdots \times d_V$ & Combined dimensions of all views \\
$\bm{x}^v$ & Data instance from view $v$ \\
$\mathcal{H}_v$ & View-specific hypothesis set \\
$S = \{(\bm{x}_i^v, y_i)\}_{i=1}^m$ & View-specific labeled sample \\
\bottomrule
\end{tabular}
\end{table}
\subsection*{Prior and Posterior Distributions}
\begin{table}[H]
\centering
\caption{Prior and posterior distributions}
\label{tab:distributions}
\begin{tabular}{@{}p{0.2\textwidth}p{0.75\textwidth}@{}}
\toprule
\textbf{Symbol} & \textbf{Description} \\ 
\midrule
$\mathcal{P}$ & Prior distribution over $\mathcal{H}$ (single-view) \\
$\mathcal{Q}$ & Posterior distribution over $\mathcal{H}$ (single-view) \\
$\mathcal{P}_v$ & Prior distribution for view $v$ \\
$\mathcal{Q}_v$ & Posterior distribution for view $v$ \\
$\pi$ & Hyper-prior distribution over views $[\![V]\!]$ \\
$\rho$ & Hyper-posterior distribution over views $[\![V]\!]$ \\
\bottomrule
\end{tabular}
\end{table}
\subsection*{Expectation Abbreviations}
\begin{table}[H]
\centering
\caption{Expectation notation abbreviations}
\label{tab:expectations}
\begin{tabular}{@{}p{0.2\textwidth}p{0.75\textwidth}@{}}
\toprule
\textbf{Symbol} & \textbf{Description} \\ 
\midrule
$\mathbb{E}\mathbb{E}[\cdot]$ & $\mathbb{E}[\mathbb{E}[\cdot]]$ (double expectation) \\
$\mathbb{E}_{\mathcal{D}}[\cdot]$ & $\mathbb{E}_{(\bm{x}^v,y) \sim \mathcal{D}}[\cdot]$ \\
$\mathbb{E}_{\mathcal{D}_{\mathcal{X}}}[\cdot]$ & $\mathbb{E}_{\bm{x} \sim \mathcal{D}_{\mathcal{X}}}[\cdot]$ \\
$\mathbb{E}_S[\cdot]$ & $\mathbb{E}_{S \sim \mathcal{D}^m}[\cdot]$ \\
$\mathbb{E}_{\rho}[\cdot]$ & $\mathbb{E}_{v \sim \rho}[\cdot]$ \\
$\mathbb{E}_{\mathcal{Q}}[\cdot]$ & $\mathbb{E}_{h \sim \mathcal{Q}}[\cdot]$ \\
$\mathbb{E}_{\rho^2}[\cdot]$ & $\mathbb{E}_{(v,v') \sim \rho^2}[\cdot]$ \\
$\mathbb{E}_{\mathcal{Q}^2}[\cdot]$ & $\mathbb{E}_{(h,h') \sim \mathcal{Q}^2}[\cdot]$ \\
\bottomrule
\end{tabular}
\end{table}
\subsection*{Classifiers and Risks}
\begin{table}[H]
\centering
\caption{Classifiers and risk measures}
\label{tab:risks}
\begin{tabular}{@{}p{0.2\textwidth}p{0.75\textwidth}@{}}
\toprule
\textbf{Symbol} & \textbf{Description} \\ 
\midrule
$\mathcal{B}_{\mathcal{Q}}(\bm{x})$ & Bayes (weighted majority vote) classifier: $\argmax_{y \in \mathcal{Y}} \mathbb{E}_{h \sim \mathcal{Q}}[\mathbb{I}(h(\bm{x}) = y)]$ \\[0.5em]
$\mathcal{B}_{\rho}(\bm{x}^v)$ & Multi-view Bayes classifier: $\argmax_{y \in \mathcal{Y}} \mathbb{E}_{v \sim \rho} \mathbb{E}_{h \sim \mathcal{Q}_v}[\mathbb{I}(h(\bm{x}^v) = y)]$ \\[0.5em]
$R_{\mathcal{D}}$ & True risk (single-view) \\
$R_{\mathcal{D}}^{\mathcal{V}}$ & Multi-view true risk \\
$\hat{R}_S$ & Empirical risk (single-view) \\
$\hat{R}_S^{\mathcal{V}}$ & Multi-view empirical risk \\
$\mathfrak{R}_{\mathcal{D}}$ & Gibbs risk (single-view): $\mathbb{E}_{\mathcal{D}}\mathbb{E}_{\mathcal{Q}}[\ell(h(\bm{x}),y)]$ \\
$\mathfrak{R}_{\mathcal{D}}^{\mathcal{V}}$ & Multi-view Gibbs risk: $\mathbb{E}_{\mathcal{D}}\mathbb{E}_{\rho}\mathbb{E}_{\mathcal{Q}_v}[\ell(h(\bm{x}^v),y)]$ \\
$\hat{\mathfrak{R}}_S^{\mathcal{V}}$ & Multi-view empirical Gibbs risk \\
\bottomrule
\end{tabular}
\end{table}
\subsection*{Error Decomposition Terms}
\begin{table}[H]
\centering
\caption{Error decomposition terms}
\label{tab:error_decomposition}
\begin{tabular}{@{}p{0.2\textwidth}p{0.75\textwidth}@{}}
\toprule
\textbf{Symbol} & \textbf{Description} \\ 
\midrule
$e_{\mathcal{D}}$ & Expected joint error (single-view) \\[0.3em]
$e_{\mathcal{D}}^{\mathcal{V}}$ & Multi-view expected joint error: $\mathbb{E}_{\mathcal{D}}\mathbb{E}_{\rho^2}\mathbb{E}_{\mathcal{Q}_v^2}[\ell(h(\bm{x}^v),y) \times \ell(h'(\bm{x}^{v'}),y)]$ \\[0.5em]
$\hat{e}_S^{\mathcal{V}}$ & Multi-view empirical joint error \\[0.3em]
$d_{\mathcal{D}_{\mathcal{X}}}$ & Expected disagreement (single-view) \\[0.3em]
$d_{\mathcal{D}_{\mathcal{X}}}^{\mathcal{V}}$ & Multi-view expected disagreement: $\mathbb{E}_{\mathcal{D}_{\mathcal{X}}}\mathbb{E}_{\rho^2}\mathbb{E}_{\mathcal{Q}_v^2}[\ell(h(\bm{x}^v),h'(\bm{x}^{v'}))]$ \\[0.5em]
$\hat{d}_S^{\mathcal{V}}$ & Multi-view empirical disagreement \\
\bottomrule
\end{tabular}
\end{table}
\subsection*{Divergence Measures}
\begin{table}[H]
\centering
\caption{Divergence measures}
\label{tab:divergences}
\begin{tabular}{@{}p{0.2\textwidth}p{0.75\textwidth}@{}}
\toprule
\textbf{Symbol} & \textbf{Description} \\ 
\midrule
$\text{KL}(Q \| P)$ & Kullback-Leibler divergence: $\mathbb{E}_{h \sim Q}[\ln \frac{Q(h)}{P(h)}]$ \\[0.5em]
$D_\alpha(Q \| P)$ & Rényi divergence ($\alpha > 1$): $\frac{1}{\alpha - 1} \ln \mathbb{E}_{h \sim P}\left[\left(\frac{Q(h)}{P(h)}\right)^\alpha\right]$ \\[0.5em]
$\alpha$ & Rényi divergence parameter (global) \\
$\alpha_v$ & View-specific Rényi divergence parameter for view $v$ \\
\bottomrule
\end{tabular}
\end{table}
\subsection*{Inverted KL Functions}
\begin{table}[H]
\centering
\caption{Inverted KL functions}
\label{tab:inverted_kl}
\begin{tabular}{@{}p{0.2\textwidth}p{0.75\textwidth}@{}}
\toprule
\textbf{Symbol} & \textbf{Description} \\ 
\midrule
$\overline{\text{KL}}(q \| \psi)$ & Upper inverted KL: $\max\{p \in (0,1) \mid \text{KL}(q \| p) \leq \psi\}$ \\[0.3em]
$\underline{\text{KL}}(q \| \psi)$ & Lower inverted KL: $\min\{p \in (0,1) \mid \text{KL}(q \| p) \leq \psi\}$ \\
\bottomrule
\end{tabular}
\end{table}
\subsection*{Bound-Related Terms}
\begin{table}[H]
\centering
\caption{Bound-related complexity terms and parameters}
\label{tab:bound_terms}
\begin{tabular}{@{}p{0.2\textwidth}p{0.75\textwidth}@{}}
\toprule
\textbf{Symbol} & \textbf{Description} \\ 
\midrule
$\delta$ & Confidence parameter (probability of failure) \\
$m$ & Number of labeled samples \\
$n$ & Number of unlabeled samples (when applicable) \\[0.3em]
$\psi_r$ & Complexity term for Gibbs risk: $\left[\mathbb{E}_{\rho}[D_{\alpha_v}(\mathcal{Q}_v \| \mathcal{P}_v)] + D_\alpha(\rho \| \pi) + \ln(2\sqrt{m}/\delta)\right]/m$ \\[0.5em]
$\psi_e$ & Complexity term for joint error \\
$\psi_d$ & Complexity term for disagreement \\
$\lambda, \lambda_1, \lambda_2$ & Trade-off parameters in $(0,2)$ for PAC-Bayes-$\lambda$ inequality \\
$\gamma$ & Parameter for lower bound in PAC-Bayes-$\lambda$ inequality \\
\bottomrule
\end{tabular}
\end{table}
\subsection*{Bound Names}
\begin{table}[H]
\centering
\caption{Bound names and their descriptions}
\label{tab:bound_names}
\begin{tabular}{@{}p{0.2\textwidth}p{0.75\textwidth}@{}}
\toprule
\textbf{Symbol} & \textbf{Description} \\ 
\midrule
$\mathcal{R}$ & First-order bound (PAC-Bayes-$\lambda$) \\
$\mathcal{E}$ & First-order bound with joint error and disagreement \\
$\mathcal{K}$ & First-order bound (inverted KL) \\
$\mathcal{K}^u$ & First-order bound with joint error and disagreement (inverted KL) \\
$\mathcal{E}_{\text{II}}$ & Second-order bound with joint error \\
$\mathcal{R}_{\text{II}}$ & Second-order bound with disagreement (binary) \\
$\mathcal{K}_{\text{II}}$ & Second-order bound (inverted KL) \\
$\mathcal{K}^u_{\text{II}}$ & Second-order bound with disagreement (inverted KL, binary) \\
$\mathcal{C}_{\mathcal{D}}^S$ & C-Bound (population version) \\
$\mathcal{C}_{\rho}^S$ & C-Bound (empirical version) \\
$\mathcal{C}_{\mathcal{D}}^T$ & C-Tandem Oracle Bound (population version) \\
$\mathcal{C}_{\rho}^T$ & C-Tandem Oracle Bound (empirical version) \\
\bottomrule
\end{tabular}
\end{table}
\subsection*{Optimization-Related}
\begin{table}[H]
\centering
\caption{Optimization-related notation}
\label{tab:optimization}
\begin{tabular}{@{}p{0.2\textwidth}p{0.75\textwidth}@{}}
\toprule
\textbf{Symbol} & \textbf{Description} \\ 
\midrule
$\mathbf{B}_t(a)$ & Log-barrier extension function for constraint handling \\
$t$ & Barrier parameter \\
\bottomrule
\end{tabular}
\end{table}

\subsection*{Multi-view hierarchy}

\begin{figure}[H]
      \centering
      \begin{subfigure}[t]{0.49\textwidth}
          \centering
          \includegraphics[width=\linewidth]{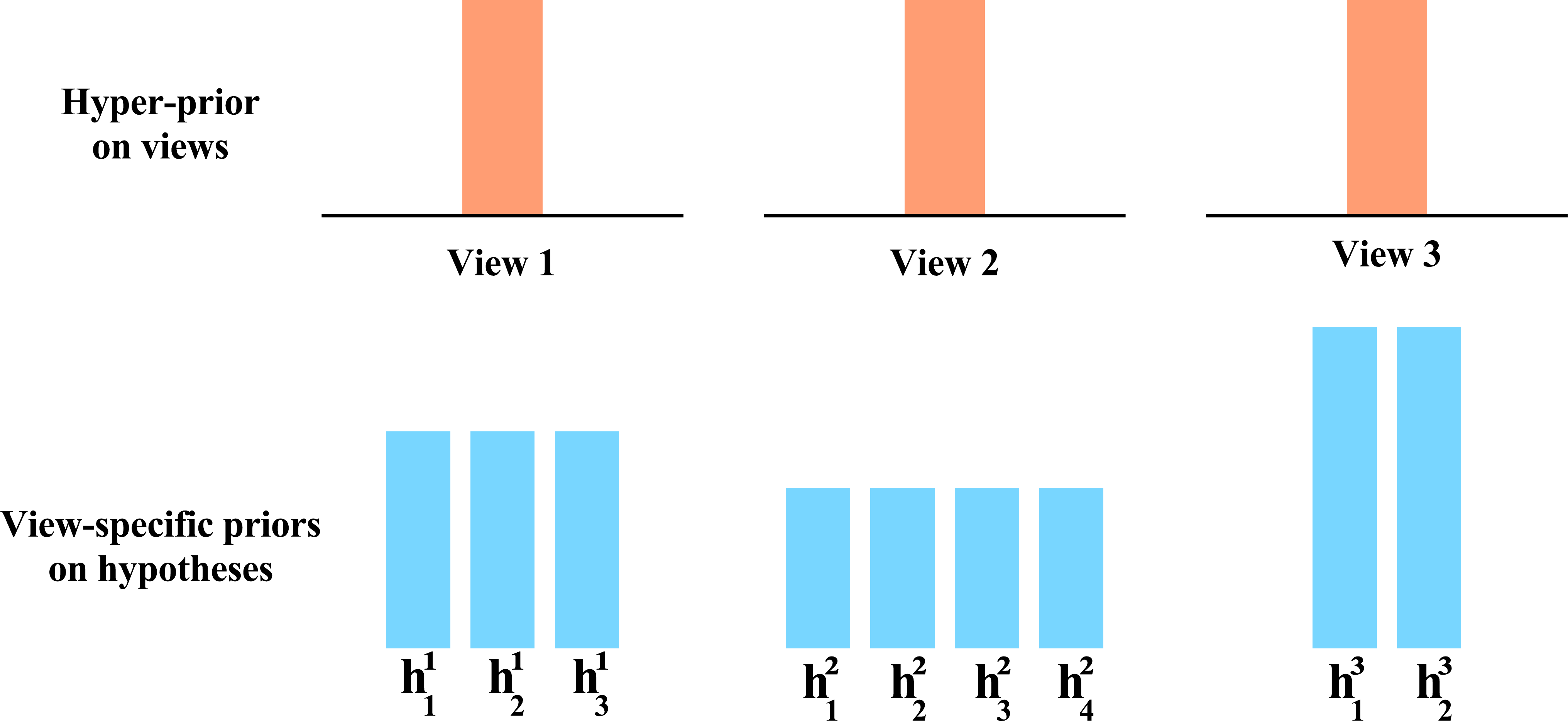}

          \caption{Before Learning}
          \label{fig:before_learning}
      \end{subfigure}
      \hfill
      \begin{subfigure}[t]{0.49\textwidth}
          \centering
          \includegraphics[width=\linewidth]{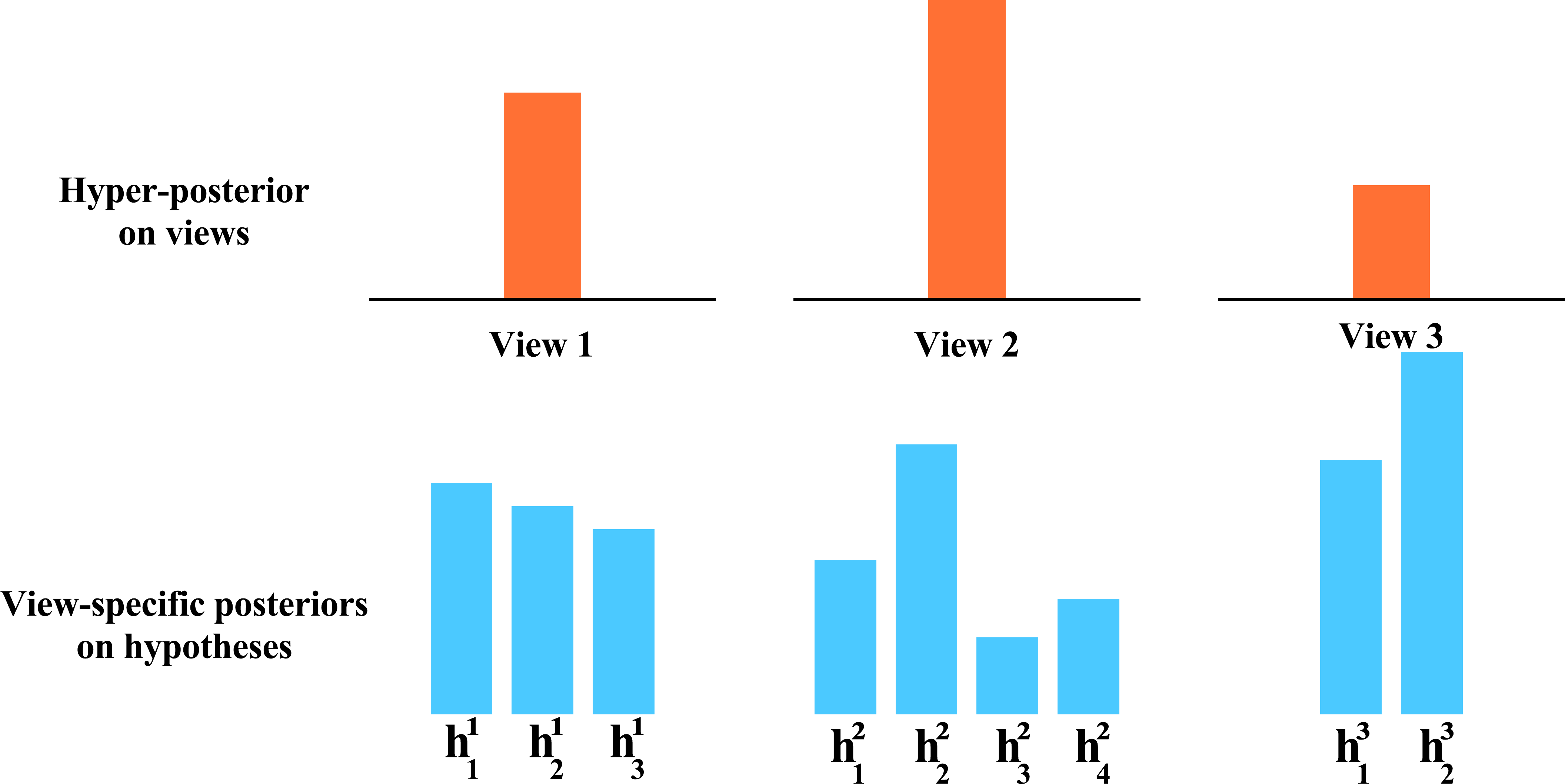}
          \caption{After Learning}
          \label{fig:after_learning}
      \end{subfigure}
      \caption{Hierarchical structure of multi-view distributions for $V=3$ views (adapted from \cite{Goyal17}). Each view has voters $\mathcal{H}_v = \{h_1^v, \dots, h_2^v\}$ with prior $\mathcal{P}_v$ before learning \textbf{(a, blue)} updated to a posterior $\mathcal{Q}_v$ after learning \textbf{(b, blue)}. And a hyper-prior $\pi$ over views \textbf{(a, orange)} is updated to hyper-posterior $\rho$ \textbf{(b, orange)}. Rectangle heights represent probability weights assigned to voters and views.}
      \label{fig:mv-hierarchy}
  \end{figure}

\section{Mathematical Tools}

\label{math_tools}

\begin{theorem}{Markov's Inequality.}\label{Markov's Inequality}
    For any random variable \(X\) such that \(\mathbb{E}[|X|] = \mu\), for any \(a > 0\), we have
\[
\mathbb{P}\{|X| \geq a\} \leq \frac{\mu}{a}.
\]

\end{theorem}

\begin{theorem}\label{Second Order Markov’s Inequality}{Second Order Markov’s Inequality.}
    For any random variable \(X\) with a finite second moment, i.e., \(\mathbb{E}[X^2] < \infty\), and for any \(a > 0\), we have
\[
\mathbb{P}\{X \geq a\} \leq \frac{\mathbb{E}[X^2]}{a^2}.
\]
\end{theorem}
\begin{theorem}{Jensen's Inequality.}\label{Jensen's Inequality}
    For any random variable \(X\), and for any concave function \(\varphi\), we have
\[
\varphi(\mathbb{E}[X]) \geq \mathbb{E}[\varphi(X)].
\]
Additionally, for any convex function \(\varphi\), we have
\[
\varphi(\mathbb{E}[X]) \leq \mathbb{E}[\varphi(X)].
\]
\end{theorem}

\begin{theorem}{Cantelli-Chebyshev Inequality.}
    For any random variable \(X\) such that \(\mathbb{E}[X] = \mu\) and \(\operatorname{Var}[X] = \sigma^2\), and for any \(a > 0\), we have
\[
\mathbb{P}\{X - \mu \geq a\} \leq \frac{\sigma^2}{\sigma^2 + a^2}.
\]

\end{theorem}

\begin{theorem}{Hölder's Inequality.}\label{Holder}
    For any random variables \(X\) and \(Y\), and for any positive real numbers \(p\) and \(q\) such that \(\frac{1}{p} + \frac{1}{q} = 1\), we have
\[
\mathbb{E}[|XY|] \leq (\mathbb{E}[|X|^p])^{\frac{1}{p}} (\mathbb{E}[|Y|^q])^{\frac{1}{q}}.
\]
\end{theorem}

\begin{proposition}\label{proposition_Hennequin}
    \begin{equation}
        D_\alpha(Q^2 \parallel P^2) = 2 \, D_\alpha(Q \parallel P)
    \end{equation}
    \begin{proof}

    We assume $Q$ and $P$ are probability distributions on a space $\mathcal{H}$. $Q^2$ and $P^2$ are product distributions defined on the product space $\mathcal{H}^2$ such that $Q^2(h_i, h_j) = Q(h_i) \cdot Q(h_j)$ and $P^2(h_i, h_j) = P(h_i) \cdot P(h_j)$. Rényi divergence of order $\alpha > 1$ is defined as:
    \[ D_\alpha(Q \parallel P) = \frac{1}{\alpha - 1} \ln \underset{h \sim P}{\mathbb{E}} \left[ \left(\frac{Q(h)}{P(h)}\right)^\alpha \right] \]
\begin{align}
D_\alpha(Q^2 \parallel P^2) &= \frac{1}{\alpha - 1}  \ln \underset{(h_i, h_j) \sim P^2}{\mathbb{E}} \left[ \left(\frac{Q^2(h_i, h_j)}{P^2(h_i, h_j)}\right)^\alpha \right] \\
&= \frac{1}{\alpha - 1} \ln \underset{(h_i, h_j) \sim P^2}{\mathbb{E}} \left[ \left(\frac{Q(h_i) Q(h_j)}{P(h_i) P(h_j)}\right)^\alpha \right] \\
&= \frac{1}{\alpha - 1} \ln \underset{(h_i, h_j) \sim P^2}{\mathbb{E}} \left[ \left(\frac{Q(h_i)}{P(h_i)}\right)^\alpha  \left(\frac{Q(h_j)}{P(h_j)}\right)^\alpha \right] \\
&= \frac{1}{\alpha - 1} \ln \left( \underset{h_i \sim P}{\mathbb{E}} \left[ \left(\frac{Q(h_i)}{P(h_i)}\right)^\alpha \right]  \underset{h_j \sim P}{\mathbb{E}} \left[ \left(\frac{Q(h_j)}{P(h_j)}\right)^\alpha \right] \right) \\
&= \frac{2}{\alpha - 1} \ln \underset{h \sim P}{\mathbb{E}} \left[ \left(\frac{Q(h)}{P(h)}\right)^\alpha \right] \\
&= 2 \, D_\alpha(Q \parallel P)
\end{align}

    \end{proof}
\end{proposition}

\section{Proofs of Multi-view PAC-Bayesian Bounds Based on Rényi Divergence}\label{proof of General Multiview PAC-Bayesian Theorem}

To demonstrate the three most popular PAC-Bayes approaches \citep{McAllester99, Catoni07, Seeger03, Langford05a} we rely on a general PAC-Bayesian theorem, as proposed by \cite{Germain09,Germain15a}, adapted to the multi-view learning framework with a two-hierarchy of distributions on views and voters \citep{Goyal17}. In our study, we integrate the Rényi divergence, as suggested by \cite{begin16}. An important step in PAC-Bayes proofs involves the use of a measure-change inequality, based on the Donsker-Varadhan inequality \citep{Donsker1975AsymptoticEO}. The lemma below extends this tool to our multi-view framework using the Rényi divergence.
\begin{lemma}[Multi-view Rényi change of measure inequality]\label{Multi-view Rényi change of measure}  

For any set of priors \(\{\mathcal{P}_v\}_{v=1}^V\) and any set of posteriors \(\{\mathcal{Q}_v\}_{v=1}^V\), for any hyper-prior distribution \(\pi\) over $[\![V]\!]$ and hyper-posterior distribution \(\rho\) over $[\![V]\!]$, and for any measurable function \(\phi : \mathcal{H}_v \rightarrow \mathbb{R}\), we have:
\begin{align}
    \underset{v \sim \rho}{\mathbb{E}} \, \underset{h \sim \mathcal{Q}_v}{\mathbb{E}} [\phi(h)]  \leq & \underset{v \sim \rho}{\mathbb{E}} [D_{\alpha_v}(\mathcal{Q}_v \| \mathcal{P}_v)] + D_\alpha(\rho \| \pi)  + \ln \left( \underset{v \sim \pi}{\mathbb{E}}  \underset{h \sim P_v}{\mathbb{E}} [e^{\phi(h)}]  \right)
\end{align}
\end{lemma}
where \(D_{\alpha}(Q \| P)\) is the Rényi divergence of order \(\alpha > 1 \) between the distributions \(Q\) and \(P\).

\begin{proof}
    \begin{align}
        \frac{\alpha}{\alpha - 1} \ln \left( \underset{v \sim \rho}{\mathbb{E}}\,\underset{h \sim \mathcal{Q}_v}{\mathbb{E}}[\phi(h)] \right) \leq \frac{\alpha}{\alpha - 1} \ln \left( \underset{v \sim \rho}{\mathbb{E}}\,\underset{h \sim \mathcal{P}_v}{\mathbb{E}} \left[ \frac{\mathcal{Q}_v(h)}{\mathcal{P}_v(h)} \phi(h) \right] \right)\label{starting-ineq}
    \end{align}
    
    by applying Hölder’s inequality~\ref{Holder} (as stated by \cite{begin16}), equality
    with $p = \alpha$ and $q = \frac{\alpha}{\alpha-1}$
     states that: 
    \begin{align}
        &\leq \frac{\alpha}{\alpha - 1} \ln \left(
        \left(\underset{v \sim \rho}{\mathbb{E}}\underset{h \sim \mathcal{P}_v}{\mathbb{E}}\left[\left(\frac{\mathcal{Q}_v(h)}{\mathcal{P}_v(h)}\right)^\alpha\right]\right)^{\frac{1}{\alpha}} \left(\underset{v \sim \rho}{\mathbb{E}}\underset{h \sim \mathcal{P}_v}{\mathbb{E}}\left[\phi(h)^{\frac{\alpha}{\alpha - 1}}\right]\right)^{\frac{\alpha - 1}{\alpha}} \right)
        \\ &= \frac{1}{\alpha - 1} \ln \left( \underset{v \sim \rho}{\mathbb{E}}\underset{h \sim \mathcal{P}_v}{\mathbb{E}}\left[\left(\frac{\mathcal{Q}_v(h)}{\mathcal{P}_v(h)}\right)^\alpha\right] \right)
         + \ln \left( \underset{v \sim \rho}{\mathbb{E}}\underset{h \sim \mathcal{P}_v}{\mathbb{E}}\left[\phi(h)^{\frac{\alpha}{\alpha-1}}\right] \right)
    \end{align} 
    by applying Jensen's inequality to the concave function $\ln()$ in the term $\ln \left( \underset{v \sim \rho}{\mathbb{E}} \underset{h \sim \mathcal{P}_v}{\mathbb{E}} \left[ \left(\frac{\mathcal{Q}_v(h)}{\mathcal{P}_v(h)} \right)^\alpha \right] \right)$:
     \begin{align}
         \geq \underset{v \sim \rho}{\mathbb{E}} \left[\ln \left( \underset{h \sim \mathcal{P}_v}{\mathbb{E}} \left[ \left(\frac{\mathcal{Q}_v(h)}{\mathcal{P}_v(h)} \right)^\alpha \right] \right)\right] + \ln \left( \underset{v \sim \rho}{\mathbb{E}}\,\underset{h \sim \mathcal{P}_v}{\mathbb{E}}\left[\phi(h)^{\frac{\alpha}{\alpha-1}}\right] \right)\label{flipped-ineq}
     \end{align}

Now that the sign has been flipped, we need to check that the left-hand side of~\ref{starting-ineq} is still less than~\ref{flipped-ineq}. First, for $\alpha>1$ the function $f(x)=x^{\frac{\alpha}{\alpha - 1}}$ is convex, according to \cite{Matkowski1997}, indeed, $\left(\frac{\alpha}{\alpha - 1}\right)^2-\frac{\alpha}{\alpha - 1}>0$.

Then, applying Jensen's inequality to left-hand side of~\ref{starting-ineq} gives:
\begin{align}
    &\ln \left(\left[ \underset{v \sim \rho}{\mathbb{E}}\,\underset{h \sim \mathcal{Q}_v}{\mathbb{E}}[\phi(h)] \right]^{\frac{\alpha}{\alpha - 1}}\right) \leq \ln \left( \underset{v \sim \rho}{\mathbb{E}}\,\underset{h \sim \mathcal{P}_v}{\mathbb{E}}\left[\phi(h)^{\frac{\alpha}{\alpha-1}}\right] \right)
\end{align}

Therefore, the left-hand side of~\ref{starting-ineq} is indeed less than~\ref{flipped-ineq}, which allows us to continue our proof,
\begin{align}
    (\ref{flipped-ineq}) &=  \underset{v \sim \rho}{\mathbb{E}} \left[D_{\alpha_v}(\mathcal{Q}_v \| \mathcal{P}_v)\right] + \ln \left( \underset{v \sim \rho}{\mathbb{E}}\,\underset{h \sim \mathcal{P}_v}{\mathbb{E}}\left[\phi(h)^{\frac{\alpha}{\alpha-1}}\right] \right)
\end{align}

Applying Jensen's inequality once more to the convex function $f(x)=x^{\frac{\alpha}{\alpha - 1}}$ in the second term yields the following.
 \begin{align}
& \leq \underset{v \sim \rho}{\mathbb{E}} \left[D_{\alpha_v}(\mathcal{Q}_v \| \mathcal{P}_v)\right]  + {\frac{\alpha}{\alpha-1}}\ln \left(  \underset{h \sim \mathcal{P}_v}{\mathbb{E}}\left[\sum_{v \in [\![V]\!]} \phi(h) \frac{\rho(v)}{\pi(v)} \pi(v)\right] \right) \\ & = \underset{v \sim \rho}{\mathbb{E}} \left[D_{\alpha_v}(\mathcal{Q}_v \| \mathcal{P}_v)\right]  + {\frac{\alpha}{\alpha-1}}\ln \left(  \underset{h \sim \mathcal{P}_v}{\mathbb{E}}\underset{v \sim \pi}{\mathbb{E}}\left[ \frac{\rho(v)}{\pi(v)}\phi(h)\right] \right) \end{align} 
 
by applying Hölder’s inequality~\ref{Holder}, equality with $p = \alpha$ and $q = \frac{\alpha}{\alpha-1}$ states that: 
\begin{align}
    &= \underset{v \sim \rho}{\mathbb{E}} \left[D_{\alpha_v}(\mathcal{Q}_v \| \mathcal{P}_v)\right]  + \frac{\alpha}{\alpha - 1} \ln \left(
    \left(\underset{h \sim \mathcal{P}_v}{\mathbb{E}}\,\underset{v \sim \pi}{\mathbb{E}}\left[\left(\frac{\rho(v)}{\pi(v)}\right)^\alpha\right]\right)^{\frac{1}{\alpha}}  \left(\underset{h \sim \mathcal{P}_v}{\mathbb{E}}\,\underset{v \sim \pi}{\mathbb{E}}\left[\phi(h)^{\frac{\alpha}{\alpha - 1}}\right]\right)^{\frac{\alpha - 1}{\alpha}} \right) \\ 
    &= \underset{v \sim \rho}{\mathbb{E}} \left[D_{\alpha_v}(\mathcal{Q}_v \| \mathcal{P}_v)\right] \nonumber  + \frac{1}{\alpha - 1} \ln 
    \left(\underset{h \sim \mathcal{P}_v}{\mathbb{E}}\,\underset{v \sim \pi}{\mathbb{E}}\left[\left(\frac{\rho(v)}{\pi(v)}\right)^\alpha\right]\right)  + \ln \left(\underset{h \sim \mathcal{P}_v}{\mathbb{E}}\,\underset{v \sim \pi}{\mathbb{E}}\left[\phi(h)^{\frac{\alpha}{\alpha - 1}}\right]\right)\\ 
    &= \underset{v \sim \rho}{\mathbb{E}} \left[D_{\alpha_v}(\mathcal{Q}_v \| \mathcal{P}_v)\right] + D_\alpha(\rho \| \pi)  +
    \ln \left(\underset{h \sim \mathcal{P}_v}{\mathbb{E}}\,\underset{v \sim \pi}{\mathbb{E}}\left[\phi(h)^{\frac{\alpha}{\alpha - 1}}\right]\right)
\end{align}
We obtain the following inequality:
\begin{align}
    &\frac{\alpha}{\alpha - 1} \ln \left( \underset{v \sim \rho}{\mathbb{E}}\,\underset{h \sim \mathcal{Q}_v}{\mathbb{E}}[\phi(h)] \right)  \leq \underset{v \sim \rho}{\mathbb{E}} \left[D_{\alpha_v}(\mathcal{Q}_v \| \mathcal{P}_v)\right]  + D_\alpha(\rho \| \pi)  +
    \ln \left(\underset{h \sim \mathcal{P}_v}{\mathbb{E}}  \underset{v \sim \pi}{\mathbb{E}}\left[\phi(h)^{\frac{\alpha}{\alpha - 1}}\right]\right)
\end{align}
We apply Jensen’s inequality on the concave function $\ln(\cdot)$ of the left-hand side inequality above and with $\phi(h)$ replaced by $e^{\frac{\alpha-1}{1}\phi(h)}$ gives rise to the following looser change of measure inequality that is also based on the Rényi divergence:
\begin{align}
    \underset{v \sim \rho}{\mathbb{E}} \, \underset{h \sim \mathcal{Q}_v}{\mathbb{E}} [\phi(h)]  \leq \underset{v \sim \rho}{\mathbb{E}} [D_{\alpha_v}(\mathcal{Q}_v \| \mathcal{P}_v)] + D_\alpha(\rho \| \pi)  + \ln \left( \underset{v \sim \pi}{\mathbb{E}}  \underset{h \sim \mathcal{P}_v}{\mathbb{E}} [e^{\phi(h)}]  \right)
\end{align}
\end{proof}

Based on Lemma~\ref{Multi-view Rényi change of measure}, we introduce a multi-view general PAC-Bayesian theorem with Rényi divergence,
it takes the form of an upper bound on the deviation between the true risk $\mathfrak{R}_{\mathcal{D}}^{\mathcal{V}}$ and empirical risk $\hat{\mathfrak{R}}_{S}^{\mathcal{V}}$ of
the Gibbs classifier, according to a convex function $\Upsilon:[0, 1] \times [0, 1] \rightarrow \mathbb{R}$, we have:

\begin{theorem}[General Multiview PAC-Bayesian Theorem based on the Rényi Divergence]\label{General_Multiview_PAC_Bayesian_Theorem_based_on_the_Renyi_Divergence} Let \( V \geq 2 \) be the number of views. For any distribution \( \mathcal{D} \) on \( \mathcal{X} \times \mathcal{Y} \), for any set of prior distributions \( \{\mathcal{P}_v\}_{v=1}^V \) and any set of posteriors \(\{\mathcal{Q}_v\}_{v=1}^V\), for any hyper-prior distribution \( \pi \) over \( [\![V]\!] \), and for any convex function \( \Upsilon: [0, 1] \times [0, 1] \rightarrow \mathbb{R} \), with probability at least $1-\delta$ over a random draw of
a sample $S$, we have
\begin{align}
    &\Upsilon \left(\hat{\mathfrak{R}}_S,\mathfrak{R}_{\mathcal{D}} \right) 
     \leq \frac{1}{m} \left[  \underset{v \sim \rho}{\mathbb{E}} \left[ D_{\alpha_v}(\mathcal{Q}_{v}\| \mathcal{P}_v) \right] + D_\alpha(\rho \| \pi)  + \ln \left( \frac{1} {\delta} \underset{S\sim\mathcal D^{m}}{\mathbb E} \underset{v \sim \pi}{\mathbb{E}} \underset{h \sim \mathcal{P}_v}{\mathbb{E}} \left[ e^{m \Upsilon(\hat{R}_S(h), R_{\mathcal{D}}(h))} \right]  \right) \right]
\end{align}
\end{theorem}
where 
$\hat{R}_S(h) \triangleq \frac{1}{m} \sum_{i=1}^{m} \ell(h(x_i), y_i)$
is the empirical risk, and the true risk 
$R_{\mathcal{D}}(h) \triangleq \underset{(x, y) \sim \mathcal{D}}{\mathbb{E}} \left[ \ell(h(x), y)\right]$.

\begin{proof}
 Note that the random variable
\(
Z \triangleq \mathbb E_{v\sim\pi}\mathbb E_{h\sim \mathcal{P}_v}
         e^{m\Upsilon(\hat R_S(h),R_{\mathcal D}(h))}
\)
is non‑negative.  For any $\delta\in(0,1]$, by Markov’s
inequality~\ref{Markov's Inequality}, we have
\begin{align}
    &\Pr_{S\sim\mathcal D^{m}}\!
   \left\{ Z\le\frac {1}{\delta}\, \underset{S\sim\mathcal D^{m}}{\mathbb E} [Z] \right\}
   \;\ge\;1-\delta, \text{ by taking the logarithm on each side of the inequality,
}
   \\ & \Pr_{S\sim\mathcal D^{m}}\!
   \left\{\ln Z\le \ln \left(\frac {1}{\delta}\, \underset{S\sim\mathcal D^{m}}{\mathbb E} [Z]\right) \right\}
   \;\ge\;1-\delta
\end{align}

We apply the Multi-view Rényi change of measure inequality~\ref{Multi-view Rényi change of measure} on the left side of inequality, with $\phi(h) \triangleq \;m\,
   \Upsilon\bigl(\hat R_S(h),R_{\mathcal D}(h)\bigr)$. We then use Jensen’s inequality~\ref{Jensen's Inequality}, exploiting the convexity of $\Upsilon$:
   \begin{align}
       \forall \, \mathcal{Q}_{v} \,\text{on} \,\mathcal{H}_{v}, \ln Z &\geq  \underset{v \sim \rho}{\mathbb{E}} \, \underset{h \sim \mathcal{Q}_v}{\mathbb{E}} [\phi(h)] - \underset{v \sim \rho}{\mathbb{E}} [D_{\alpha_v}(\mathcal{Q}_v \| \mathcal{P}_v)] + D_\alpha(\rho \| \pi) 
       \\& \geq \underset{v \sim \rho}{\mathbb{E}} \, \underset{h \sim \mathcal{Q}_v}{\mathbb{E}} [m\,
   \Upsilon\bigl(\hat R_S(h),R_{\mathcal D}(h)\bigr)] - \underset{v \sim \rho}{\mathbb{E}} [D_{\alpha_v}(\mathcal{Q}_v \| \mathcal{P}_v)] + D_\alpha(\rho \| \pi) 
   \\& \geq m\,
   \Upsilon\bigl(\mathfrak{\hat{R}}_S,\mathfrak{R}_{\mathcal D}\bigr) - \underset{v \sim \rho}{\mathbb{E}} [D_{\alpha_v}(\mathcal{Q}_v \| \mathcal{P}_v)] + D_\alpha(\rho \| \pi) 
   \end{align}

We therefore have, $\forall \, \mathcal{Q}_{v}$ on $\mathcal{H}$
\begin{align}
    \Pr_{S\sim\mathcal D^{m}}\!
   \Bigg\{&m\,
   \Upsilon\bigl(\mathfrak{\hat{R}}_S,\mathfrak{R}_{\mathcal D}\bigr) - \underset{v \sim \rho}{\mathbb{E}} [D_{\alpha_v}(\mathcal{Q}_v \| \mathcal{P}_v)] + D_\alpha(\rho \| \pi) \nonumber \\
        &\quad \leq \ln \left( \frac{1} {\delta} \underset{S\sim\mathcal D^{m}}{\mathbb E} \underset{v \sim \pi}{\mathbb{E}} \underset{h \sim \mathcal{P}_v}{\mathbb{E}} \left[ e^{m \Upsilon(\hat{R}_S(h), R_{\mathcal{D}}(h))} \right]  \right) \Bigg\} \geq 1-\delta
\end{align}

The result follows by straightforward calculations with probability at least $1-\delta$ over a random draw of
a sample $S$.
\end{proof}

\begin{theorem}\label{general_theorem_disagreement_joint}
    Let $V \geq 2$ be the number of views. For any distribution $\mathcal{D}$ on $\mathcal{X} \times \mathcal{Y}$, for any set of prior distributions $\{\mathcal{P}_v\}_{v=1}^V$ and any set of posteriors \(\{\mathcal{Q}_v\}_{v=1}^V\), for any hyper-prior distribution $\pi$ over $[\![V]\!]$, for any convex function $\Upsilon: [0, 1] \times [0, 1] \rightarrow \mathbb{R}$, with probability at least $1-\delta$ over a random draw of
a sample $S$, we have
\begin{align}
\Upsilon \left(\hat{\mathfrak{A}}^{\mathcal{V}}_S,  \mathfrak{A}^{\mathcal{V}}_{\mathcal{D}}\right)  
&\leq \frac{1}{m} \left[
  2 \left( 
    \underset{v \sim \rho}{\mathbb{E}} \left[ D_{\alpha_v}(\mathcal{Q}_{v} \| \rho_v) \right]  
    + D_\alpha(\rho \| \pi) 
  \right) \right. \nonumber \\
&\quad \left. + \ln \left( 
    \frac{1}{\delta} \, \underset{S \sim \mathcal{D}^m}{\mathbb{E}}  
    \underset{(v,v') \sim \pi^2}{\mathbb{E}}  
    \underset{(h,h') \sim \mathcal{P}^{2}_v}{\mathbb{E}} \left[
      e^{m\Upsilon(\hat{\mathfrak{A}}_S(h,h'), \mathfrak{A}_{\mathcal{D}}(h,h'))}
    \right] 
  \right)
\right]
\end{align}

\end{theorem}
where $\mathfrak{A}^{\mathcal{V}}$ can be either $d^{\mathcal{V}}$, $e^{\mathcal{V}}$. 

$\hat{\mathfrak{A}}_S(h,h'), \mathfrak{A}_{\mathcal{D}}(h,h')$ can be defined as follows:
\begin{itemize}
    \item Expected Disagreement and its empirical counterpart : 
    \begin{align*}
    \mathfrak{A}_{\mathcal{D}}(h,h') &\triangleq d_{\mathcal{D}_{\mathcal{X}}}(h,h') = \underset{(h,h')\sim \mathcal{Q}^2}{\mathbb{E}}\underset{x\sim \mathcal{D}_{\mathcal{X}}}{\mathbb{E}}\,\big[\ell(h(\bm{x}),h'(\bm{x}))\big], \\  \hat{\mathfrak{A}}_{S}(h,h') &\triangleq \hat{d}_{S} = \frac{1}{m}\sum_{i=1}^{m} \underset{(h,h')\sim \mathcal{Q}^2}{\mathbb{E}}\big[\ell(h(\bm{x}_i),h'(\bm{x}_i))\big] 
    \end{align*}
    where \( d \) is the disagreement between two hypotheses \( h \) and \( h' \).

    \item  Expected joint error its empirical counterpart : 
    \begin{align*}
    &\mathfrak{A}_{\mathcal{D}}(h,h') \triangleq e_{\mathcal{D}}(h,h') = \underset{(h,h')\sim \mathcal{Q}^2}{\mathbb{E}}\underset{x\sim \mathcal{D}}{\mathbb{E}}\,\big[\ell(h(\bm{x}),y) \times \ell(h'(\bm{x}),y)\big]
        \\ &\hat{\mathfrak{A}}_{S}(h,h') \triangleq\hat{e}_{S}(h,h') = \frac{1}{m}\sum_{i=1}^{m} \underset{(h,h')\sim \mathcal{Q}^2}{\mathbb{E}}\big[\ell(h(\bm{x}_i),y_i) \times \ell(h'(\bm{x}_i,y_i))\big].
    \end{align*}
    where \( e \) is the joint error between two hypotheses \( h \) and \( h' \).
\end{itemize}
\begin{proof}
    First, we apply the exact same steps as in the proof of Theorem~\ref{General_Multiview_PAC_Bayesian_Theorem_based_on_the_Renyi_Divergence}. Then, we use the fact that $D_{\alpha}(Q^2\|P^2) = 2 \, D_{\alpha}(Q\|P)$ via the Proposition~\ref{proposition_Hennequin}.
\end{proof}

We provide specialization of our multiview theorem to the most popular PAC-Bayesian approaches. To do so, we follow the same principles as Germain et al. \cite{Germain09,Germain15a}.

By using the Kullback-Leibler divergence between two Bernoulli distributions with success probabilities $a$ and $b$ as the function $\Upsilon(a, b)$ to measure the deviation between the empirical risk and the true risk, we can derive a bound similar to those presented by \cite{Seeger03, Langford05a}. Specifically, we apply Theorem~\ref{General_Multiview_PAC_Bayesian_Theorem_based_on_the_Renyi_Divergence} with the following setup:

\begin{corollary}\label{KL_Bound}
Let $V \geq 2$ be the number of views. For any distribution $\mathcal{D}$ on $\mathcal{X} \times \mathcal{Y}$, for any set of prior distributions $\{\mathcal{P}_v\}_{v=1}^V$ and any set of posteriors \(\{\mathcal{Q}_v\}_{v=1}^V\), for any hyper-prior distribution $\pi$ over $[\![V]\!]$, with probability at least $1-\delta$ over a random draw of
a sample $S$, we have
\begin{align}
&\textnormal{KL}\left(\hat{\mathfrak{R}}^{\mathcal{V}}_S \middle\| \mathfrak{R}^{\mathcal{V}}_{\mathcal{D}}\right)  \leq  \frac{\underset{v \sim \rho}{\mathbb{E}}\left[D_{\alpha_v}(\mathcal{Q}_{v} \| \mathcal{P}_v)\right] + D_\alpha(\rho \| \pi)   
    + \ln{\frac{2\sqrt{m}}{\delta}}}{m}, 
\end{align}
\end{corollary}
\begin{proof}

The result follows from Theorem~\ref{General_Multiview_PAC_Bayesian_Theorem_based_on_the_Renyi_Divergence} by taking
\[
  \Upsilon(a,b) = \text{KL}(a,b)
\]
and upper-bounding
\[
  \underset{S\sim\mathcal D^m}{\mathbb E}\,
  \underset{v\sim\pi}{\mathbb E}\,
  \underset{h\sim\mathcal P_v}{\mathbb E}\,
  e^{\,m\,\text{KL}\bigl(\hat{R}_S(h),R_{\mathcal D}(h)\bigr)}.
\]
By considering \(\hat{R}_S(h)\) as a random variable following a \(\mathrm{Binomial}(m, R_{\mathcal D}(h))\) distribution, we can then show that:
 \begin{align}
&\underset{S\sim\mathcal D^m}{\mathbb E}\,
 \underset{v\sim\pi}{\mathbb E}\,
 \underset{h\sim\mathcal P_v}{\mathbb E}\,
 e^{\,m\,\text{KL}\bigl(\hat{R}_S(h),R_{\mathcal D}(h)\bigr)}
\nonumber\\
&\quad=\;\underset{v\sim\pi}{\mathbb E}\,
 \underset{h\sim\mathcal P_v}{\mathbb E}\,
 \underset{S\sim\mathcal D^m}{\mathbb E}\!
 \Biggl(\frac{\hat{R}_S(h)}{R_{\mathcal D}(h)}\Biggr)^{m\,\hat{R}_S(h)}
 \Biggl(\frac{1-\hat{R}_S(h)}{1-R_{\mathcal D}(h)}\Biggr)^{m\bigl(1-\hat{R}_S(h)\bigr)}
\nonumber\\
&\quad=\;\underset{v\sim\pi}{\mathbb E}\,
 \underset{h\sim\mathcal P_v}{\mathbb E}\,
 \sum_{k=0}^m
 \Pr_{S\sim\mathcal D^m}\bigl\{\hat{R}_S(h)=\tfrac{k}{m}\bigr\}\,
 \Biggl(\frac{k/m}{R_{\mathcal D}(h)}\Biggr)^{k}
 \Biggl(\frac{1-k/m}{1-R_{\mathcal D}(h)}\Biggr)^{m-k}
\nonumber\\
&\quad=\;\sum_{k=0}^m
 \binom{m}{k}\,
 \Bigl(\tfrac{k}{m}\Bigr)^{k}
 \Bigl(1-\tfrac{k}{m}\Bigr)^{\,m-k}
\; \leq 2\sqrt{m}\; \text{ (Maurer \citep{maurer2004notepacbayesiantheorem})}.
\end{align}

\end{proof}

We derive here the specialization of our multi-view PAC-Bayesian theorem to \cite{McAllester99,McAllester03a}’s point of view.
\begin{corollary}

Let $V \geq 2$ be the number of views. For any distribution $\mathcal{D}$ on $\mathcal{X} \times \mathcal{Y}$, for any set of prior distributions $\{\mathcal{P}_v\}_{v=1}^V$ and any set of posteriors \(\{\mathcal{Q}_v\}_{v=1}^V\), for any hyper-prior distribution $\pi$ over $[\![V]\!]$, with probability at least $1-\delta$ over a random draw of
a sample $S$, we have
\begin{align}
 \mathfrak{R}^{\mathcal{V}}_{\mathcal{D}}   \leq  \mathfrak{\hat{R}}^{\mathcal{V}}_{S}    + \sqrt{\frac{ \underset{v \sim \rho}{\mathbb{E}} \left[D_{\alpha_v}(\mathcal{Q}_{v} \| \mathcal{P}_v)\right] +  D_{\alpha}(\rho \| \pi) + \ln \frac{2\sqrt{m}}{ \delta}}{2m}}  
\end{align}
\end{corollary}
\begin{proof}
     The result is derived from Corollary~\ref{KL_Bound}, by using $2(q-p)^2 \leq \text{KL}(q\|p)$ (Pinsker’s
inequality), and isolating $\mathfrak{R}^{\mathcal{V}}_{\mathcal{D}}$ in the obtained inequality. 
\end{proof}

To establish a generalization bound following \cite{Catoni07}'s perspective—given a convex function \( F \) and a constant \( C > 0 \)—we define the deviation between the empirical Gibbs risk and the true Gibbs risk as \( \Upsilon(a, b) = F(b) - C a \) \citep{Germain09, Germain15a}. This leads us to the following generalization bound, 
\begin{corollary}

Let $V \geq 2$ be the number of views. For any distribution $\mathcal{D}$ on $\mathcal{X} \times \mathcal{Y}$, for any set of prior distributions $\{\mathcal{P}_v\}_{v=1}^V$ and any set of posteriors $\{\mathcal{Q}_v\}_{v=1}^V$, for any hyper-prior distribution $\pi$ over $[\![V]\!]$, for all $C > 0$, with probability at least $1-\delta$ over a random draw of
a sample $S$, we have
\begin{align}
 \mathfrak{R}_{\mathcal{D}}^{\mathcal{V}} \leq \frac{1}{1 - e^{-C}} \left\{1 - \exp\left[-\left(C \,\mathfrak{\hat{R}}_{S}^{\mathcal{V}} +\frac{1}{m}\left[\underset{v \sim \rho}{\mathbb{E}}\left[ D_{\alpha_v}(\mathcal{Q}_{v} \| \mathcal{P}_v)\right] +  D_{\alpha}(\rho \| \pi)  + \ln \frac{1}{\delta}\right]\right)\right]\right\}.
\end{align}
\end{corollary}

\begin{proof}
    The result follows from Theorem~\ref{General_Multiview_PAC_Bayesian_Theorem_based_on_the_Renyi_Divergence} by taking
\[
  \Upsilon(a,b) = F(b) - C\,a,
\]
for a convex \(F\) and \(C>0\), and by upper‐bounding
\[
  \underset{S\sim\mathcal D^m}{\mathbb E}\,
  \underset{v\sim\pi}{\mathbb E}\,
  \underset{h\sim\mathcal P_v}{\mathbb E}\,
  e^{\,m\,\Upsilon\bigl(\hat{R}_S(h),R_{\mathcal D}(h)\bigr)}.
\]
By considering \(\hat{R}_S(h)\) as a random variable following a \(\mathrm{Binomial}(m, R_{\mathcal D}(h))\) distribution, we can then show that:
\begin{align}
  &\underset{S\sim\mathcal D^m}{\mathbb E}\,
   \underset{v\sim\pi}{\mathbb E}\,
   \underset{h\sim\mathcal P_v}{\mathbb E}\,
   e^{\,m\,\Upsilon\bigl(\hat{R}_S(h),R_{\mathcal D}(h)\bigr)}
  \nonumber\\
  &=\;
   \underset{v\sim\pi}{\mathbb E}\,
   \underset{h\sim\mathcal P_v}{\mathbb E}\,
   e^{\,m\,F\bigl(R_{\mathcal D}(h)\bigr)}
   \;\underset{S\sim\mathcal D^m}{\mathbb E}\,
   e^{-C\,m\,\hat{R}_S(h)}
  \nonumber\\
  &=\;
   \underset{v\sim\pi}{\mathbb E}\,
   \underset{h\sim\mathcal P_v}{\mathbb E}\,
   e^{\,m\,F\bigl(R_{\mathcal D}(h)\bigr)}
   \sum_{k=0}^m
   \Pr_{S\sim\mathcal D^m}\bigl\{\hat{R}_S(h)=\tfrac{k}{m}\bigr\}\,
   e^{-C\,k}
  \nonumber\\
  &=\;
   \underset{v\sim\pi}{\mathbb E}\,
   \underset{h\sim\mathcal P_v}{\mathbb E}\,
   e^{\,m\,F\bigl(R_{\mathcal D}(h)\bigr)}
   \sum_{k=0}^m
   \binom{m}{k}\,
   R_{\mathcal D}(h)^k\,(1-R_{\mathcal D}(h))^{m-k}\,e^{-C\,k}.
\end{align}

\end{proof}

\begin{corollary}\label{Hennequin_Bound_Macallester_dis}
    Let $V \geq 2$ be the number of views. For any distribution $\mathcal{D}$ on $\mathcal{X} \times \mathcal{Y}$, for any set of prior distributions $\{\mathcal{P}_v\}_{v=1}^V$ and any set of posteriors \(\{\mathcal{Q}_v\}_{v=1}^V\), for any hyper-prior distribution $\pi$ over $[\![V]\!]$, with probability at least $1-\delta$ over a random draw of
a sample $S$, we have 
\begin{align}
&\textnormal{KL}\left(\hat{\mathfrak{A}}^{\mathcal{V}}_S \middle\|  \mathfrak{A}^{\mathcal{V}}_{\mathcal{D}}\right)  \leq  \frac{2\left[\underset{v \sim \rho}{\mathbb{E}}\left[D_{\alpha_v}(\mathcal{Q}_{v} \| \mathcal{P}_v)\right] + D_\alpha(\rho \| \pi)\right]   
    + \ln{\frac{2\sqrt{m}}{\delta}}}{m}. 
\end{align}
\end{corollary}
\begin{proof}
    The proof follows the same methodology as Corollary~\ref{KL_Bound}; however, we apply the Theorem~\ref{general_theorem_disagreement_joint} with Proposition~\ref{proposition_Hennequin}.
\end{proof}

\section{Multi-view Bounds in Expectation
}\label{Multi-view Bounds in Expectation}

\cite{Alquier_2024} distinguish clearly between expectation‐ and probability‐based PAC–Bayes bounds (see Sec. 2.4). They argue that bounds in expectation are simpler and more tractable than high-probability PAC bounds, but only guarantee average performance over datasets rather than for each sample, so they don’t fully fit the classic Probably Approximately Correct framework. \cite{Dalalyan_2008} called them Expectedly Approximately Correct (EAC–Bayes), and \cite{pmlr-v134-grunwald21a} later dubbed them Mean Approximately Correct (MAC–Bayes). To avoid confusion, Alquier prefers the straightforward label “PAC–Bayes bound in expectation.”

\textbf{Multi-view PAC-Bayes bounds: in \textit{expectation} and in \textit{probability}.}  
The general PAC-Bayes theorem of \cite{germain2015a} provides a classical foundation for deriving PAC-Bayesian bounds in \textit{probability}, using Markov's inequality. In contrast, \cite{begin16} proposed to replace Markov's inequality with Jensen's inequality, which allows the derivation of tighter PAC-Bayes bounds stated in \textit{expectation}, involving the R\'enyi divergence.

Building on this idea, \cite{Goyal17} formalized a general \textit{multi-view} PAC-Bayes theorem in \textit{expectation} (Theorem 3), which unifies three different PAC-Bayesian bounds within a single framework. This formulation is directly inspired by the trick introduced by Bégin and extends the applicability of PAC-Bayes analysis beyond traditional probabilistic guarantees.

In the following, we compare these bounds stated in \textcolor{blue}{\textit{expectation}} with our own result stated in \textcolor{red}{\textit{probability}}. These differences are highlighted by the color annotations used in the equations below.

\medskip\noindent
\begin{corollary}
Let $V \geq 2$ be the number of views. For any distribution $\mathcal{D}$ on $\mathcal{X} \times \mathcal{Y}$, for any set of prior distributions $\{P_v\}_{v=1}^V$, for any hyper-prior distribution $\pi$ over $[\![V]\!]$, with probability at least $1-\delta$ over a random draw of
a sample $S$, we have
\begin{align}
&\textnormal{KL}\left(\textcolor{blue}{\underset{S\sim \mathcal{D}^m}{\mathbb{E}}}\hat{\mathfrak{R}}^{\mathcal{V}}_S \middle\| \textcolor{blue}{\underset{S\sim \mathcal{D}^m}{\mathbb{E}}}\mathfrak{R}^{\mathcal{V}}_{\mathcal{D}}\right) \leq \frac{\textcolor{blue}{\underset{S\sim \mathcal{D}^m}{\mathbb{E}}}\underset{v \sim \rho}{\mathbb{E}}\left[\textnormal{KL}(\mathcal{Q}_{v} \| \mathcal{P}_v)\right] + \textcolor{blue}{\underset{S\sim \mathcal{D}^m}{\mathbb{E}}}\textnormal{KL}(\rho \| \pi) + \ln{\frac{2\sqrt{m}}{\textcolor{blue}{\cancel{\delta}}}}}{m}, \textnormal{\cite{Goyal17}}\\
&\textcolor{red}{\Pr_{S \sim \mathcal{D}^m}}\left\{\textnormal{KL}\left( \hat{\mathfrak{R}}^{\mathcal{V}}_S \middle\| \mathfrak{R}^{\mathcal{V}}_{\mathcal{D}}\right) \leq \frac{\underset{v \sim \rho}{\mathbb{E}}\left[D_\alpha(\mathcal{Q}_{v} \| \mathcal{P}_v)\right] + D_{\alpha_v}(\rho \| \pi) + \ln{\frac{2\sqrt{m}}{\delta}}}{m}\right\} \textcolor{red}{\geq 1- \delta}
\end{align}
\end{corollary}

\section{PAC-Bayes-$\lambda$ inequality}\label{PAC_Bayes_lambda_proof}

\begin{theorem}{Multi-view PAC-Bayes-$\lambda$ Inequality, in the idea of  \cite{Thiemann17}'s theorem.} Under the same assumption of Corollary~\ref{PAC-Bayes-kl Inequality based on Rényi Divergence} and for all $\lambda \in (0, 2)$ and $\gamma> 0$ we have:
\begin{align}
    & \mathfrak{R}_{\mathcal{D}}^{\mathcal{V}} \leq \frac{\mathfrak{\hat{R}}_{S}^{\mathcal{V}}}{1-\frac{\lambda}{2}} + \frac{\psi_r}{\lambda(1-\frac{\lambda}{2})},  \mathfrak{R}_{\mathcal{D}}^{\mathcal{V}} \geq \left(1-\frac{\gamma}{2}\right)\mathfrak{\hat{R}}_{S}^{\mathcal{V}} - \frac{\psi_r}{\gamma}.
\end{align}

\end{theorem}

    We provide a proof of the upper and lower bounds in Theorem~\ref{multi_view_thieman}. Both bounds have been demonstrated by \cite{Masegosa2020} and \cite{Thiemann17} in single view. Thus, the proof follows the same steps as those proposed by them.

We use the following version of refined Pinsker’s inequality:
\begin{align}\label{Pinsker_inequality_1}
    \text{for}\, p < q,  \textnormal{KL}(p\Vert q) \geq \frac{(p-q)^{2}}{2p}
\end{align}

By applying Inequality~\ref{Pinsker_inequality_1}, the equation in Theorem~\ref{PAC-Bayes-kl Inequality based on Rényi Divergence} can be relaxed as follows:
\begin{align}
    \mathfrak{R}_{\mathcal{D}}^{\mathcal{V}}
     - \hat{\mathfrak{R}}_{S}^{\mathcal{V}} 
     \leq \sqrt{2\,\mathfrak{R}_{\mathcal{D}}^{\mathcal{V}}\left[\frac{\underset{v\sim \rho}{\mathbb{E}}D_{\alpha_v}\left(\mathcal{Q}_{v} \Vert \mathcal{P}_{v}\right) + D_{\alpha}\left(\rho \Vert \pi\right) + \ln \frac{2\sqrt{m}}{\delta}}{m}  \right]}
\end{align}

By using the inequality $\sqrt{xy} \leq \frac{1}{2}(\lambda x + \frac{y}{\lambda})$ for all $\lambda \geq 0$ \citep{McAllester03a}, we have that with probability at least $1-\delta$ for all $\mathcal{Q}_v$ and $\rho$.
\begin{align}
    \mathfrak{R}^\mathcal{V}_{\mathcal{D}} - \hat{\mathfrak{R}}^\mathcal{V}_{S} \leq \frac{\lambda}{2} \mathfrak{R}^\mathcal{V}_{\mathcal{D}}  + \frac{\underset{v\sim \rho}{\mathbb{E}}\left[D_{\alpha_v}(\mathcal{Q}_{v} \Vert \mathcal{P}_{v})\right]  + D_\alpha(\rho \Vert \pi)  
    + \ln{\frac{2\sqrt{m}}{\delta}}}{\lambda m}  
\end{align}

By changing sides:
\begin{align}
    (1 - \frac{\lambda}{2}) \mathfrak{R}^{\mathcal{V}}_{\mathcal{D}} \leq \hat{\mathfrak{R}}^{\mathcal{V}}_{S} + \frac{\underset{v \sim \rho}{\mathbb{E}}\left[D_{\alpha_v}(\mathcal{Q}_{v} \Vert \mathcal{P}_{v})\right] + D_\alpha(\rho \Vert \pi) + \ln{\frac{2\sqrt{m}}{\delta}}}{\lambda m}
\end{align}

For $\lambda < 2$ we can divide both sides by $1 - \lambda$ and obtain the theorem statement.

We use the following version of refined Pinsker’s inequality:
\begin{align}\label{Pinsker’s inequality}
    \text{for}\, p>q,  \textnormal{KL}(p\Vert q) \leq (p-q)^{2}/(2p)
\end{align}

By application of inequality~\ref{Pinsker’s inequality} with inequality~\ref{PAC-Bayes-kl Inequality based on Rényi Divergence} of paper we obtain the following inequality:
\begin{align}
    &\hat{\mathfrak{R}}_{S} 
    - \mathfrak{R}_{\mathcal{D}} 
    \leq  \sqrt{2\,\hat{\mathfrak{R}}^{\mathcal{V}}_{S}\left[\frac{\underset{v\sim \rho}{\mathbb{E}}D_{\alpha_v}\left(\mathcal{Q}_{v} \Vert \mathcal{P}_{v}\right) + D_\alpha\left(\rho \Vert \pi\right) + \ln\left(\frac{2\sqrt{m}}{\delta}\right)}{m} \right]}
\end{align}

By using the inequality $\sqrt{xy} \leq \frac{1}{2}(\gamma x + \frac{y}{\gamma})$ for all $\gamma \geq 0$ \citep{McAllester03a}, we have that with probability at least $1-\delta$ for all $\mathcal{Q}_v$ and $\rho$. 
\begin{align}
    &\hat{\mathfrak{R}}^\mathcal{V}_{S} - \mathfrak{R}^{\mathcal{V}}_{\mathcal{D}} \leq \frac{\gamma}{2} \hat{\mathfrak{R}}_{S} + \frac{\underset{v\sim \rho}{\mathbb{E}}\left[D_{\alpha_v}(\mathcal{Q}_{v} \Vert \mathcal{P}_{v})\right] + D_\alpha(\rho \Vert \pi)  
    + \ln{\frac{2\sqrt{m}}{\delta}}}{\gamma m}
\end{align}

By changing sides
\begin{align}
    \mathfrak{R}_{\mathcal{D}} \geq & \left(1-\frac{\gamma}{2}\right) \hat{\mathfrak{R}}_{S}  - \frac{\underset{v\sim \rho}{\mathbb{E}}\left[D_{\alpha_v}(\mathcal{Q}_{v} \Vert \mathcal{P}_{v})\right] + D_\alpha(\rho \Vert \pi) + \ln{\frac{2\sqrt{m}}{\delta}}}{\gamma m}
\end{align}

\section{Second Order Multi-view Oracle Bound}\label{Second Order Multi-view Oracle Bound}
To demonstrate the second-order oracle bound \citep{Masegosa2020}, we first aim to define the multi-view Tandem loss. This foundational understanding will facilitate a clearer demonstration of the second-order oracle bound with the multi-view Tandem loss.

\begin{lemma}[Multi-view Tandem Loss]\label{Tandem_Loss} In multiclass classification
\begin{align}
\underset{(\bm{x}, y) \sim \mathcal{D}}{\mathbb{E}}\left[ \underset{v \sim \rho}{\mathbb{E}}\left[\underset{h \sim \mathcal{Q}_v}{\mathbb{E}} [\ell(h(\bm{x}^v),y)]^2\right]\right] = e^{\mathcal{V}}_{\mathcal{D}}.
\end{align}

\begin{proof}
\begin{align}
    \underset{(\bm{x}^v, y) \sim \mathcal{D}}{\mathbb{E}}\left[ \underset{v \sim \rho}{\mathbb{E}}\left[\underset{h \sim \mathcal{Q}_v}{\mathbb{E}} \left[\mathbb{I}(h(\bm{x}^v) \neq y)\right]^2\right]\right] 
    &= \underset{(\bm{x}^v, y) \sim \mathcal{D}}{\mathbb{E}}\left[ \underset{v \sim \rho}{\mathbb{E}}\,\underset{h \sim \mathcal{Q}_v}{\mathbb{E}}\left[ \mathbb{I}(h(\bm{x}^v) \neq y) \right] \underset{h \sim \mathcal{Q}_v}{\mathbb{E}}\left[ \mathbb{I}(h(\bm{x}^v) \neq y) \right] \right] \\
    &= \underset{(\bm{x}^v, y) \sim \mathcal{D}}{\mathbb{E}}\left[ \underset{(v,v') \sim \rho^{2}}{\mathbb{E}}\,\underset{(h,h') \sim \mathcal{Q}^{2}_v}{\mathbb{E}}\left[ \mathbb{I}(h(\bm{x}^v) \neq y)\mathbb{I}(h'(\bm{x}^{v'}) \neq y) \right] \right]\\
    &= \underset{(\bm{x}^v, y) \sim \mathcal{D}}{\mathbb{E}}\left[ \underset{(v,v') \sim \rho^{2}}{\mathbb{E}}\,\underset{(h,h') \sim \mathcal{Q}^{2}_v}{\mathbb{E}}\left[ \mathbb{I}(h(\bm{x}^v) \neq y \land h'(\bm{x}^{v'}) \neq y) \right] \right]\\
    &= \underset{(\bm{x}^v, y) \sim \mathcal{D}}{\mathbb{E}}\left[ \underset{(v,v') \sim \rho^{2}}{\mathbb{E}}\,\underset{(h,h') \sim \mathcal{Q}^{2}_v}{\mathbb{E}}\left[ \mathbb{I}(h(\bm{x}^v) \neq y \land h'(\bm{x}^{v'}) \neq y) \right] \right]\\
    &= \underset{(v,v') \sim \rho^{2}}{\mathbb{E}}\,\underset{(h,h') \sim \mathcal{Q}^{2}_v}{\mathbb{E}}\left[ \underset{(\bm{x}^v, y) \sim \mathcal{D}}{\mathbb{E}} \left[ \mathbb{I}(h(\bm{x}^v) \neq y \land h'(\bm{x}^{v'}) \neq y) \right] \right]\\
    &= e^{\mathcal{V}}_{\mathcal{D}}.
\end{align}
\end{proof}    
\end{lemma}

\begin{theorem}[Second Order Oracle Bound)]  In multiclass classification
    \begin{align}
    \mathfrak{R}_{\mathcal{D}}^{\mathcal{V}} = 4\,e_{\mathcal{D}}^{\mathcal{V}}.
\end{align}
\end{theorem}
\begin{proof}
    By applying the second-order Markov's inequality~\ref{Second Order Markov’s Inequality} to \(Z = \underset{v \sim \rho}{\mathbb{E}}\left[\underset{h \sim \mathcal{Q}_v}{\mathbb{E}}[\mathbb{I}(h(\bm{x}^v) \neq y)]\right]\) and using Lemma~\ref{Tandem_Loss}, we derive:
\begin{align}
\mathfrak{R}^{\mathcal{V}}_{\mathcal{D}} &\leq \Pr\left\{\underset{v \sim \rho}{\mathbb{E}}[\mathbb{I}(h(\bm{x}^v) \neq y)] \geq 0.5\right\},
        \\ &\leq 4 \underset{(\bm{x}^v, y) \sim \mathcal{D}}{\mathbb{E}}\left[ \underset{v \sim \rho}{\mathbb{E}}\left[\underset{h \sim \mathcal{Q}_v}{\mathbb{E}} [\mathbb{I}(h(\bm{x}^v),y)]^{2}\right]\right], \\ 
        &= 4\,e_{\mathcal{D}}^{\mathcal{V}}.
    \end{align}

\end{proof}

\section{Proofs of Multi-view Oracle Bounds Inequalities}\label{Proofs of Multi-view Oracle Bounds Inequalities}

\subsection{Proof of Corollary~\ref{Pac-bayes-kl-mv-FO}}

The corollary follows by using the bound stated in the equation of Theorem~\ref{multi_view_thieman} to bound $\mathfrak{R}_{\mathcal{D}}^{\mathcal{V}}$ in first-order oracle bound~\ref{First Order Multi-view Oracle Bound}. To bound ${e}^{\mathcal{V}}_{\mathcal{D}}$ and ${d}^{\mathcal{V}}_{\mathcal{D}}$ with first-order oracle bound, we also apply Theorem~\ref{multi_view_thieman}, Corollary~\ref{Hennequin_Bound_Macallester_dis}, and Proposition~\ref{proposition_Hennequin}, where it is stated that $D_\alpha(Q^2 \| P^2) = 2\,D_\alpha(Q \| P)$. We replace $\delta$ by $\delta/2$ in the upper
and lower bound and take a union bound over them.

\subsection{Proof of Corollary~\ref{Pac-bayes-kl-inv-mv-FO}}
The proof is same as proof~\ref{Pac-bayes-kl-mv-FO}, instead of use Theorem~\ref{multi_view_thieman} we use inverted KL with Corollary~\ref{PAC-Bayes-kl Inequality based on Rényi Divergence} to bound $\hat{\mathfrak{R}}_{S}^{\mathcal{V}}$, and Corollary~\ref{Hennequin_Bound_Macallester_dis} to bound the joint and disagreement error. 

\subsection{Proof of Corollary~\ref{Pac-bayes-kl-mv-SO}}
To bound $e^{\mathcal{V}}_{\mathcal{D}}$ with second-order oracle bound~\ref{Second Order Multi-view Oracle Bound th}, we apply Theorem~\ref{multi_view_thieman}, Corollary~\ref{Hennequin_Bound_Macallester_dis}, and Proposition~\ref{proposition_Hennequin}, where it is stated that $D_\alpha(Q^2 \| P^2) = 2\,D_\alpha(Q \| P)$.

\subsection{Proof of Corollary~\ref{PAC-Bayes_lambda_binary classification}}

The proof is same as proof in Corollary~\ref{Pac-bayes-kl-mv-FO}, instead of use first-order we use second-order oracle bound in binary classification. As noted to bound $d^{\mathcal{V}}_{\mathcal{D}}$ we use lower bound of Theorem~\ref{multi_view_thieman}. We replace $\delta$ by $\delta/2$ in the upper
and lower bound and take a union bound over them.

\subsection{Proof of Corollary~\ref{Pac-bayes-kl-inv-mv-SO}}

To bound  $e^{\mathcal{V}}_{\mathcal{D}}$ with second-order oracle bound~\ref{Second Order Multi-view Oracle Bound th}, we apply KL inverted with Corollary~\ref{Hennequin_Bound_Macallester_dis}, and Proposition~\ref{proposition_Hennequin}, where it is stated that $D_\alpha(Q^2 \| P^2) = 2\,D_\alpha(Q \| P)$.

\subsection{Proof of Corollary~\ref{Multi-view PAC-Bayes-KL_Inequality_binary}}

We use second-order oracle bound in binary classification~\ref{Second Order Multi-view Oracle Bound th} with Corollary~\ref{PAC-Bayes-kl Inequality based on Rényi Divergence} and and inverted KL to bound $\hat{\mathfrak{R}}_{S}$, and Corollary~\ref{Hennequin_Bound_Macallester_dis} with Proposition~\ref{proposition_Hennequin} to bound ${\hat{d}}_{\mathcal{D}}^{\mathcal{V}}$. We replace $\delta$ by $\delta/2$ and take a union bound over them.

\section{Detailed Comparison of PAC-Bayesian Approaches}
\label{app:comparison}

Table~\ref{tab:comparison_detailed} provides a comprehensive comparison of existing PAC-Bayesian approaches for multi-view learning and our contributions. Notably, our reimplementation of KL-based bounds with complete optimization algorithms represents an advancement over the original works, enabling fair empirical comparison.

\begin{table}[H]
\caption{Detailed comparison of PAC-Bayesian approaches for multi-view learning}
\label{tab:comparison_detailed}
\begin{center}
\begin{tabular}{p{2.5cm}p{2cm}p{2.5cm}p{2cm}p{3cm}p{3cm}}
\hline \\
\textbf{Method} & \textbf{Setting} & \textbf{Divergence} & \textbf{Bound Type} & \textbf{Optimization} & \textbf{Comments} \\
\hline \\
\cite{Goyal17} & Multi-view & KL only & Expectation & Partial (theory only) & General formulation without explicit algorithms \\
\cite{Goyal19} & Multi-view & KL only & Probability & C-Bound only & Only optimized Lemma 1, Eq. 3 \\
\cite{begin16} & Single-view & Rényi ($\alpha > 1$) & Probability & None provided & Theoretical bounds only \\
\hline
\textbf{Our KL baseline} & Multi-view & KL only & Probability & Complete (Alg.~\ref{alg:compute_lambda}) & Our implementation for comparison \\
\textbf{Our approach} & Multi-view & Rényi with view-specific $\alpha_v$ & Probability & Complete (Alg.\ref{alg:compute_lambda}) & Adaptive $\alpha_v$ per view \\
\hline
\end{tabular}
\end{center}
\end{table}

\section{Optimization of Multi-View Bounds}\label{Optimization of Inverted Multi-View KL Bounds}

\begin{algorithm}[H]
\caption{Minimization of Equations {\color{purple}\ref{min_R}}, {\color{purple}\ref{min_e_second order}}, {\color{teal}\ref{min_e}},  {\color{orange}\ref{min_R_second_order}}, \textcolor{blue}{\ref{min_k},~\ref{min_K_u},~\ref{min_K_SO},~\ref{min_K_SO_u},~\ref{min_C_bound},~\ref{C_tandem_bound}} by Gradient Descent}
\label{alg:compute_lambda}
\begin{algorithmic}[1]
   \STATE \textbf{Given:} learning sample $S$, prior distribution $\mathcal{P}_v$ on $\mathcal{H}_v$, hyper-prior distribution $\pi$ on $[\![V]\!]$, the objective function {\color{purple}\ref{min_R},~\ref{min_e_second order}}, {\color{teal}\ref{min_e}}, {\color{orange}\ref{min_R_second_order}},\textcolor{blue}{\ref{min_k},~\ref{min_K_u},~\ref{min_K_SO},~\ref{min_K_SO_u},~\ref{min_C_bound},~\ref{C_tandem_bound}}
   \STATE \textbf{Hyperparameters:} number of iterations $T$, learning rate $lr$, parameter of log-barrier $t$ (\cite{Kervadec17}), convergence criterion $\epsilon$, parameter $\bm{\alpha}$ of Rényi divergence

   \STATE \textbf{Initialize} $\mathcal{Q}_v \gets \mathcal{P}_v$, $\rho \gets \pi$, $\bm{\lambda} \in (0.0001, 1.9999)$ \COMMENT{$\forall \,\lambda \in (0,2)$}, $\gamma > 0.0001$ \COMMENT{Keeps $\gamma > 0$}
   \FOR{$t = 1$ {\bfseries to} $T$}
       \STATE \textcolor{blue}{\textbf{Compute}~\ref{min_k},~\ref{min_K_u},~\ref{min_K_SO},~\ref{min_K_SO_u},~\ref{min_C_bound},~\ref{C_tandem_bound} using Compute-$\overline{\textnormal{KL}}(q\|\psi)$
       \& Compute-$\underline{\textnormal{KL}}(q\|\psi)$} (Algorithm~\ref{alg:kl_computation})
       \STATE \textcolor{blue}{$\mathcal{Q}_v, \rho \gets \text{update}(\mathcal{Q}_v, \rho,~\ref{min_k},~\ref{min_K_u},~\ref{min_K_SO},\ref{min_K_SO_u},~\ref{min_C_bound},~\ref{C_tandem_bound})$}
       \STATE {\color{purple}$\mathcal{Q}_v, \rho, \lambda \gets \text{update}(\mathcal{Q}_v, \rho, \lambda,~\ref{min_R},~\ref{min_e_second order})$} 
       \STATE {\color{teal}$\mathcal{Q}_v, \rho, \lambda_1, \lambda_2 \gets \text{update}(\mathcal{Q}_v, \rho, \lambda_1,\lambda_2,~\ref{min_e})$}
       \STATE {\color{orange}$\mathcal{Q}_v, \rho, \lambda, \gamma \gets \text{update}(\mathcal{Q}_v, \rho, \lambda,\gamma,~\ref{min_R_second_order})$}

       \STATE Clip $\bm{\lambda}$: $\bm{\lambda} \leftarrow \text{clamp}(\lambda, 0.0001, 1.9999)$ 
       \STATE Clip $\gamma$: $\gamma \leftarrow \text{clamp}(\gamma, 0.0001)$ 
       \STATE \texttt{with torch.no\_grad():} 
       \STATE Apply \textbf{Softmax} to $\mathcal{Q}_v$ and $\rho$ \COMMENT{Normalize}
       \IF{Convergence criterion $\leq \epsilon$}
           \STATE \textbf{break}
       \ENDIF
   \ENDFOR
   
    \STATE {\color{blue}\textbf{return} $\mathcal{Q}_v, \rho$}
   \STATE {\color{purple}\textbf{return} $\mathcal{Q}_v, \rho, \lambda$}
   \STATE {\color{teal}\textbf{return} $\mathcal{Q}_v, \rho, \lambda_1, \lambda_2$}
   \STATE {\color{orange}\textbf{return} $\mathcal{Q}_v, \rho, \lambda, \gamma$}
   
\end{algorithmic}
\begin{flushleft}
\textbf{Note:} The function $\textbf{update}$ is a generic update function, such as Gradient Descent or any other algorithm; in our implementation, we use AdamW \citep{Ilya17} in PyTorch \citep{paszke2019pytorch}.
\end{flushleft}
\end{algorithm}

 \begin{algorithm}[H]
    \caption{Compute-$\overline{\textnormal{KL}}(q\|\psi)$ and $\underline{\textnormal{KL}}(q\|\psi)$ \cite{Reeb18}}
    \label{alg:kl_computation}
 \begin{algorithmic}[1]
    \STATE \textbf{Hyperparameters:} Tolerance $\epsilon$, maximum number of iterations $T_{\max}$
    \STATE Initialize $p_{\max}$ and $p_{\min}$:
    \STATE \quad For $\overline{\textnormal{KL}}$: $p_{\max} \gets 1$, $p_{\min} \gets q$
    \STATE \quad For $\underline{\textnormal{KL}}$: $p_{\max} \gets q$, $p_{\min} \gets 0$
    \FOR{$t=1$ {\bfseries to} $T_{\max}$}
        \STATE $p \gets \dfrac{1}{2}(p_{\min} + p_{\max})$
        \IF{$\textnormal{KL}(q \parallel p) = \psi$ \textbf{ or } $(p_{\max} - p_{\min}) < \epsilon$}
            \STATE \textbf{return} $p$
        \ENDIF
        \IF{$\textnormal{KL}(q \parallel p) > \psi$}
            \STATE For $\overline{\textnormal{KL}}$: $p_{\max} \gets p$
            \STATE For $\underline{\textnormal{KL}}$: $p_{\min} \gets p$
        \ELSE
            \STATE For $\overline{\textnormal{KL}}$: $p_{\min} \gets p$
            \STATE For $\underline{\textnormal{KL}}$: $p_{\max} \gets p$
        \ENDIF
    \ENDFOR
    \STATE \textbf{return} $p$
 \end{algorithmic}
 \end{algorithm}

\section{Experiments}\label{sup_experimental}

\subsection{Multi-view Datasets}\label{sup_datasets}

We have chosen some readily available multi-view datasets, as shown in Table~\ref{tab:datasets}, with varying sizes $1000 \leq N \leq 110250$. Additionally, we created multi-view versions of known benchmark datasets. To transform the datasets into binary classification tasks, we consider the strategies \textit{one-versus-all}, \textit{one-versus-one}, and some dataset-specific transformations. For all experiments, we merge, shuffle, and split the datasets. Below, we explain how we created each one:

\begin{table}[h!]
\centering
\caption{Multi-view Dataset Information with Original Locations. "Was MV" means "Was multi-view".}
\label{tab:datasets}
\begin{tabular}{lllllll}
\hline
\hline
Dataset Name & Original location & Was MV & Views & Samples & Classes & Size (MB) \\ \hline
aloi\_csv & \cite{elki-aloi-dataset_2010} & Yes & 4 & 110250 & 1000 & 673.4 \\
corel\_features & \cite{corel} & No & 7 & 1000 & 10 & 29.9 \\
MNIST\_1 & \cite{GOYAL201981} & Yes & 4 & 70000 & 10 & 318.7 \\
MNIST\_2 & \cite{GOYAL201981} & Yes & 4 & 70000 & 10 & 338.3 \\
Fash\_MNIST\_1 & \cite{xiao2017/online} & No & 4 & 70000 & 10 & 155.6 \\
Fash\_MNIST\_2 & \cite{xiao2017/online} & No & 4 & 70000 & 10 & 177.6 \\
EMNIST\_Letters\_1 & \cite{cohen2017emnist} & No & 4 & 70000 & 10 & 201.1 \\
EMNIST\_Letters\_2 & \cite{cohen2017emnist} & No & 4 & 70000 & 10 & 227.7 \\
mfeat & \cite{misc_multiple_features_72} & Yes & 6 & 2000 & 10 & 17.5 \\
mfeat-large & \cite{lecun2010mnist} & No & 6 & 70000 & 10 & 389.5 \\
Mushroom & \cite{misc_mushroom_73} & No & 2 & 8124 & 2 & 0.4 \\
NUS-WIDE-OBJECT & \cite{nus-wide-civr09} & Yes & 5 & 30000 & 31 & 231.4 \\
PTB-XL-plus & \makecell[tl]{\cite{ptb-xl} \\ \& \cite{ptb-xl+}} & Yes & 3 & 21800 & 5 Superclasses & 248 \\
ReutersEN & \cite{reuters} & Yes & 5 & 1200 & 6 & 22.1 \\
\hline
\hline
\end{tabular}
\end{table}

\textbf{Fash\_MNIST\_\{1-2\} and EMNIST\_Letters\_\{1-2\}.} we performed the same transformation made by \cite{GOYAL201981}, that is we  generated 2 four-view datasets where each view is a vector of $R^{14\times14}$\footnote{Their repository containing MNIST\_\{1-2\} can be found at this link \url{https://github.com/goyalanil/Multiview_Dataset_MNIST}}:
\begin{itemize}
    \item \textbf{EMNIST\_Letters\_1 and Fash\_MNIST\_1}: It is generated by considering 4 quarters of an image as 4 views.
    \item \textbf{EMNIST\_Letters\_2 and Fash\_MNIST\_2}: It is generated by considering 4 overlapping views around the center of images: this dataset brings redundancy between the views.
\end{itemize}

\textbf{Feature Extraction Methods for corel\_features.} 
The following features are extracted from the Corel dataset\footnote{Because we didn't find the original large dataset, we took a subset of it that was available on Kaggle: \url{https://www.kaggle.com/datasets/elkamel/corel-images}}:

\begin{itemize}

\item \textbf{Histogram of Oriented Gradients (HOG)}: \texttt{extract\_hog\_features(image)} converts the image to grayscale and computes HOG features using 9 orientations, $32\times32$ pixels per cell, and $2\times2$ cells per block.

\item \textbf{Local Binary Pattern (LBP)}: \texttt{extract\_lbp\_features(image)} converts the image to grayscale and computes LBP features with 8 points and a radius of 1. A histogram of the LBP is then normalized.

\item \textbf{Color Histogram}: \texttt{extract\_color\_histogram(image)} calculates the color histogram of the image using 8 bins for each channel (RGB) and normalizes the histogram.

\item \textbf{Gray-Level Co-Occurrence Matrix (GLCM)}: \texttt{extract\_glcm\_features(image)} calculates GLCM properties such as contrast, dissimilarity, homogeneity, energy, correlation, and ASM from the grayscale image.

\item \textbf{Zernike Moments}: \texttt{extract\_zernike\_moments(image)} computes Zernike moments of the grayscale image.

\item \textbf{Hu Moments}: \texttt{extract\_hu\_moments(image)} calculates Hu moments from the grayscale image’s moments.

\item \textbf{Haralick Texture Features}: \texttt{extract\_haralick\_features(image)} computes the mean Haralick texture features from the grayscale image.

\end{itemize}

\textbf{mfeat-large.} Directly inspired by the Multiple Features dataset \cite{misc_multiple_features_72}, which contains only 1000 samples, we attempted to extract similar but not exactly the same features from the original MNIST dataset, which contains 70000 images. We describe the features below:

\begin{itemize}

\item \textbf{mfeat-fou}: 76 Fourier coefficients of the character shapes.

\item \textbf{mfeat-fac}: 216 profile correlations. These features are obtained by measuring the correlation between profiles of the character images.

\item \textbf{mfeat-kar}: 64 Karhunen-Love coefficients. These coefficients are derived from a Karhunen-Loève transform (also known as Principal Component Analysis) and represent the main components of variation in the character images.

\item \textbf{mfeat-pix}: 240 pixel averages in $2\times3$ windows. This feature set consists of average pixel values computed over $2\times3$ pixel windows, providing a downsampled representation of the character images.

\item \textbf{mfeat-zer}: 47 Zernike moments. These moments are calculated to capture the shape and structure of the characters in a way that is invariant to rotation, scaling, and translation.

\item \textbf{mfeat-mor}: 6 morphological features. These features describe the morphological properties of the character images, such as the structure and form of the shapes within the images.

\end{itemize}

\textbf{Mushroom.} We simply split the features of the original Mushroom dataset \citep{misc_mushroom_73} into features that specifically describe the top of the mushroom and features that describe the bottom, resulting in 2 views.

\subsection{Experimental Setup}\label{app:exp_setup}

Random forests were trained using the Gini criterion for splitting and considering $\sqrt{d}$ features in each split, under three configurations: \textbf{1) Stump} ($\texttt{max\_depth}=1$), \textbf{2) Weak learners} ($\texttt{max\_depth}=3$), and \textbf{3) Strong learners} ($\texttt{max\_depth}=6$).

To obtain a comprehensive view of how our approaches perform in different scenarios, we varied the Rényi divergence order $\alpha \in \{1.1, 2\}$ for our multi-view bounds, while retaining the usual Kullback-Leibler divergence for the view-specific bounds. We also considered setting $\alpha$ as a learnable parameter during the optimization of the bounds, allowing for a view-specific optimal $\alpha_v$ (used in $\mathbb{E}_{\rho}\left[D_{\alpha_v}(\mathcal{Q}_{S,v} \| \mathcal{P}_v)\right]$) and an optimized multi-view $\alpha$ (used in $D_\alpha(\rho \| \pi)$).

The size of the available unlabeled data was varied over $\{0.05, 0.1, 0.2, 0.3, 0.4, 0.5, 1\}$.

For some multi-class classification experiments, we used only strong learners with increased depth ($\texttt{max\_depth}=20$), given the greater difficulty of multi-class learning compared to binary classification. Additionally, we reduced the sample size to 50\% for certain large, time-intensive datasets.

\subsection{Hyper-parameters}\label{app:sup_hyperparams}

The parameters of the algorithms were selected as follows. \textbf{1)} We re-implemented \cite{Masegosa2020}'s and \cite{Viallard2011}'s algorithms in PyTorch \citep{paszke2019pytorch} to take advantage of the Autograd Engine while keeping the same parameters. That is, $\delta = 0.05$ (the bounds hold with probability at least 95\%), and the log barrier $\mathbf{B}_t$ parameter $t=100$ is used for all algorithms. We use the AdamW optimizer with $lr=0.1$ and $weight\_decay=0.05$ for all algorithms except for $\mathcal{C}_{\rho}^{T}$\ref{Multi-view PAC-Bayesian C-tandem Oracle Bound} where COCOB optimizer \citep{COCOB} is used. \textbf{2)} All of the prior distributions $\mathcal{P}_v$, the hyper-prior $\pi$, the posterior distributions $\mathcal{Q}_v$, and the hyper-posterior $\rho$ are set to the uniform distribution before optimization. The optimization process involves computing the gradient of the right-hand side of each of the inequalities mentioned above w.r.t $\mathcal{Q}_v$, and the parameters $\lambda$ and $\gamma$ of their corresponding bounds, and then updating all at once. We fix the tolerance $\epsilon=10^{-9}$ and the maximum number of iterations $T=1000$.

\subsection{Hardware}\label{sup_hardware}

The experiments were conducted on a high-performance computing cluster equipped with NVIDIA Tesla V100 GPUs. Given the significant number of experiments (considering each dataset version, each $s\_labeled\_size$, each random forest configuration, and each $\alpha$), we utilized additional NVIDIA A40 GPUs to expedite the process.

\section{Results}\label{App:results}

We present additional results on binary and multi-class classification, spanning from Figure~\ref{figure:mfeat-binary-4-9-all-configs} to Figure~\ref{figure:corel-mult-full-plot}.\footnote{Note that the total number of results exceeds 150.} These results explore different configurations, including variations in $\alpha \in \{1.1, 2\}$ (as well as an optimizable $\alpha$), $s\_labeled\_size \in \{0.05, 0.1, 0.2, 0.3, 0.4, 0.5, 1\}$, and data poisoning through the addition of Gaussian noise.

\subsection{Analysis}\label{App:analysis}

Our experimental results reveal several insights regarding the performance of various bounds and configurations:

In general, inverted KL bounds demonstrate tighter results compared to those optimized using the relaxation in \cite{Thiemann17}, suggesting that the inverted form may provide a stricter constraint on generalization in practice.

Across all settings—single views, concatenated view, and our multi-view method—the first-order bound consistently provides the tightest results, despite being theoretically the loosest. This suggests that, in practice, optimizing the first-order bound with the inverted-KL approach (Equation~\ref{Eq-Pac-bayes-kl-inv-FO}) offers surprisingly strong empirical tightness.

In most experiments, the multi-view Bayes risk is found to be lower than the view-specific Bayes risk after optimization. This outcome implies that optimizing across multiple views provides robustness that individual view-specific risks lack, likely due to the integration of complementary information from each view.

Results from the concatenated view generally surpass those from the multi-view approach, which we attribute to the additional divergence term in the bound inequalities for the multi-view setup, slightly increasing the bound value and optimization complexity.

We observe a progressively lower bound on the Bayes risk ratio as we move from stump to weak learners, and then to strong learners. This trend suggests that as learner strength increases, the bound becomes tighter, indicating predictive reliability.

In datasets created using the method in \cite{GOYAL201981} (where each image is split into four parts), the concatenated view frequently outperforms the multi-view setup. This advantage likely arises because concatenation effectively reconstructs the full image, thereby preserving more information.

\subsubsection{Comparison of PAC-Bayesian Bounds}

The results presented in Figures~\ref{figure:vary-labeled} and~\ref{figure:vary-alpha} demonstrate the performance of different PAC-Bayesian bounds on the mfeat-large dataset (4vs9) as a function of the labeled data size and the Rényi divergence order, $\alpha$. Figure~\ref{figure:vary-labeled} highlights the effects of varying the proportion of labeled data ($s\_labeled\_size$) on bound values, with a fixed $\alpha = 1.1$. As labeled data increases, bounds improve, with $\mathcal{K}^{u}_{\textnormal{II}}$ (Equation~\ref{Eq-Pac-bayes-dis-inv-mv-SO}) achieving tighter values than $\mathcal{K}_{\textnormal{II}}$ (Equation~\ref{Eq-Pac-bayes-joint-inv-mv-SO}), especially with more unlabeled data and the inclusion of the disagreement term. This suggests that incorporating the disagreement enhances the bound's tightness due to the fact that it can learn with unlabeled data.

Figure~\ref{figure:vary-alpha} examines the effect of varying $\alpha$ on bound values with a fixed $s\_labeled\_size = 0.5$. The bounds generally tighten around $\alpha = 1.1$, suggesting that this value may provide an optimal trade-off for controlling the Rényi divergence. This observation is further supported by the results in Figure~\ref{figure:mfeat-binary-4-9-alpha1.1vsoptim}, where setting $\alpha$ as an optimizable parameter leads to convergence near 1.1. This trend highlights the importance of $\alpha$ in regulating bound tightness.

\subsubsection{Poisoning the data} \label{App:poisoning}

The presented Figure~\ref{figure:poisoning} illustrates the impact of data poisoning on the hyper-posterior distributions ($\rho^*$) of the two best-performing algorithms evaluated on the mfeat-large dataset. The comparison is made by analyzing the posterior distributions before and after introducing Gaussian noise to the most effective views (2 and 5).

\textbf{Before data poisoning (upper panel).} The optimized hyper-posterior distributions show a higher concentration, indicating that the algorithms are more confident in their predictions across different views. This reflects the algorithms' ability to effectively utilize the clean data to achieve tight bounds and reliable performance.

\textbf{After data poisoning (lower panel).} The introduction of Gaussian noise results in a noticeable shift in the posterior distributions. This shift demonstrates the robustness of our approaches to mitigate the effects of such perturbations, that is weighting down the attacked views. The added noise slightly disrupts the algorithms' ability to generalize.

\subsubsection{The effect of overfitting} \label{App:overfitting}

In some multi-class experiments, as shown in Figure~\ref{figure:corel-mult-full-plot}, we observe indications of overfitting, likely due to the relatively small dataset size. Specifically, certain bounds, such as $\mathcal{E}_{\textnormal{II}}$ and $\mathcal{K}_{\textnormal{II}}$ in several individual views, fall below the Bayes risk on the test set. This suggests that the algorithm may be overfitting, as the PAC-Bayesian bounds are expected to be conservative estimates of generalization error.

\begin{figure}
\centering
\subfloat[Stump]{\includegraphics[width=.7\linewidth]{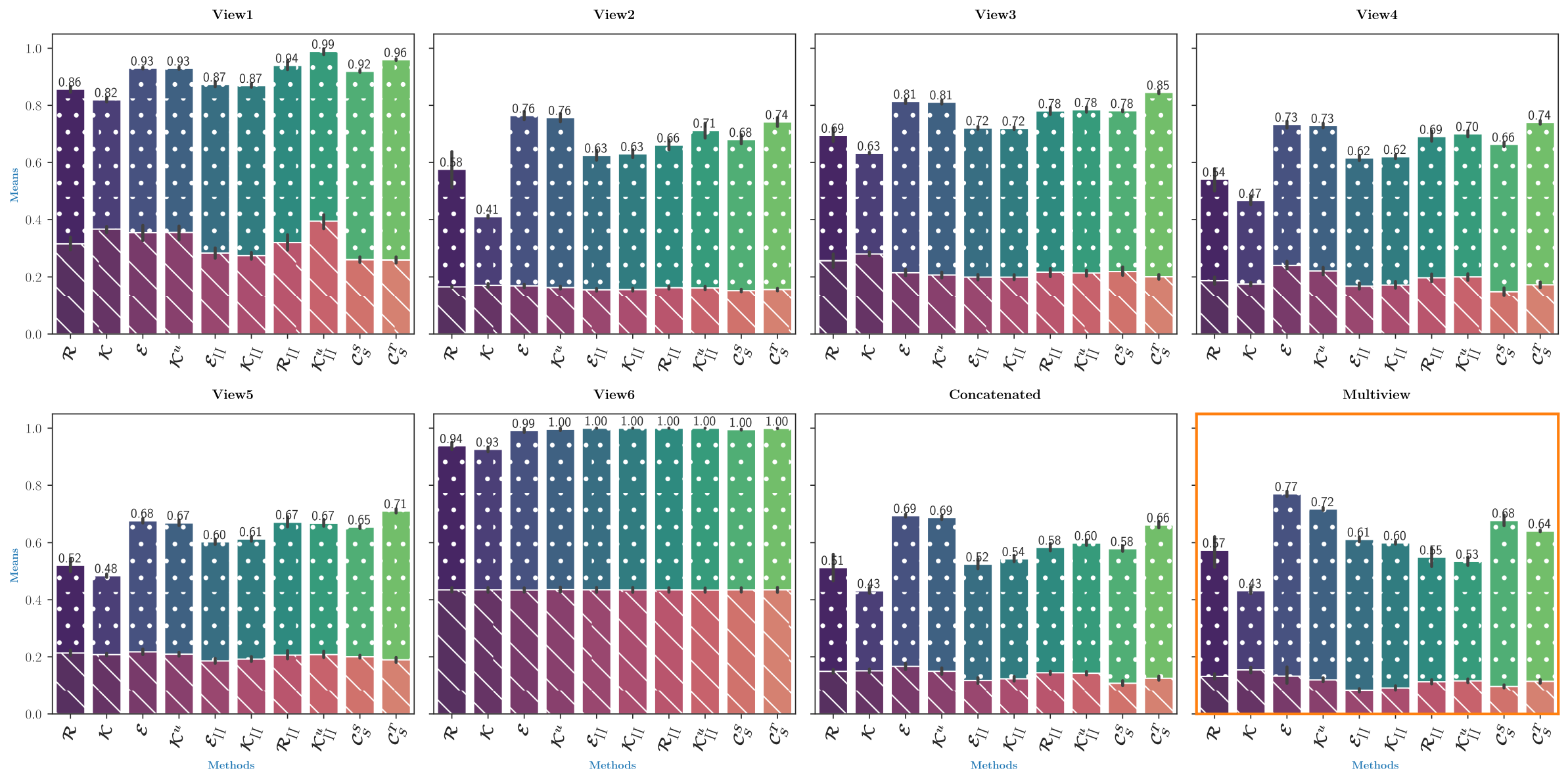}} \vspace{-1pt}\\

\subfloat[Weak learner]{\includegraphics[width=.7\linewidth]{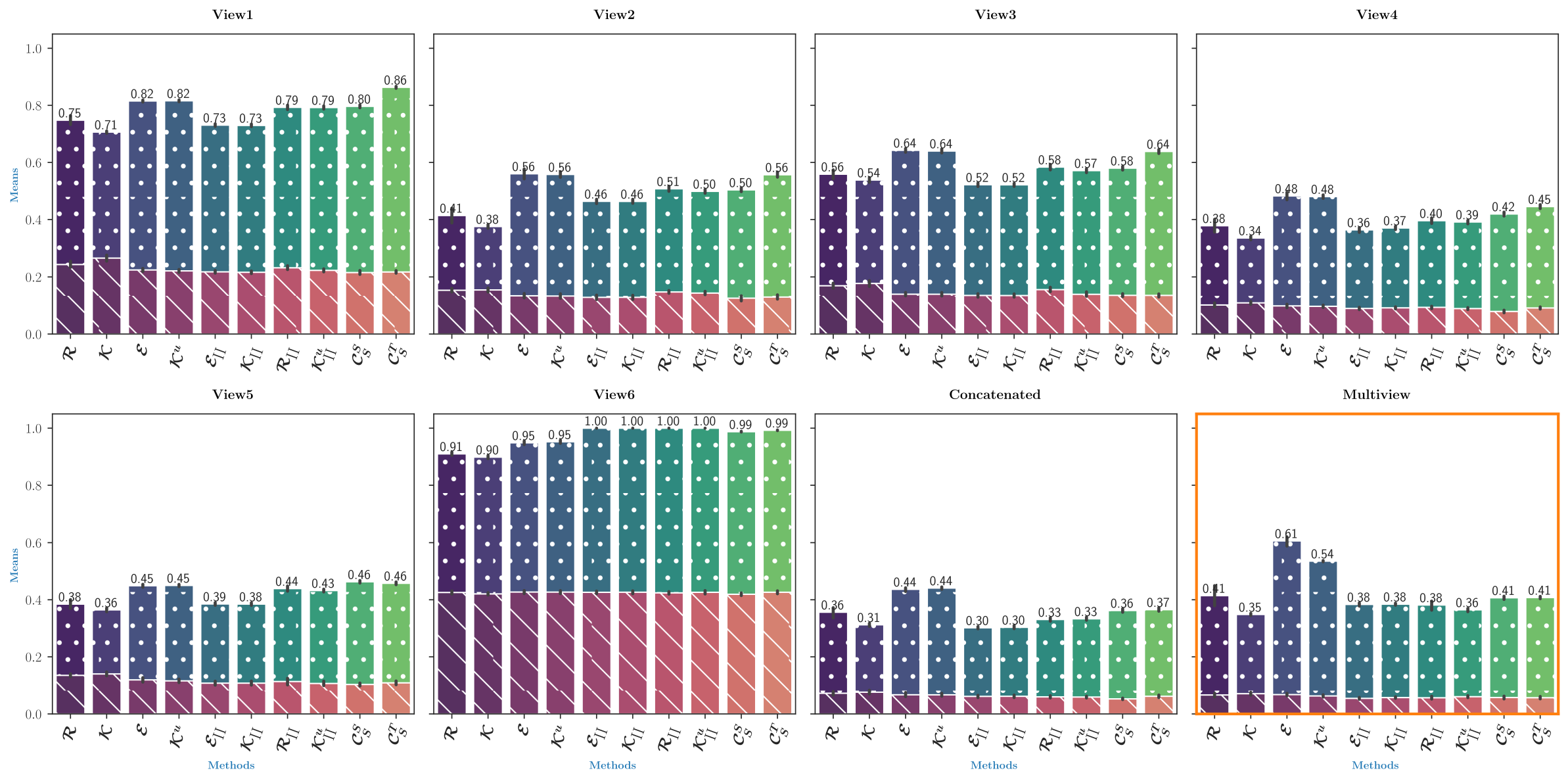}} \vspace{-1pt} \\

\subfloat[Strong learner]{\includegraphics[width=.7\linewidth]{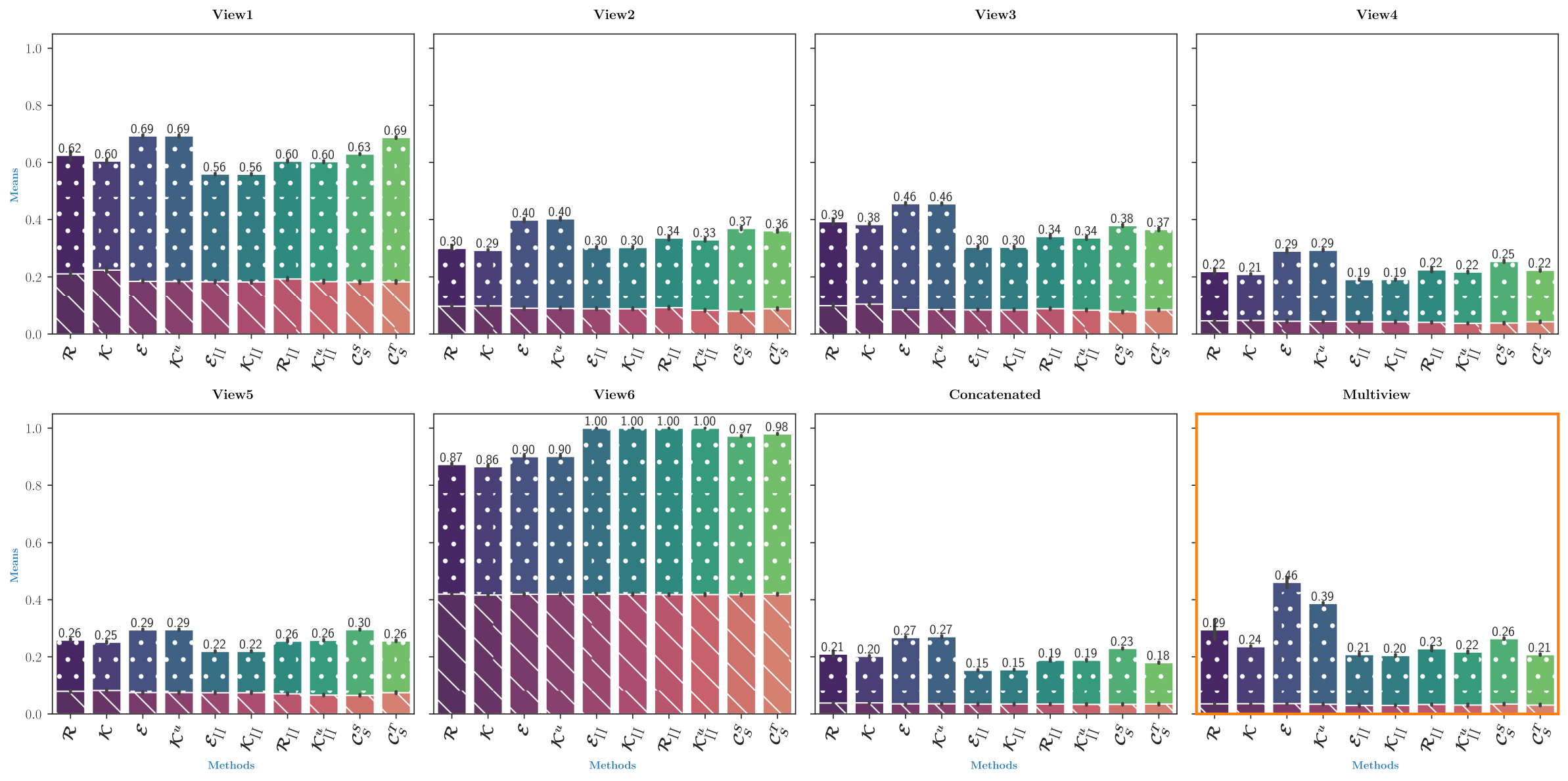}} \\

\caption{Test error rates and PAC-Bayesian bounds for binary classification between labels 4 and 9 on the mfeat-large dataset, averaged over 10 runs. The experiment uses KL divergence for single-view and Rényi divergence ($\alpha=1.1$) for multi-view, with a stump configuration for \textbf{(a)}, weak, and strong learners for \textbf{(b)} and \textbf{(c)} resp. and 50\% labeled data. Multi-view results are highlighted in orange.}
\label{figure:mfeat-binary-4-9-all-configs}

\end{figure}

\begin{figure}
\centering
\subfloat[$\alpha=1.1$]{\includegraphics[width=.8\linewidth]{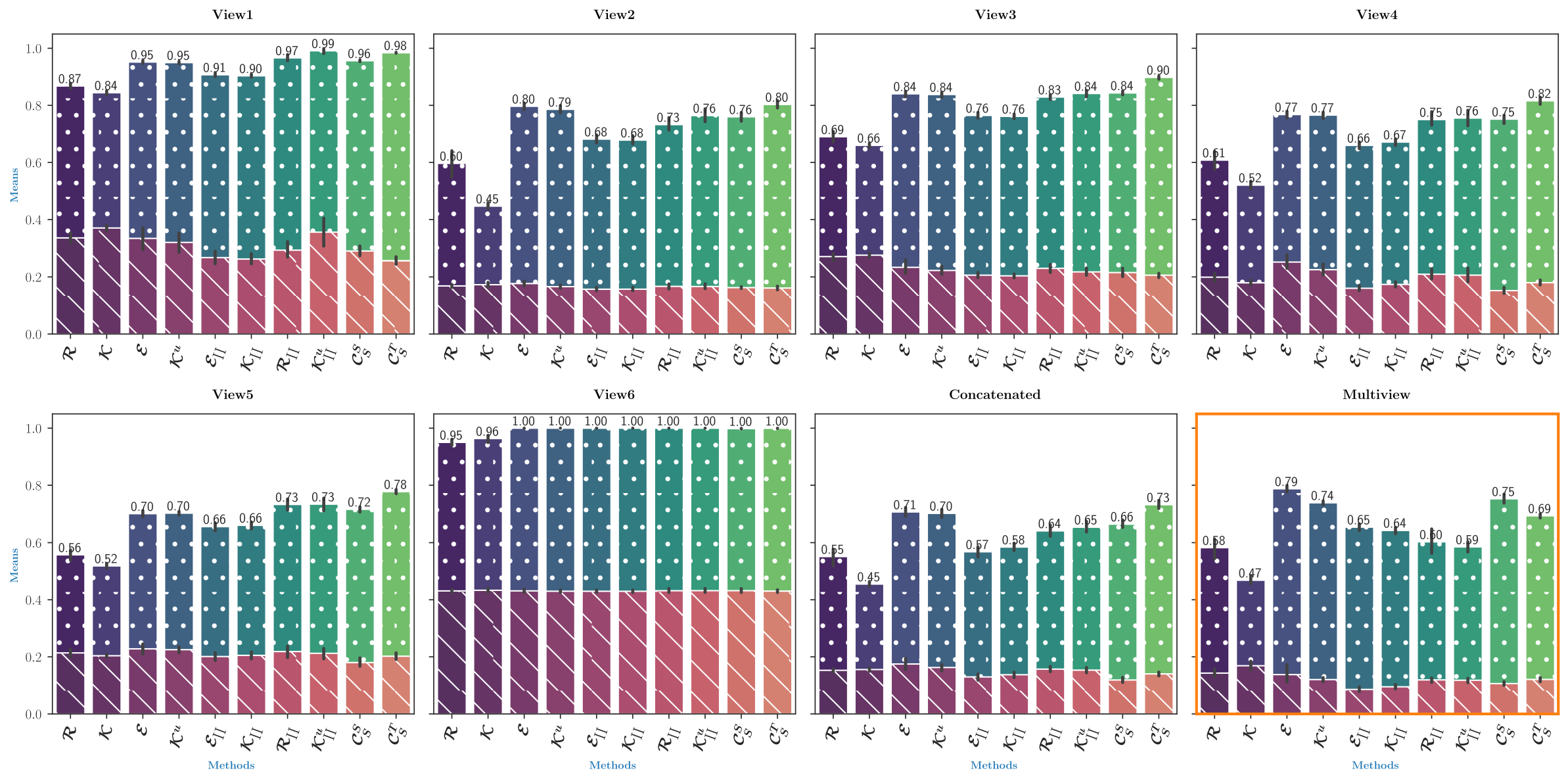}}\\

\subfloat[Optimizable $\alpha$]{\includegraphics[width=.8\linewidth]{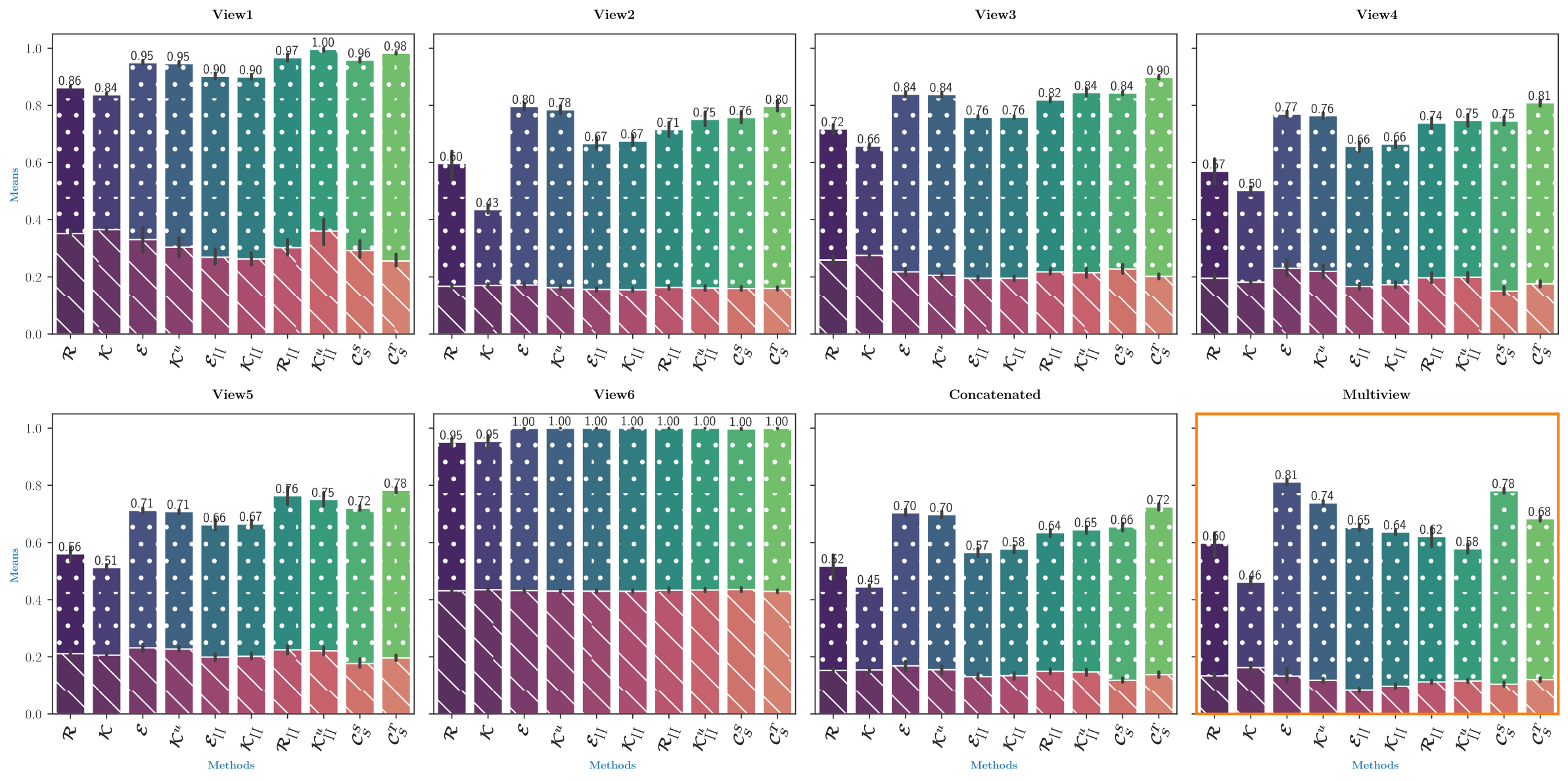}}\\

\caption{Test error rates and PAC-Bayesian bounds for binary classification between labels 4 and 9 on the mfeat-large dataset, averaged over 10 runs. The experiment uses KL divergence for single-view and Rényi divergence for multi-view, we compare between the setting \textbf{(a)} with $\alpha=1.1$ and \textbf{(b)} with $\alpha$ set as a learnable parameter. Using stump configuration and 20\% labeled data. Multi-view results are highlighted in orange.}
\label{figure:mfeat-binary-4-9-alpha1.1vsoptim}

\end{figure}

\begin{figure}
\centering
\subfloat[Stump]{\includegraphics[width=.55\linewidth]{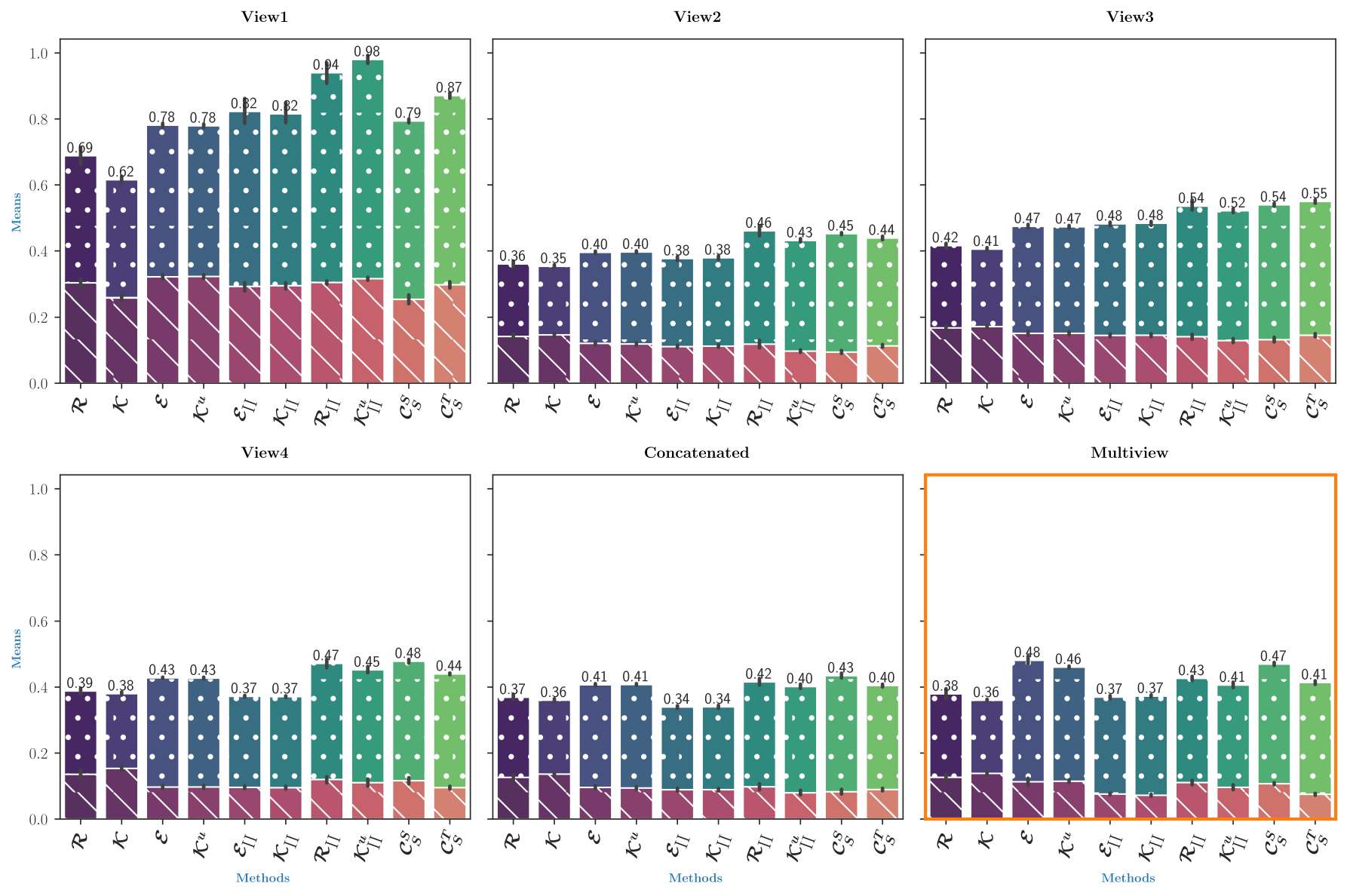}} \vspace{-1pt} \\

\subfloat[Weak learner]{\includegraphics[width=.55\linewidth]{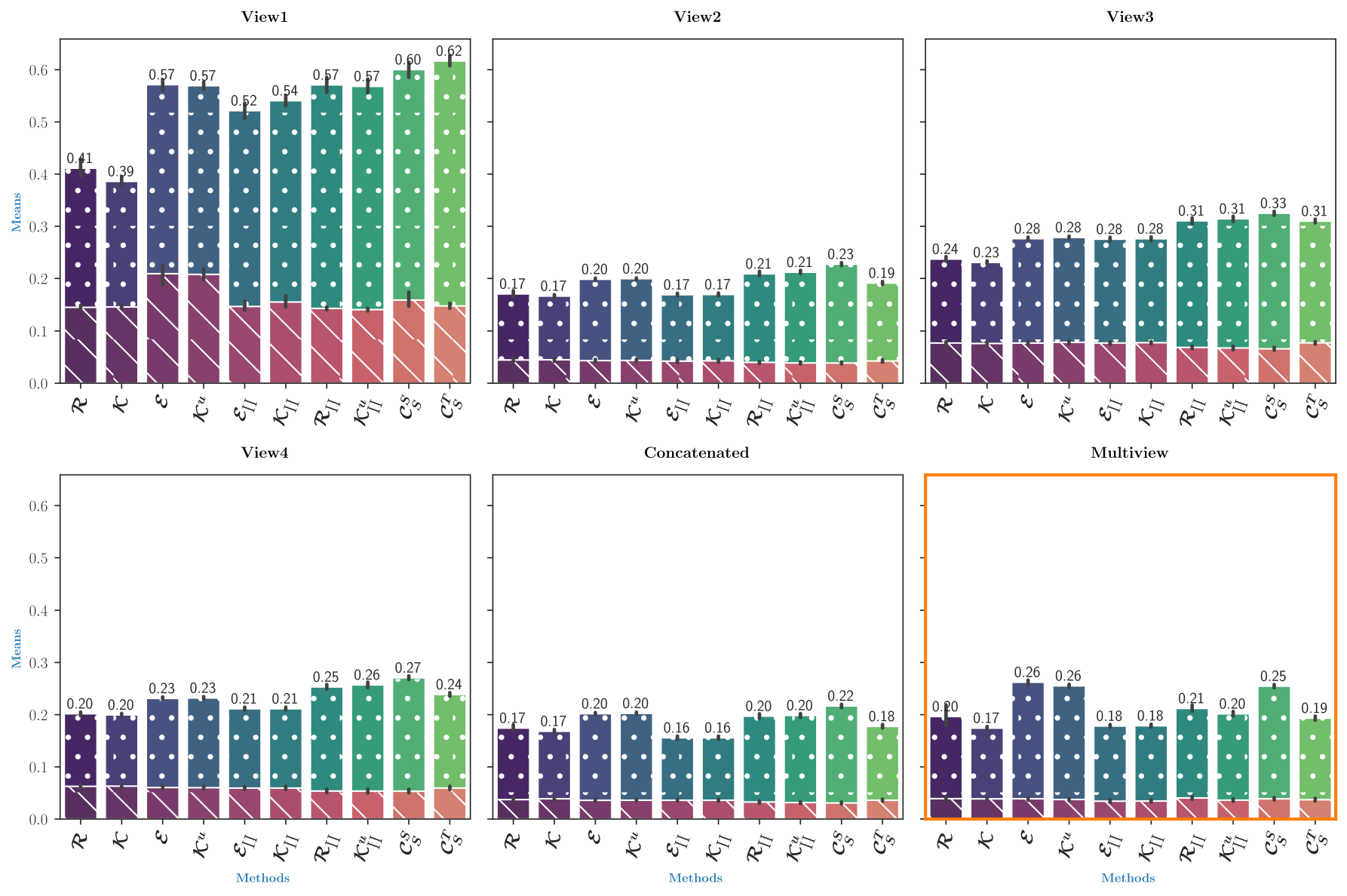}} \vspace{-1pt} \\

\subfloat[Strong learner]{\includegraphics[width=.55\linewidth]{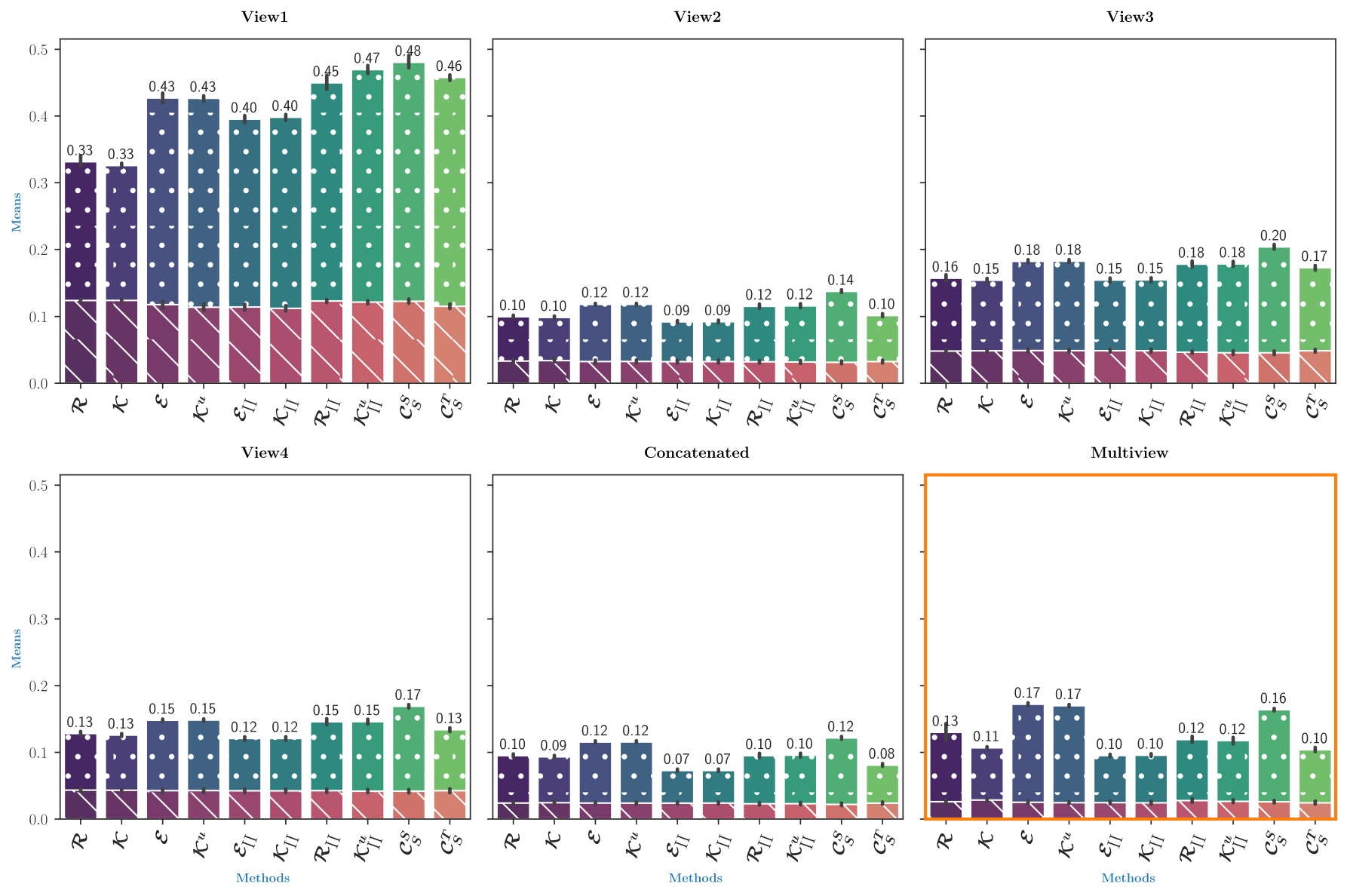}}\\

\caption{Test error rates and PAC-Bayesian bounds for binary classification between labels "Sandal" and "Ankle boot" on the Fashion-MNIST-MV dataset, averaged over 10 runs. The experiment uses KL divergence for single-view and Rényi divergence ($\alpha=1.1$) for multi-view, with a stump configuration for \textbf{(a)}, weak, and strong learners for \textbf{(b)} and \textbf{(c)} resp. and 50\% labeled data. Multi-view results are highlighted in orange.}
\label{figure:fash-mnist-binary-5-9-all-configs}

\end{figure}

\begin{figure}
\centering
\subfloat[Stump]{\includegraphics[width=.55\linewidth]{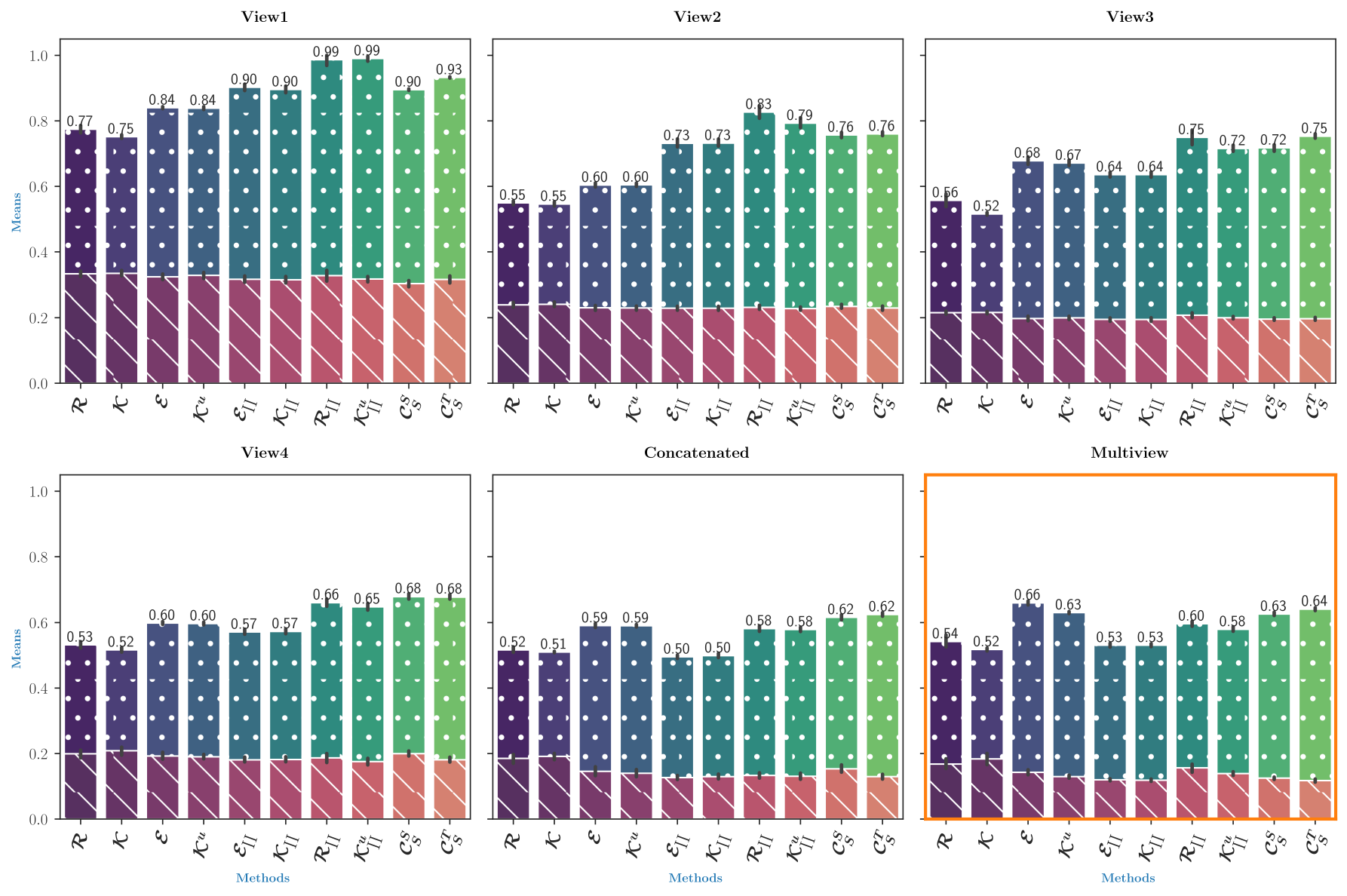}} \vspace{-1pt} \\

\subfloat[Weak learner]{\includegraphics[width=.55\linewidth]{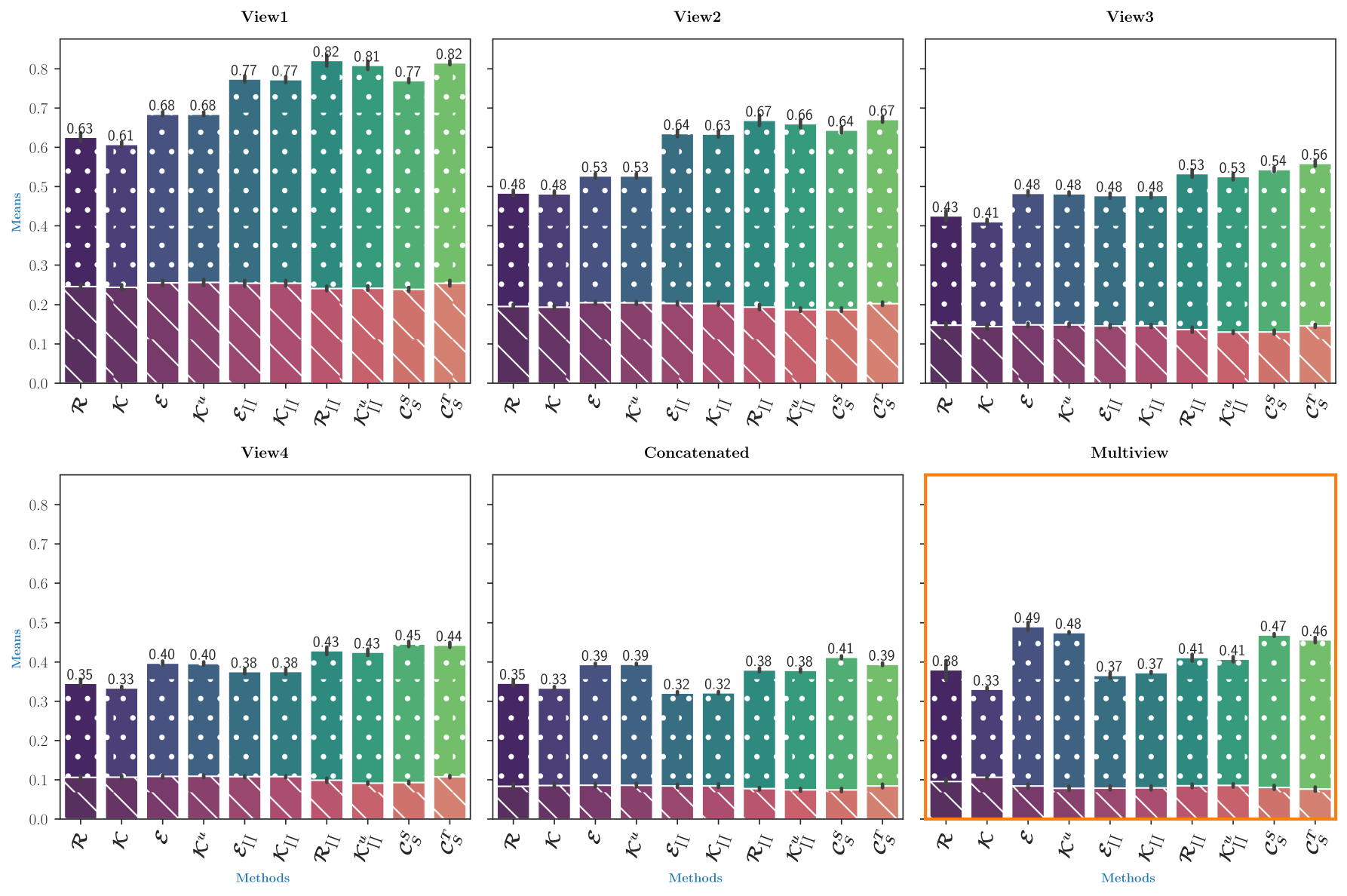}} \vspace{-1pt} \\

\subfloat[Strong learner]{\includegraphics[width=.55\linewidth]{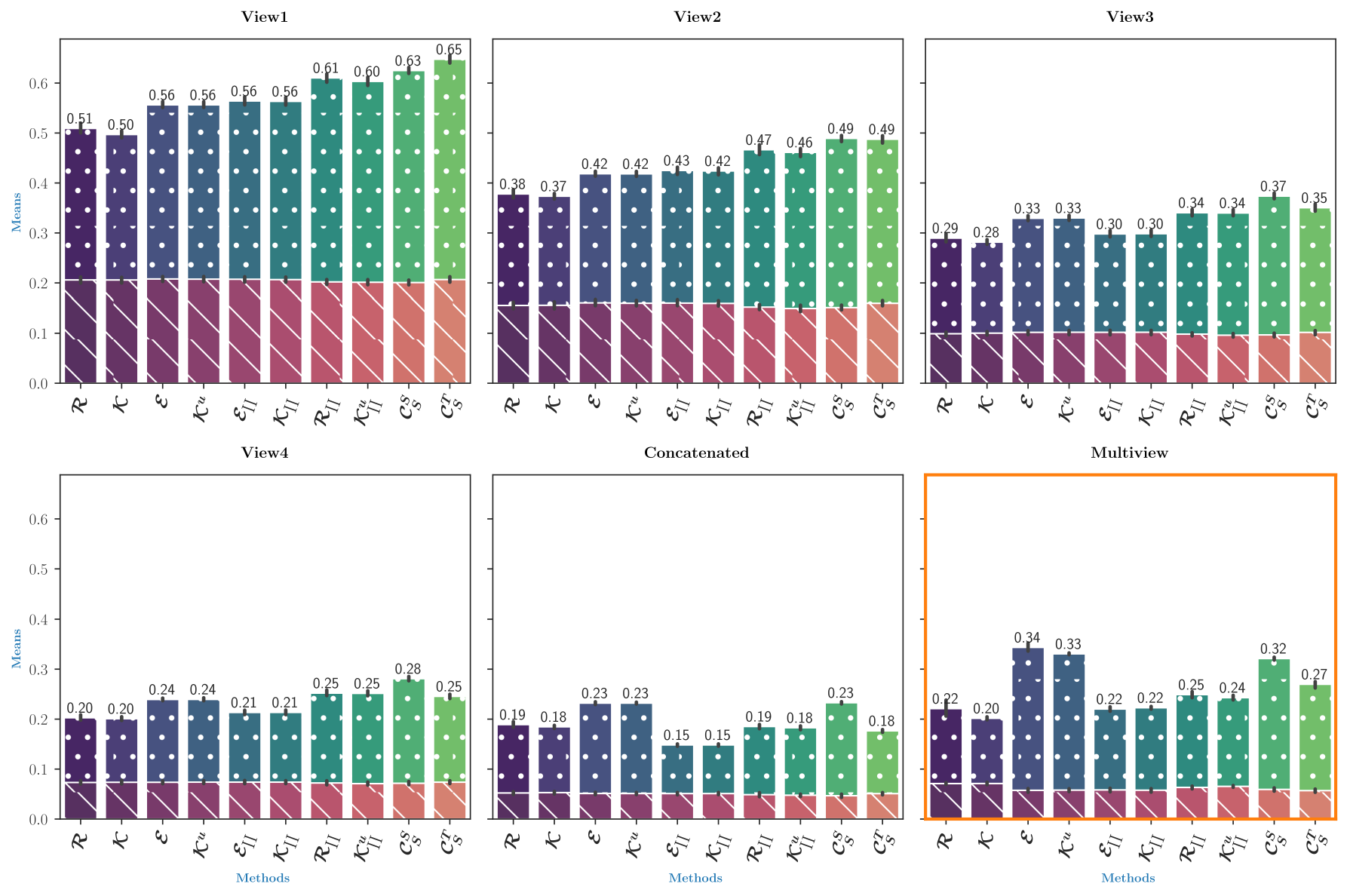}}\\

\caption{Test error rates and PAC-Bayesian bounds for binary classification between labels "M" and "N" on the EMNIST-Letters-MV dataset, averaged over 10 runs. The experiment uses KL divergence for single-view and Rényi divergence ($\alpha=1.1$) for multi-view, with a stump configuration for \textbf{(a)}, weak, and strong learners for \textbf{(b)} and \textbf{(c)} resp. and 50\% labeled data. Multi-view results are highlighted in orange.}
\label{figure:emnist-binary-13-14-all-configs}

\end{figure}

\begin{figure}
\centering
\subfloat[Stump]{\includegraphics[width=.95\linewidth]{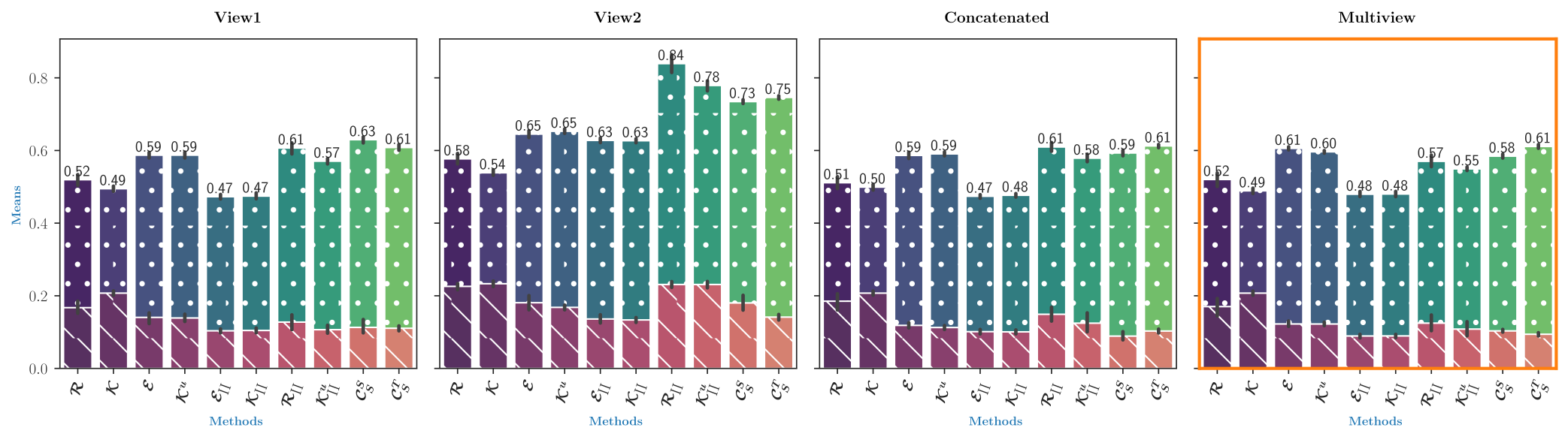}}\\

\subfloat[Weak learner]{\includegraphics[width=.95\linewidth]{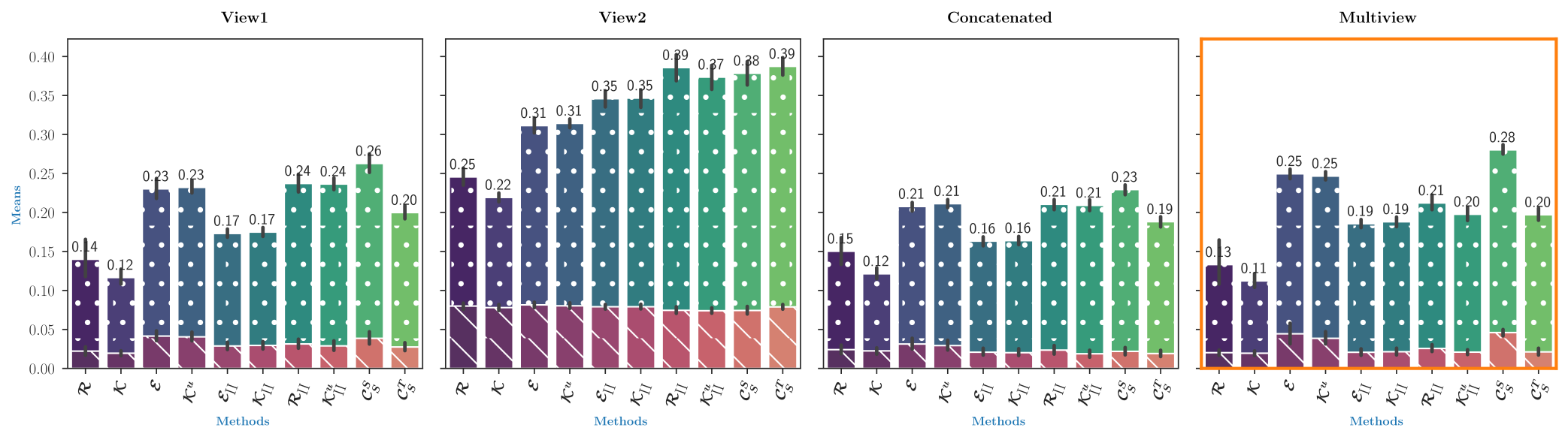}}\\

\subfloat[Strong learner]{\includegraphics[width=.95\linewidth]{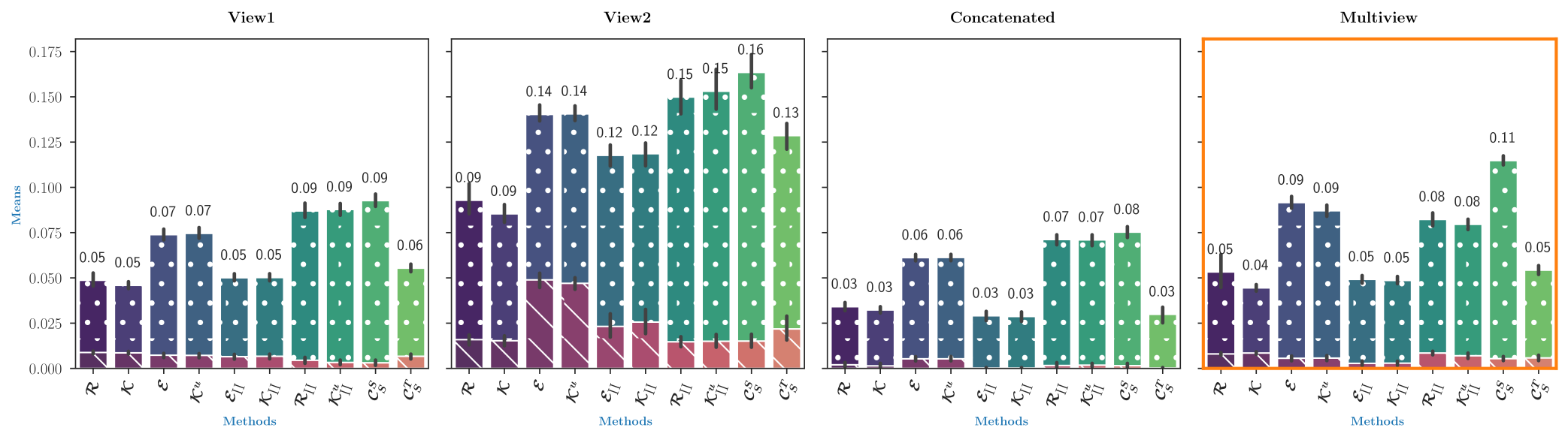}}\\

\caption{Test error rates and PAC-Bayesian bounds for binary classification on the Mushroom dataset, averaged over 10 runs. The experiment uses KL divergence or single-view and Rényi divergence ($\alpha=1.1$) for multi-view, with a stump configuration for \textbf{(a)}, weak, and strong learners for \textbf{(b)} and \textbf{(c)} resp. and 50\% labeled data. Multi-view results are highlighted in orange.}
\label{figure:mushroom-binary-e-p-all-configs}
\end{figure}

\begin{figure}
\centering
\subfloat[Stump]{\includegraphics[width=.55\linewidth]{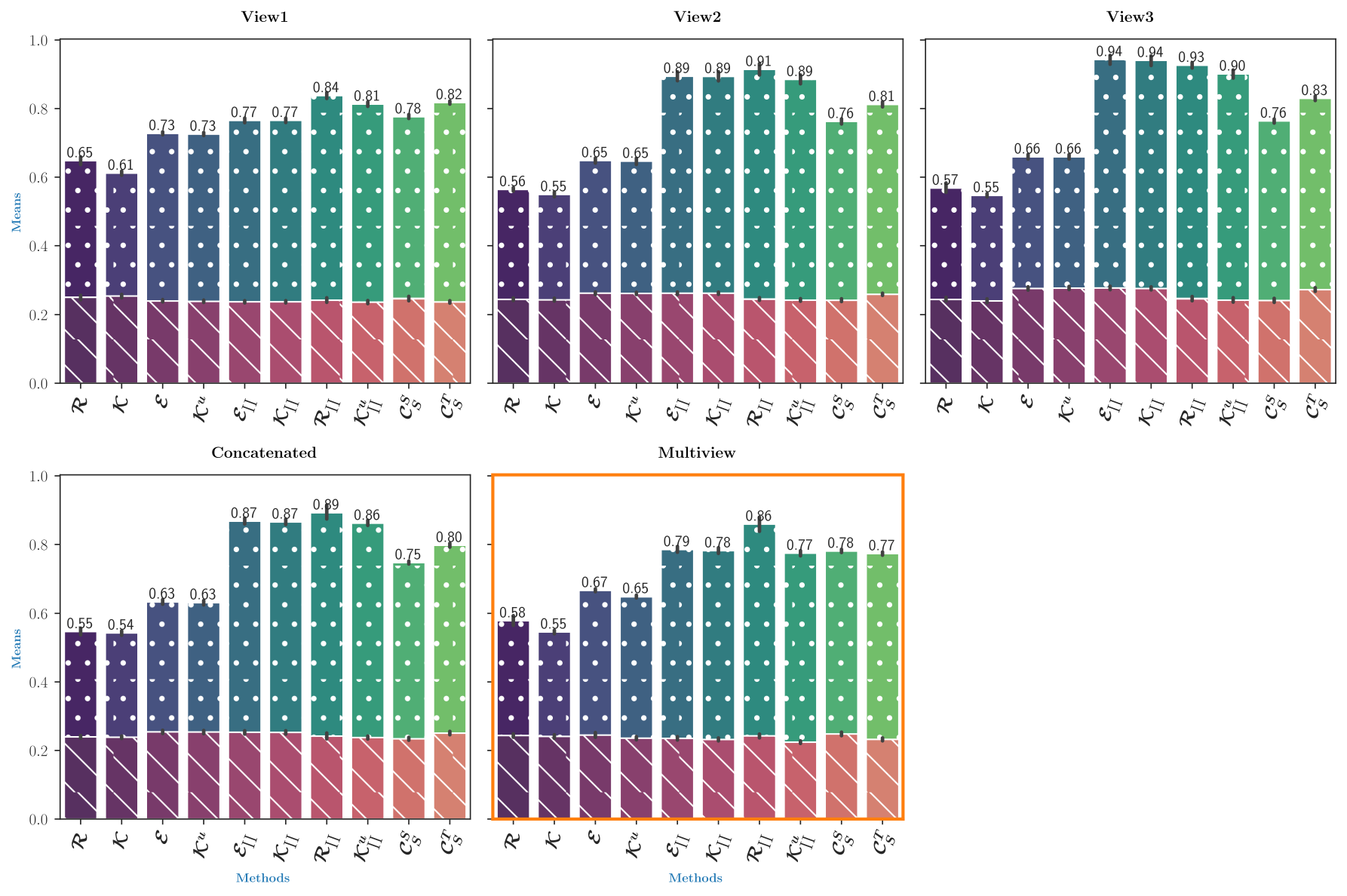}} \vspace{-1pt} \\

\subfloat[Weak learner]{\includegraphics[width=.55\linewidth]{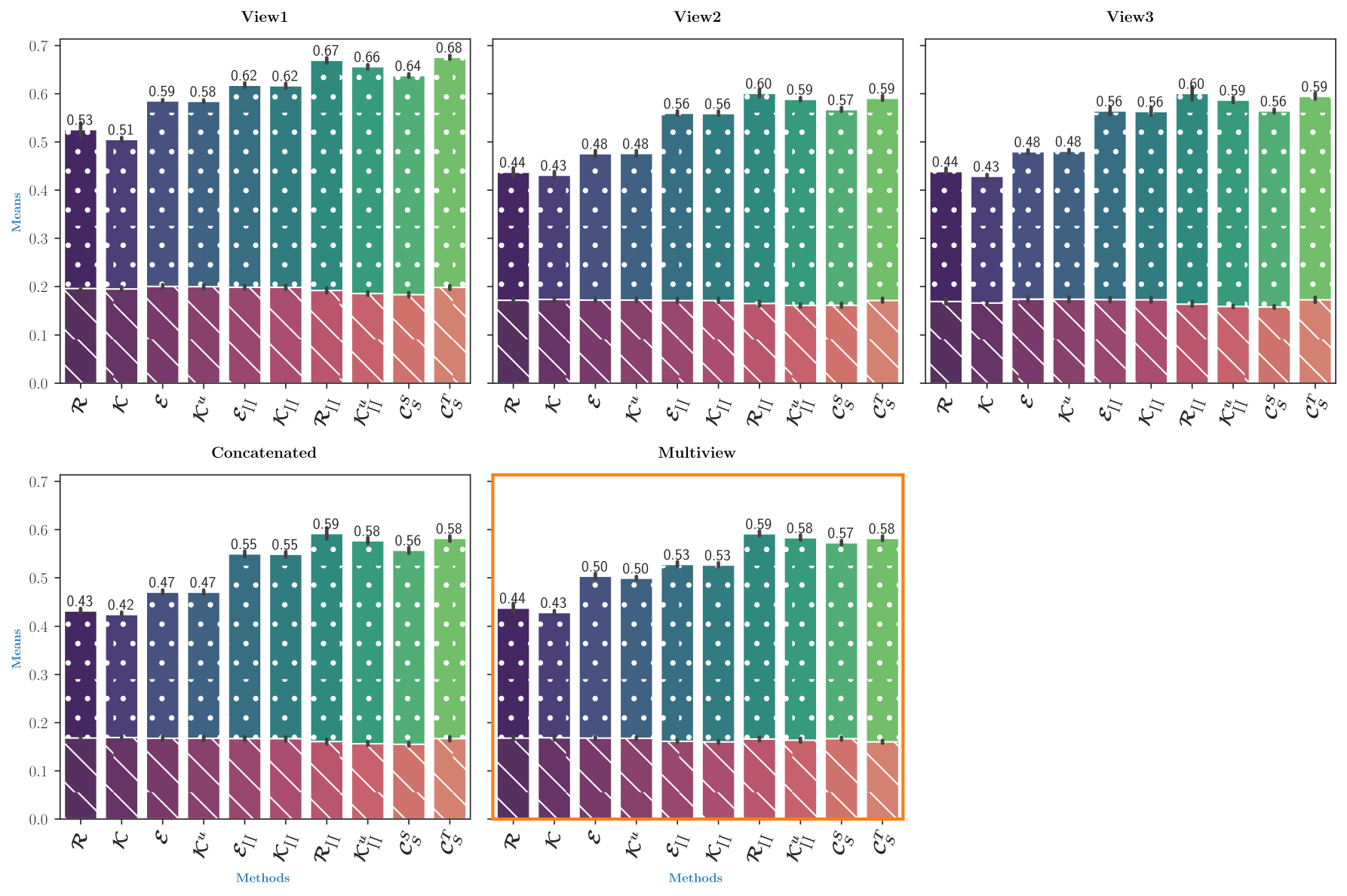}} \vspace{-1pt} \\

\subfloat[Strong learner]{\includegraphics[width=.55\linewidth]{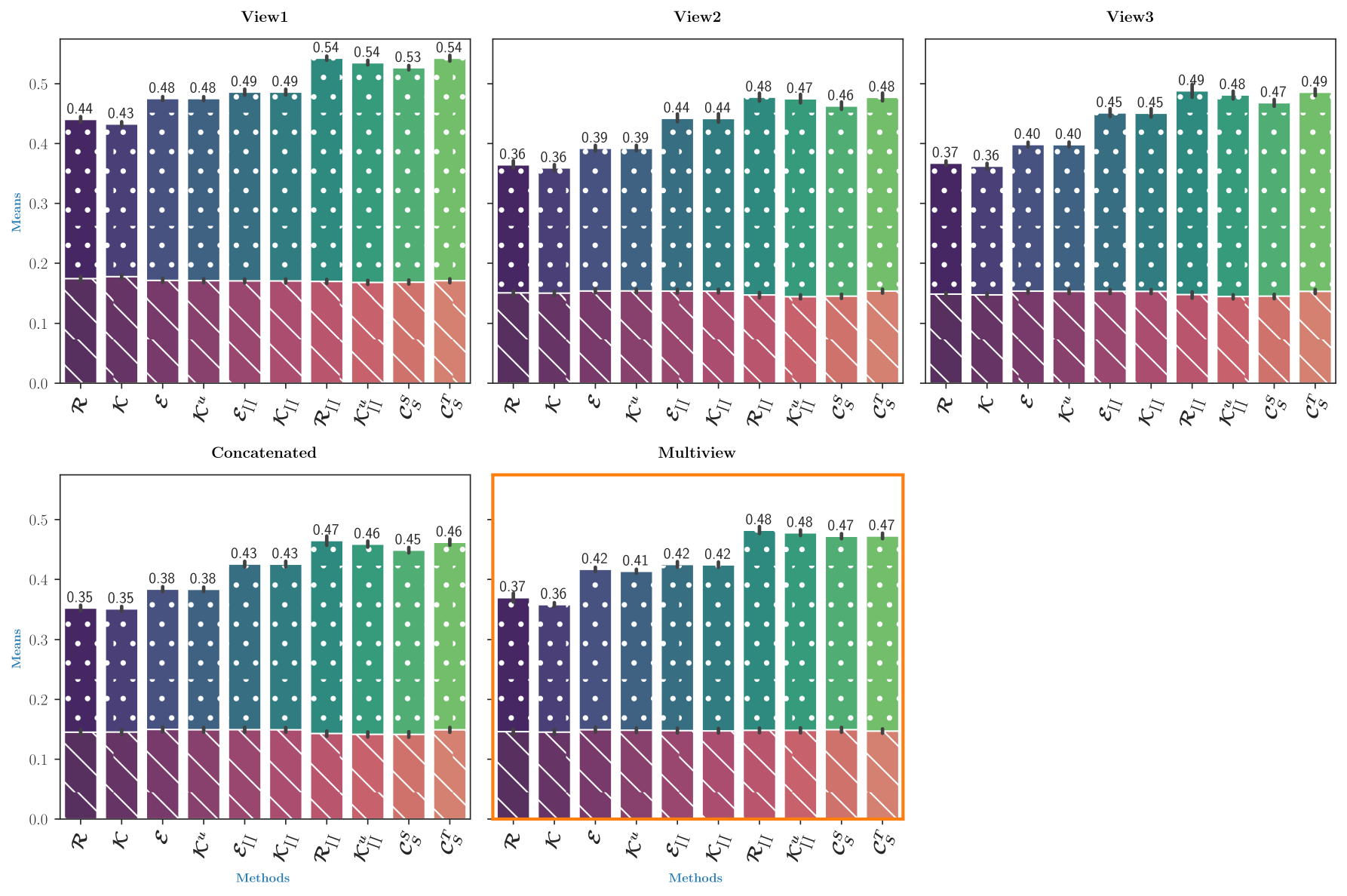}}\\

\caption{Test error rates and PAC-Bayesian bounds for binary classification on the PTB-XL+ dataset (Normal vs All), averaged over 10 runs. The experiment uses KL divergence or single-view and Rényi divergence ($\alpha=1.1$) for multi-view, with a stump configuration for \textbf{(a)}, weak, and strong learners for \textbf{(b)} and \textbf{(c)} resp. and 50\% labeled data. Multi-view results are highlighted in orange.}
\label{figure:ptb-xl-binary-normal-rest-all-configs}

\end{figure}

\begin{figure}
\centering
\includegraphics[width=0.98\linewidth]{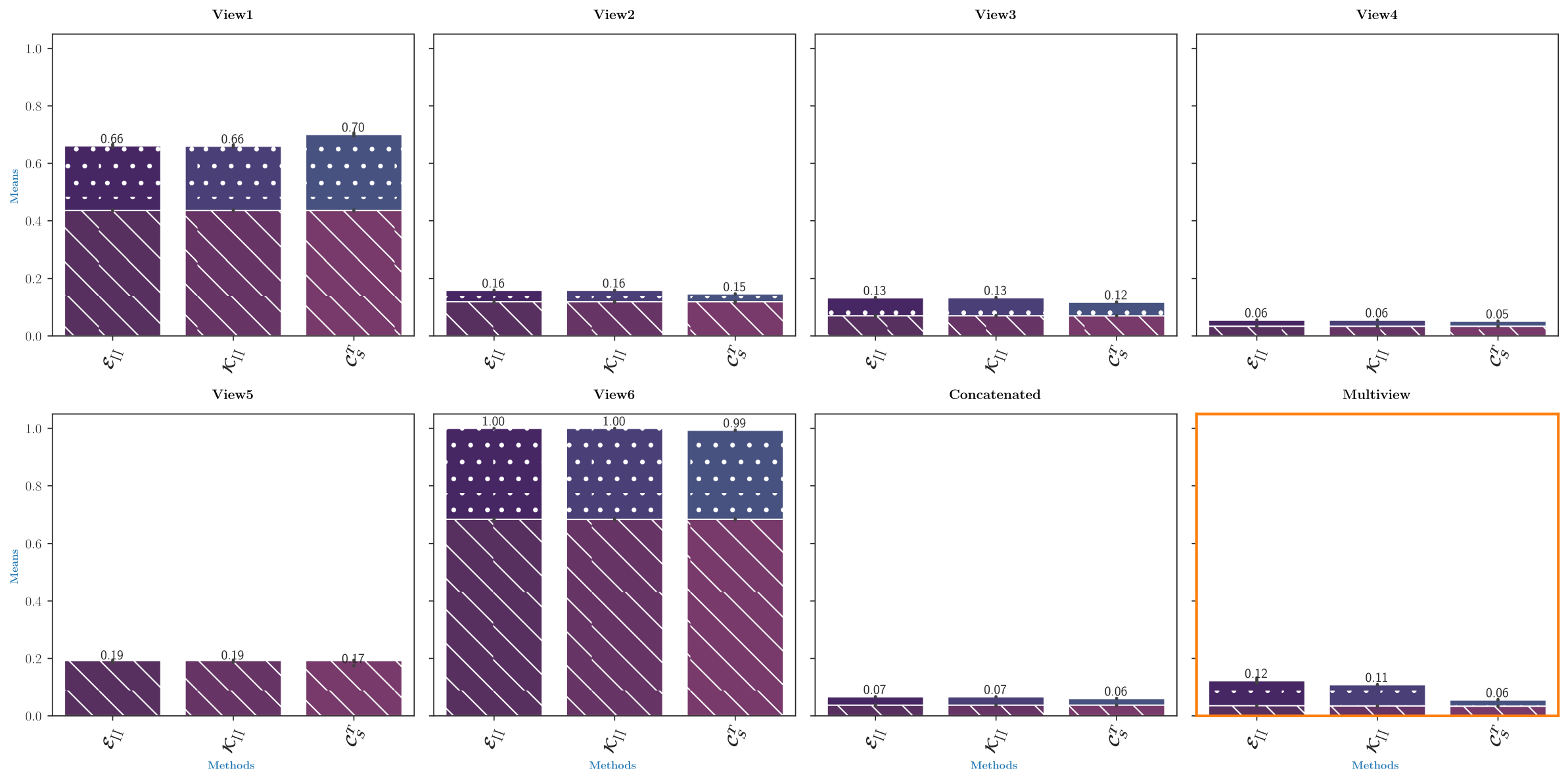}

\caption{Test error rates and PAC-Bayesian bounds for multiclass classification on the mfeat-large dataset, averaged over 10 runs. All views are shown.The experiment uses the same configuration as Figure~\ref{figure:mfeat-binary-4-9} with modifications to aid multi-class learning, strong learners with depth=20, and 100\% labeled data. Multi-view results are highlighted in orange.}
\label{figure:mfeat-mult-full-plot}
\end{figure}

\begin{figure}
\centering
\includegraphics[width=0.98\linewidth]{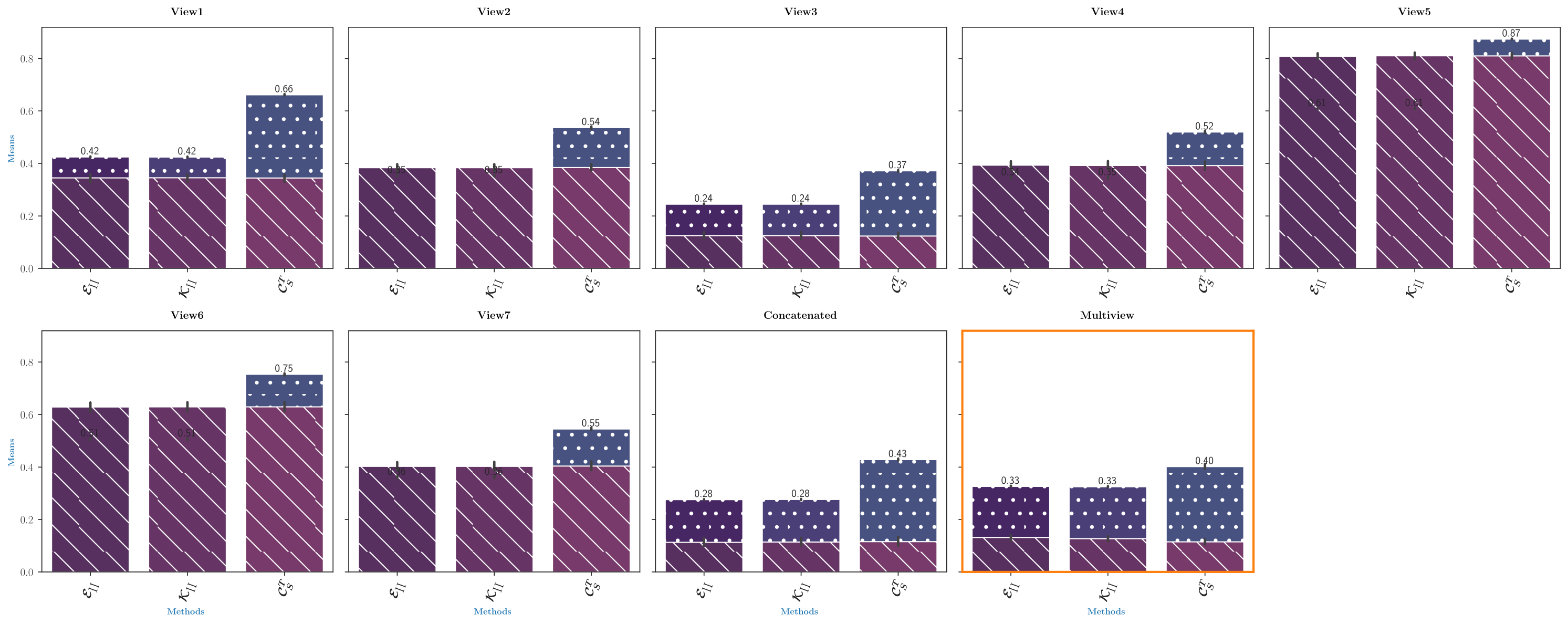}

\caption{Test error rates and PAC-Bayesian bounds for multiclass classification on the Corel-Image-Features dataset, averaged over 10 runs. All views are shown.The experiment uses the same configuration as Figure~\ref{figure:mfeat-mult-full-plot}.}
\label{figure:corel-mult-full-plot}
\end{figure}

\begin{figure}
\centering
\includegraphics[width=0.98\linewidth]{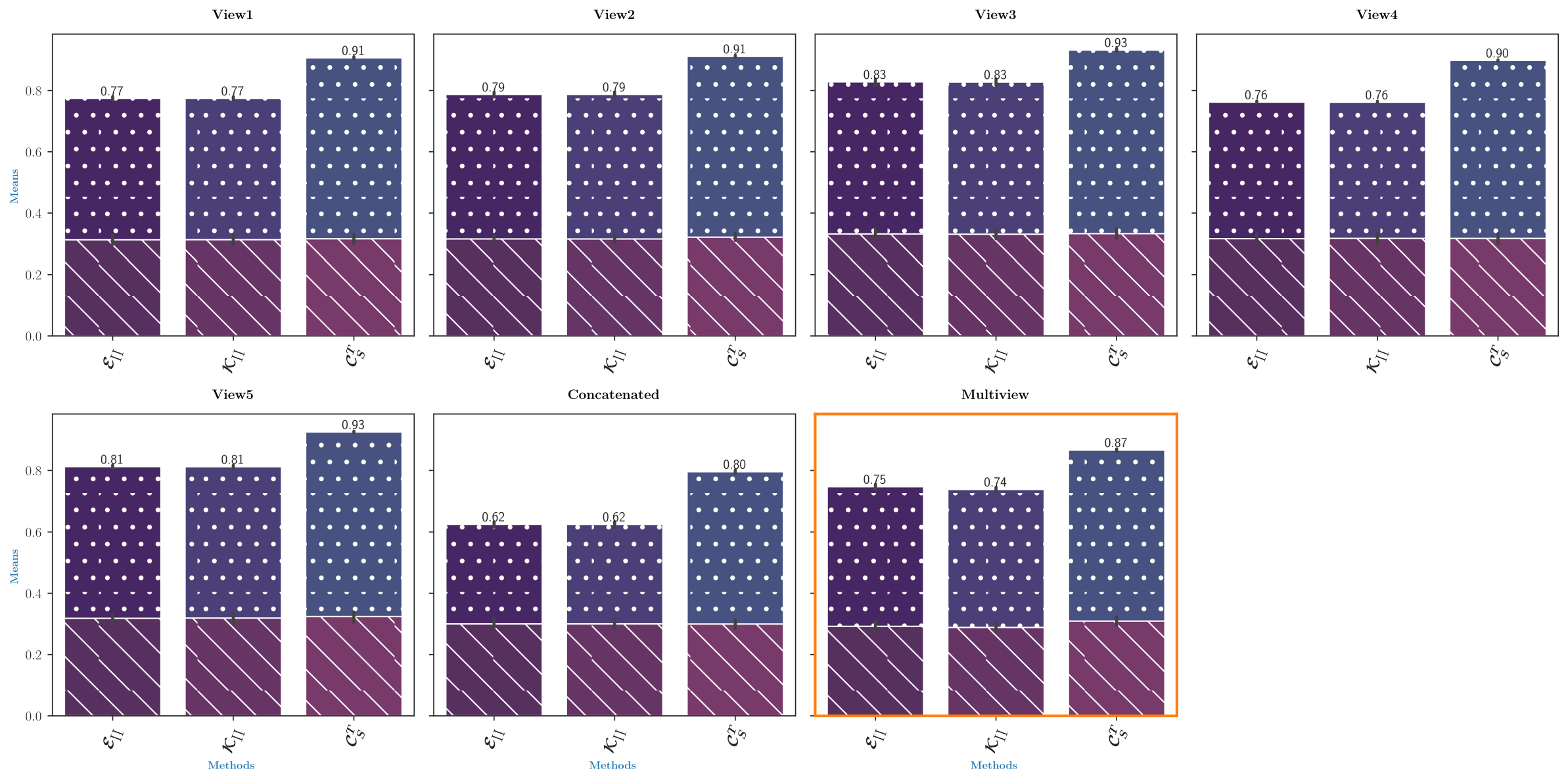}

\caption{Test error rates and PAC-Bayesian bounds for multiclass classification on the Reuters-EN dataset, averaged over 10 runs. All views are shown.The experiment uses the same configuration as Figure~\ref{figure:mfeat-mult-full-plot}.}
\label{figure:reuters-mult-full-plot}
\end{figure}

\begin{figure}
\centering
\includegraphics[width=0.7\linewidth]{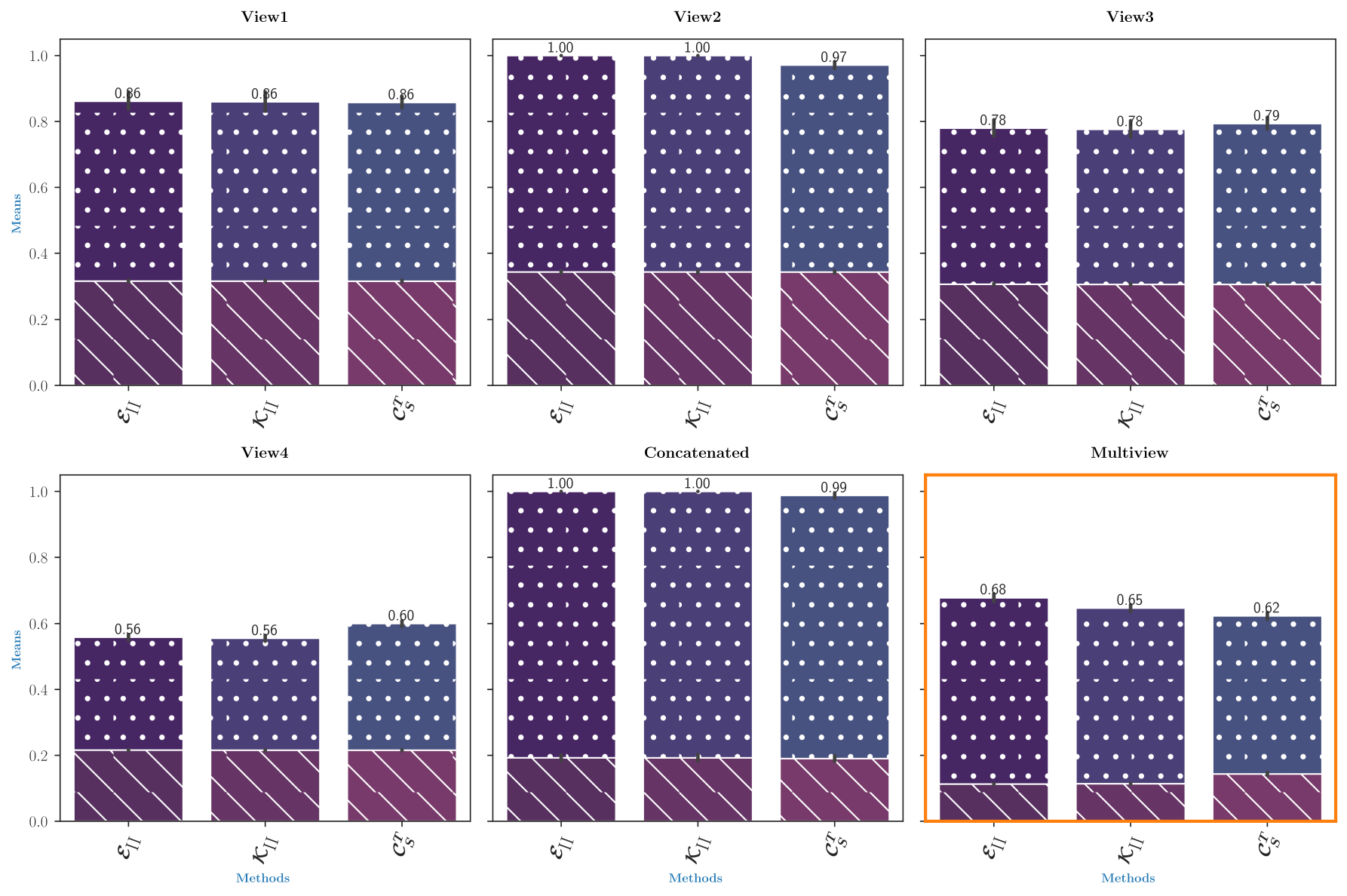}

\caption{Test error rates and PAC-Bayesian bounds for multiclass classification on the ALOI dataset, averaged over 10 runs. All views are shown.The experiment uses the same configuration as Figure~\ref{figure:mfeat-mult-full-plot}. with modifications of 50 estimators instead of 100, and 50\% labeled data to aid multi-class learning due to the large dataset size.}
\label{figure:aloi-mult-full-plot}
\end{figure}

\begin{table}[htbp]
\centering
\caption{Results for the dataset \textbf{mfeat-large (5vs6)}. Each column represents a different view of the dataset (or the concatenation or multi-view). The values for each bound method are shown in rows. The bold values indicate the triple (\textbf{Bnd}, $\bm{G}$, $\bm{\mathcal{B}}$) with the lowest total mean Bound, while the underlined values indicate the triple (\underline{Bnd}, $\underline{G}$, \underline{$\mathcal{B}$}) with the lowest total mean Risk.}
\begin{tabular}{|l|l||l||l|l|l|l|l|l||l|}
\hline
 \hline
\multirow{2}{*}{} & \multicolumn{8}{c||}{\textbf{mfeat-large 5vs6}} & \textbf{{Mean}} \\ 
& \multicolumn{1}{c}{\textbf{Con.}} & \multicolumn{1}{c}{\textbf{MV}} & \multicolumn{1}{c}{\textbf{View1}} & \multicolumn{1}{c}{\textbf{View2}} & \multicolumn{1}{c}{\textbf{View3}} & \multicolumn{1}{c}{\textbf{View4}} & \multicolumn{1}{c}{\textbf{View5}} & \multicolumn{1}{c||}{\textbf{View6}} & \\ 
\hline
 \hline
\multicolumn{1}{|c|}{\boldmath{$\mathcal{R}$\ref{Eq-Pac-bayes-kl-FO}}} \\ 
Bnd & .415 & .522 & .709 & .615 & .533 & .395 & .486 & .769 & .556 \\ 
$G$ & .185 & .237 & .329 & .282 & .242 & .175 & .219 & .358 & .254 \\ 
$\mathcal{B}$ & .067 & .072 & .198 & .156 & .133 & .091 & .182 & .342 & .155 \\ \cline{1-10}
\multicolumn{1}{|c|}{\boldmath{$\mathcal{K}$\ref{Eq-Pac-bayes-kl-inv-FO}}} \\ 
Bnd & .367 & .367 & .390 & .549 & .388 & .361 & .430 & .744 & \textbf{.450} \\ 
$G$ & .161 & .163 & .172 & .248 & .171 & .158 & .191 & .344 & \textbf{.201} \\ 
$\mathcal{B}$ & .064 & .107 & .154 & .228 & .152 & .101 & .185 & .343 & \textbf{.167} \\ \cline{1-10}
\multicolumn{1}{|c|}{\boldmath{$\mathcal{E}$\ref{Eq-Pac-bayes-joint-dis-mv-FO}}} \\ 
Bnd & .527 & .674 & .875 & .715 & .707 & .501 & .636 & .855 & .686 \\ 
$G$ & .238 & .288 & .405 & .328 & .324 & .225 & .288 & .394 & .311 \\ 
$\mathcal{B}$ & .078 & .079 & .309 & .130 & .134 & .087 & .197 & .339 & .169 \\ \cline{1-10}
\multicolumn{1}{|c|}{\boldmath{$\mathcal{K}^{u}$\ref{Eq-Pac-bayes-joint-dis-inv-mv-FO}}} \\ 
Bnd & .522 & .606 & .864 & .715 & .695 & .510 & .630 & .858 & .675 \\ 
$G$ & .235 & .232 & .400 & .328 & .318 & .229 & .286 & .396 & .303 \\ 
$\mathcal{B}$ & .072 & .065 & .273 & .127 & .120 & .083 & .186 & .338 & .158 \\ \cline{1-10}
\multicolumn{1}{|c|}{\boldmath{$\mathcal{E}_{\textnormal{II}}$\ref{Eq-Pac-bayes-joint-mv-SO}}} \\ 
Bnd & .345 & .462 & .700 & .582 & .514 & .352 & .582 & .946 & .560 \\ 
$G$ & .229 & .264 & .371 & .328 & .301 & .226 & .280 & .400 & .300 \\ 
$\mathcal{B}$ & .047 & .046 & .170 & .124 & .107 & .059 & .173 & .342 & .133 \\ \cline{1-10}
\multicolumn{1}{|c|}{\boldmath{$\mathcal{K}_{\textnormal{II}}$\ref{Eq-Pac-bayes-joint-inv-mv-SO}}} \\ 
Bnd & .350 & .448 & .717 & .581 & .525 & .358 & .592 & .946 & .565 \\ 
$G$ & .231 & .239 & .380 & .327 & .306 & .229 & .284 & .400 & .300 \\ 
$\mathcal{B}$ & .049 & .047 & .173 & .124 & .104 & .062 & .176 & .341 & .135 \\ \cline{1-10}
\multicolumn{1}{|c|}{\boldmath{$\mathcal{R}_{\textnormal{II}}$\ref{Eq-Pac-bayes-dis-mv-SO}}} \\ 
Bnd & .395 & .431 & .803 & .628 & .560 & .407 & .650 & 1.0 & .609 \\ 
$G$ & .196 & .255 & .371 & .302 & .267 & .192 & .233 & .358 & .272 \\ 
$\mathcal{B}$ & .072 & .056 & .223 & .150 & .124 & .085 & .182 & .342 & .154 \\ \cline{1-10}
\multicolumn{1}{|c|}{\boldmath{$\mathcal{K}^{u}_{\textnormal{II}}$\ref{Eq-Pac-bayes-dis-inv-mv-SO}}} \\ 
Bnd & .400 & .386 & .862 & .640 & .572 & .408 & .629 & 1.0 & .624 \\ 
$G$ & .213 & .239 & .391 & .317 & .289 & .212 & .245 & .379 & .286 \\ 
$\mathcal{B}$ & .057 & .056 & .273 & .136 & .112 & .069 & .180 & .338 & .153 \\ \cline{1-10}
\multicolumn{1}{|c|}{\boldmath{$\mathcal{C}_{S}^{S}$\ref{Eq-Pac-bayes-mv-C-Bound}}} \\ 
Bnd & .416 & .501 & .817 & .635 & .583 & .436 & .632 & .929 & .619 \\ 
$G$ & .212 & .289 & .386 & .301 & .285 & .217 & .250 & .348 & .286 \\ 
$\mathcal{B}$ & .052 & .052 & .178 & .136 & .113 & .065 & .175 & .344 & .139 \\ \cline{1-10}
\multicolumn{1}{|c|}{\boldmath{$\mathcal{C}_{S}^{T}$\ref{Eq-Pac-bayes-mv-C-tandem-Bound}}} \\ 
Bnd & .427 & .474 & .779 & .714 & .634 & .437 & .673 & .944 & .\underline{635} \\ 
$G$ & .231 & .204 & .342 & .324 & .298 & .228 & .274 & .394 & \underline{.287} \\ 
$\mathcal{B}$ & .049 & .049 & .154 & .127 & .104 & .062 & .173 & .338 & \underline{.132} \\ \cline{1-10}
\textbf{Mean} \\ \cline{1-9}
Bnd & \underline{\textbf{.416}} & .487 & .752 & .637 & .571 & \textbf{.416} & .594 & .908 \\ 
$G$ & \underline{\textbf{.213}} & .241 & .355 & .309 & .280 & \textbf{.209} & .255 & .377 \\ 
$\mathcal{B}$ & \underline{\textbf{.061}} & .063 & .210 & .144 & .120 & \textbf{.076} & .181 & .341 \\ \cline{1-9}
\end{tabular}
\label{tab:my_label}
\end{table}

\begin{table}[htbp]
\centering
\caption{Results for the dataset \textbf{mfeat-large (4vs9)}. Each column represents a different view of the dataset (or the concatenation or multi-view). The values for each bound method are shown in rows. The bold values indicate the triple (\textbf{Bnd}, $\bm{G}$, $\bm{\mathcal{B}}$) with the lowest total mean Bound, while the underlined values indicate the triple (\underline{Bnd}, $\underline{G}$, \underline{$\mathcal{B}$}) with the lowest total mean Risk.}
\begin{tabular}{|l|l||l||l|l|l|l|l|l||l|}
\hline
 \hline
\multirow{2}{*}{} & \multicolumn{8}{c||}{\textbf{mfeat-large 4vs9}} & \textbf{{Mean}} \\ 
& \multicolumn{1}{c}{\textbf{\bf{Con.}}} & \multicolumn{1}{c}{\textbf{\bf{MV}}} & \multicolumn{1}{c}{\textbf{\bf{View1}}} & \multicolumn{1}{c}{\textbf{\bf{View2}}} & \multicolumn{1}{c}{\textbf{\bf{View3}}} & \multicolumn{1}{c}{\textbf{\bf{View4}}} & \multicolumn{1}{c}{\textbf{\bf{View5}}} & \multicolumn{1}{c||}{\textbf{\bf{View6}}} & \\ 
\hline
 \hline
\multicolumn{1}{|c|}{\boldmath{$\mathcal{R}$\ref{Eq-Pac-bayes-kl-FO}}} \\ 
Bnd & .513 & .574 & .857 & .576 & .695 & .541 & .521 & .938 & .652 \\ 
$G$ & .233 & .263 & .402 & .264 & .322 & .246 & .236 & .442 & .301 \\ 
$\mathcal{B}$ & .149 & .132 & .315 & .165 & .257 & .187 & .213 & .434 & .231 \\ \cline{1-10}
\multicolumn{1}{|c|}{\boldmath{$\mathcal{K}$\ref{Eq-Pac-bayes-kl-inv-FO}}} \\ 
Bnd & .432 & .432 & .820 & .410 & .633 & .466 & .484 & .926 & \textbf{.575} \\ 
$G$ & .192 & .194 & .381 & .182 & .289 & .209 & .218 & .435 & \textbf{.263} \\ 
$\mathcal{B}$ & .151 & .154 & .366 & .171 & .279 & .175 & .208 & .435 & \textbf{.242} \\ \cline{1-10}
\multicolumn{1}{|c|}{\boldmath{$\mathcal{E}$\ref{Eq-Pac-bayes-joint-dis-mv-FO}}} \\ 
Bnd & .694 & .770 & .930 & .764 & .814 & .733 & .677 & .992 & .797 \\ 
$G$ & .318 & .335 & .432 & .352 & .376 & .337 & .309 & .462 & .365 \\ 
$\mathcal{B}$ & .167 & .132 & .354 & .169 & .215 & .240 & .218 & .434 & .241 \\ \cline{1-10}
\multicolumn{1}{|c|}{\boldmath{$\mathcal{K}^{u}$\ref{Eq-Pac-bayes-joint-dis-inv-mv-FO}}} \\ 
Bnd & .687 & .717 & .930 & .757 & .811 & .729 & .669 & .996 & .787 \\ 
$G$ & .315 & .293 & .433 & .349 & .375 & .335 & .305 & .466 & .359 \\ 
$\mathcal{B}$ & .149 & .119 & .355 & .163 & .207 & .221 & .209 & .436 & .232 \\ \cline{1-10}
\multicolumn{1}{|c|}{\boldmath{$\mathcal{E}_{\textnormal{II}}$\ref{Eq-Pac-bayes-joint-mv-SO}}} \\ 
Bnd & .525 & .612 & .874 & .625 & .721 & .616 & .603 & 1.0 & .697 \\ 
$G$ & .306 & .333 & .429 & .340 & .373 & .334 & .302 & .466 & .360 \\ 
$\mathcal{B}$ & .117 & .083 & .283 & .155 & .199 & .168 & .186 & .435 & .203 \\ \cline{1-10}
\multicolumn{1}{|c|}{\boldmath{$\mathcal{K}_{\textnormal{II}}$\ref{Eq-Pac-bayes-joint-inv-mv-SO}}} \\ 
Bnd & .543 & .600 & .869 & .630 & .720 & .620 & .613 & 1.0 & .699 \\ 
$G$ & .313 & .315 & .428 & .343 & .373 & .335 & .305 & .466 & .360 \\ 
$\mathcal{B}$ & .123 & .091 & .273 & .156 & .199 & .172 & .192 & .434 & .205 \\ \cline{1-10}
\multicolumn{1}{|c|}{\boldmath{$\mathcal{R}_{\textnormal{II}}$\ref{Eq-Pac-bayes-dis-mv-SO}}} \\ 
Bnd & .583 & .549 & .940 & .661 & .781 & .691 & .672 & 1.0 & .735 \\ 
$G$ & .280 & .298 & .421 & .304 & .363 & .302 & .263 & .444 & .334 \\ 
$\mathcal{B}$ & .144 & .113 & .319 & .162 & .217 & .197 & .206 & .434 & .224 \\ \cline{1-10}
\multicolumn{1}{|c|}{\boldmath{$\mathcal{K}^{u}_{\textnormal{II}}$\ref{Eq-Pac-bayes-dis-inv-mv-SO}}} \\ 
Bnd & .599 & .535 & .989 & .712 & .785 & .700 & .668 & 1.0 & .749 \\ 
$G$ & .295 & .297 & .435 & .346 & .369 & .316 & .279 & .461 & .350 \\ 
$\mathcal{B}$ & .143 & .116 & .394 & .161 & .214 & .200 & .208 & .434 & .234 \\ \cline{1-10}
\multicolumn{1}{|c|}{\boldmath{$\mathcal{C}_{S}^{S}$\ref{Eq-Pac-bayes-mv-C-Bound}}} \\ 
Bnd & .579 & .677 & .919 & .680 & .781 & .663 & .655 & .995 & \underline{.744} \\ 
$G$ & .295 & .354 & .420 & .323 & .355 & .322 & .272 & .444 & \underline{.348} \\ 
$\mathcal{B}$ & .107 & .096 & .260 & .152 & .219 & .148 & .200 & .434 & \underline{.202} \\ \cline{1-10}
\multicolumn{1}{|c|}{\boldmath{$\mathcal{C}_{S}^{T}$\ref{Eq-Pac-bayes-mv-C-tandem-Bound}}} \\ 
Bnd & .662 & .641 & .960 & .742 & .846 & .740 & .710 & .999 & .787 \\ 
$G$ & .308 & .257 & .424 & .329 & .367 & .327 & .297 & .466 & .347 \\ 
$\mathcal{B}$ & .125 & .114 & .259 & .156 & .200 & .172 & .190 & .435 & .206 \\ \cline{1-10}
\textbf{Mean} \\ \cline{1-9}
Bnd & \textbf{.582} & \underline{.611} & .909 & .656 & .758 & .650 & .627 & .985 \\ 
$G$ & \textbf{.285} & \underline{.294} & .421 & .313 & .356 & .306 & .279 & .455 \\ 
$\mathcal{B}$ & \textbf{.138} & \underline{.115} & .318 & .161 & .221 & .188 & .203 & .435 \\ \cline{1-9}
\end{tabular}
\label{tab:mfeat-4-}
\end{table}

\begin{figure}
\centering
\subfloat[All methods]{\includegraphics[width=.8\linewidth]{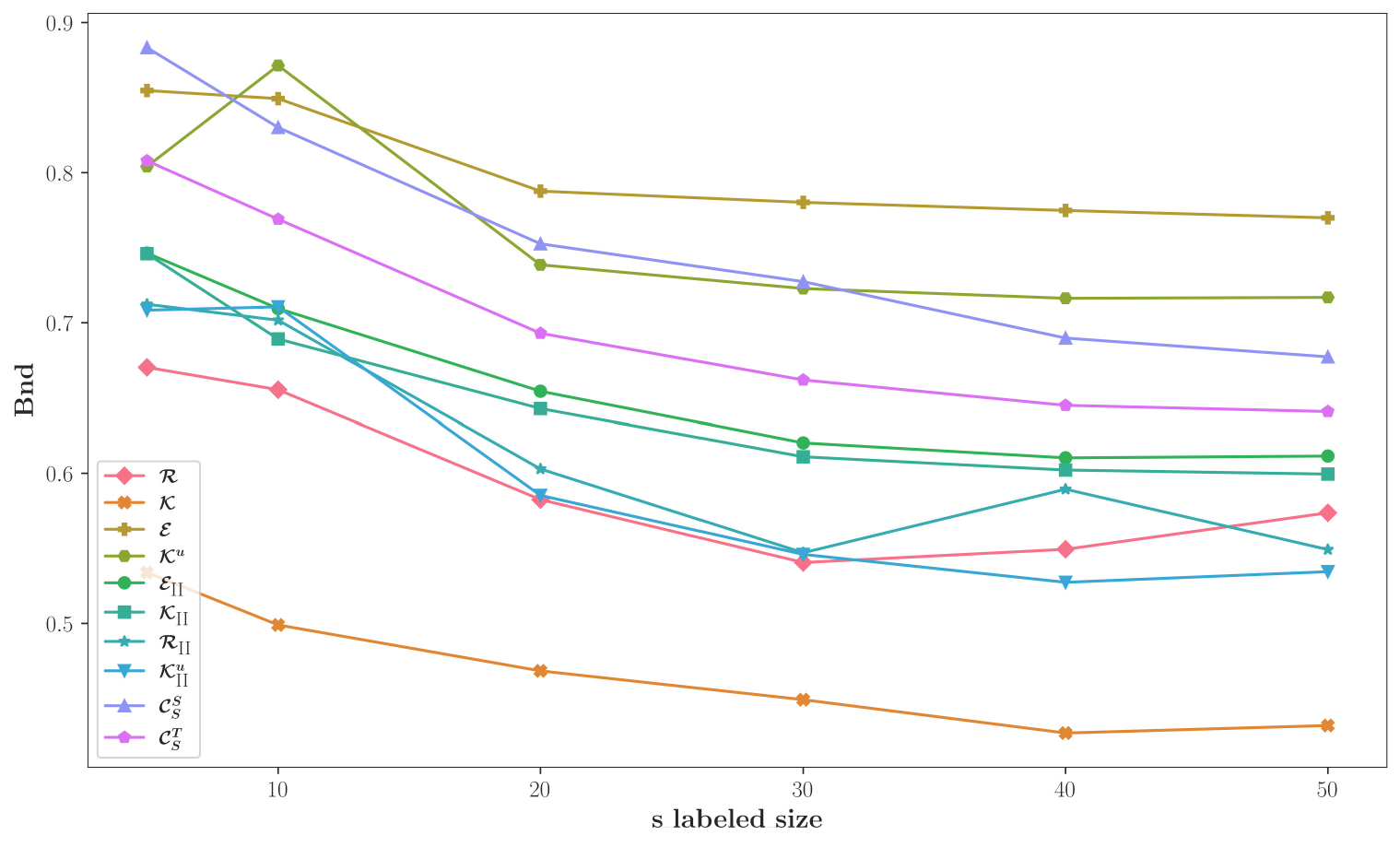}}\\

\subfloat[$\mathcal{K}^{u}_{\textnormal{II}}$\ref{Eq-Pac-bayes-dis-inv-mv-SO} vs $\mathcal{K}_{\textnormal{II}}$\ref{Eq-Pac-bayes-joint-inv-mv-SO}]{
  \includegraphics[width=0.8\linewidth]{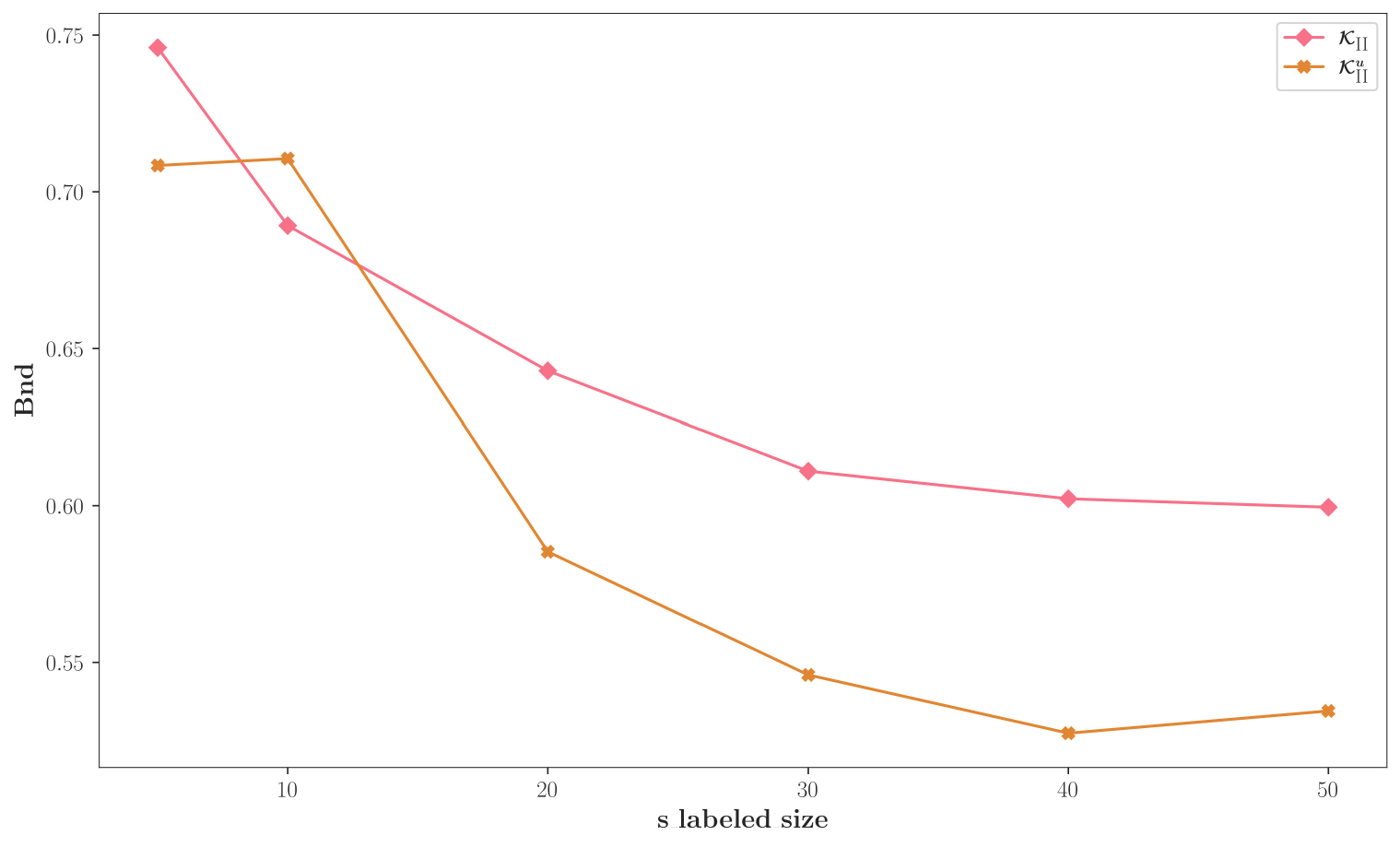}}

\caption{Comparison of bound values on the mfeat-large dataset (4vs9) as a function of $s\_labeled\_size$. \textbf{(a)} illustrates how changes in the proportion of labeled data ($s\_labeled\_size$) with a fixed $\alpha$ (1.1 in this case) affect the bound value. \textbf{(b)} shows that with access to a significant amount of unlabeled data, the bound $\mathcal{K}^{u}_{\textnormal{II}}$ (Equation~\ref{Eq-Pac-bayes-dis-inv-mv-SO}), which employs the disagreement term, is tighter than $\mathcal{K}_{\textnormal{II}}$ (Equation~\ref{Eq-Pac-bayes-joint-inv-mv-SO}).}

\label{figure:vary-labeled}
\end{figure}

\begin{figure}
\centering
\includegraphics[width=.8\linewidth]{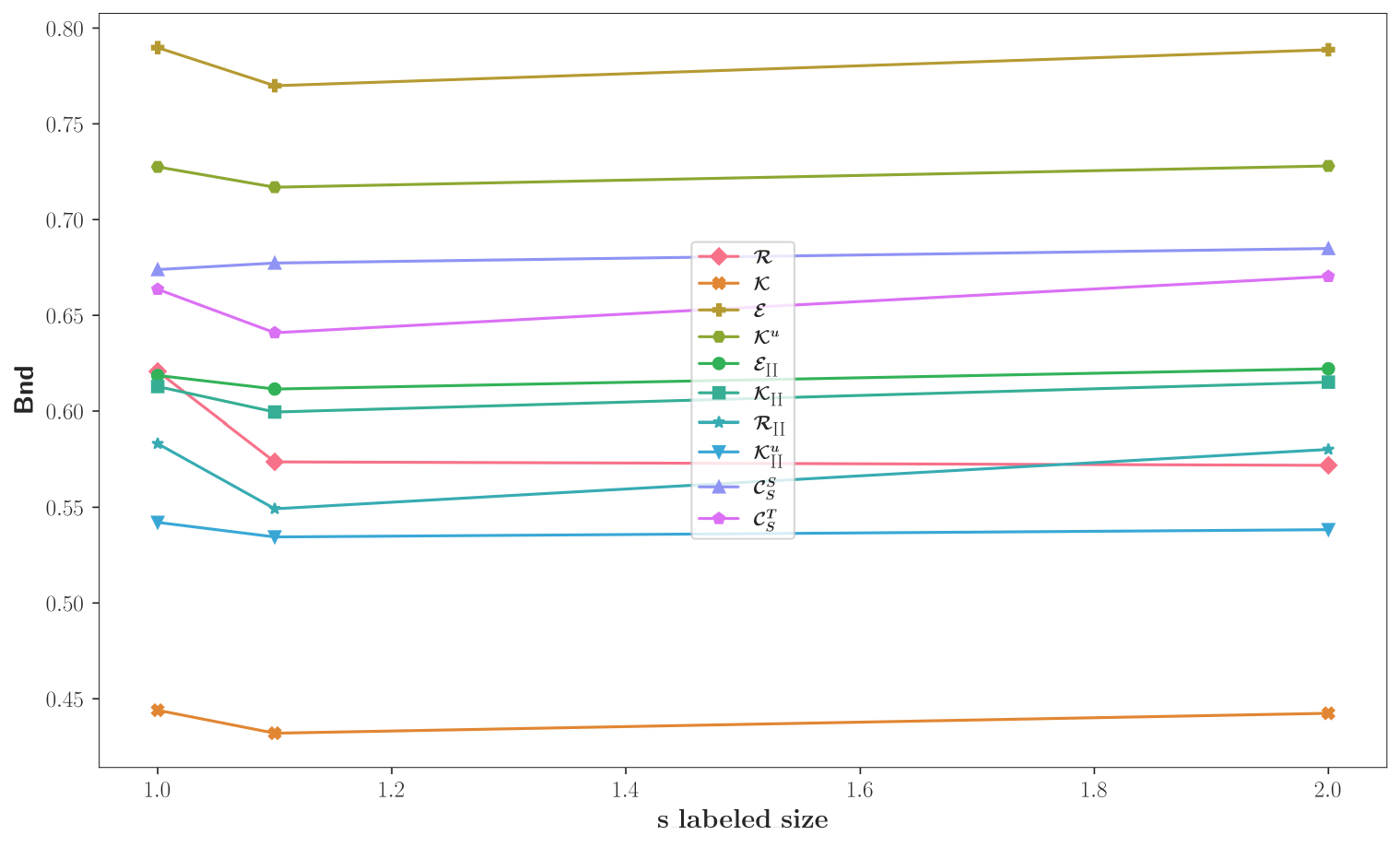}

\caption{Comparison of bound values on the mfeat-large (4vs9) dataset as a function of $\alpha$. The figure illustrates how changes in the Rényi divergence order $\alpha$ affect the bound values with a fixed $s\_labeled\_size = 0.5$. Overall, the bounds become tighter when $\alpha$ is arround 1.1.}

\label{figure:vary-alpha}
\end{figure}


\begin{figure}
\centering
\includegraphics[width=.8\linewidth]{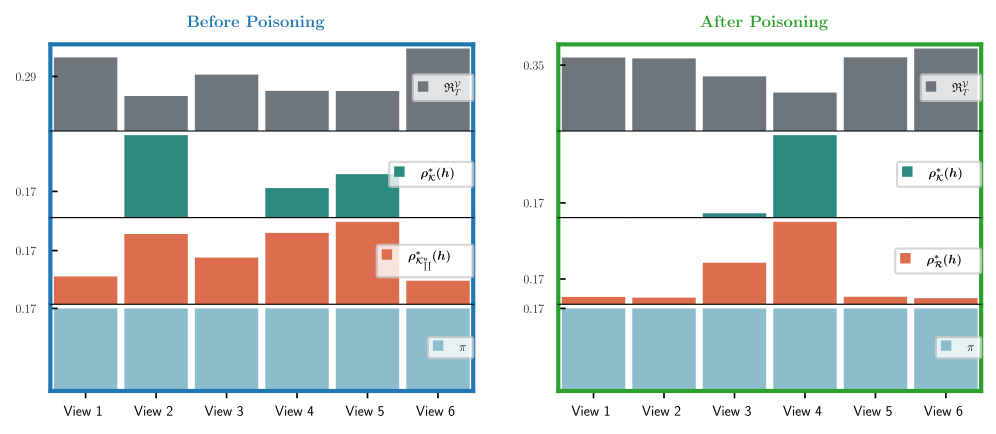}

\caption{Comparison of optimized hyper-posterior distributions ($\rho^*$) from the top two performing algorithms (based on the bound value over views on mfeat-large dataset 4vs9), before and after poisoning the most effective views (2 and 5) by adding Gaussian noise. \textbf{Left:} Posterior distribution before data poisoning. \textbf{Right:} Posterior distribution after data poisoning. The shift in the posterior distribution after data poisoning indicates a significant change in the model's confidence levels across different views. Parameters: $\alpha=1.1$, stump, $s\_labeled\_size=50\%$. (a detailed version of this figure can be found in Figure~\ref{figure:poisoning-full}}

\label{figure:poisoning}
\end{figure}

\begin{figure}
\centering
\subfloat[Before Poisoning]{
  \includegraphics[width=0.85\linewidth]{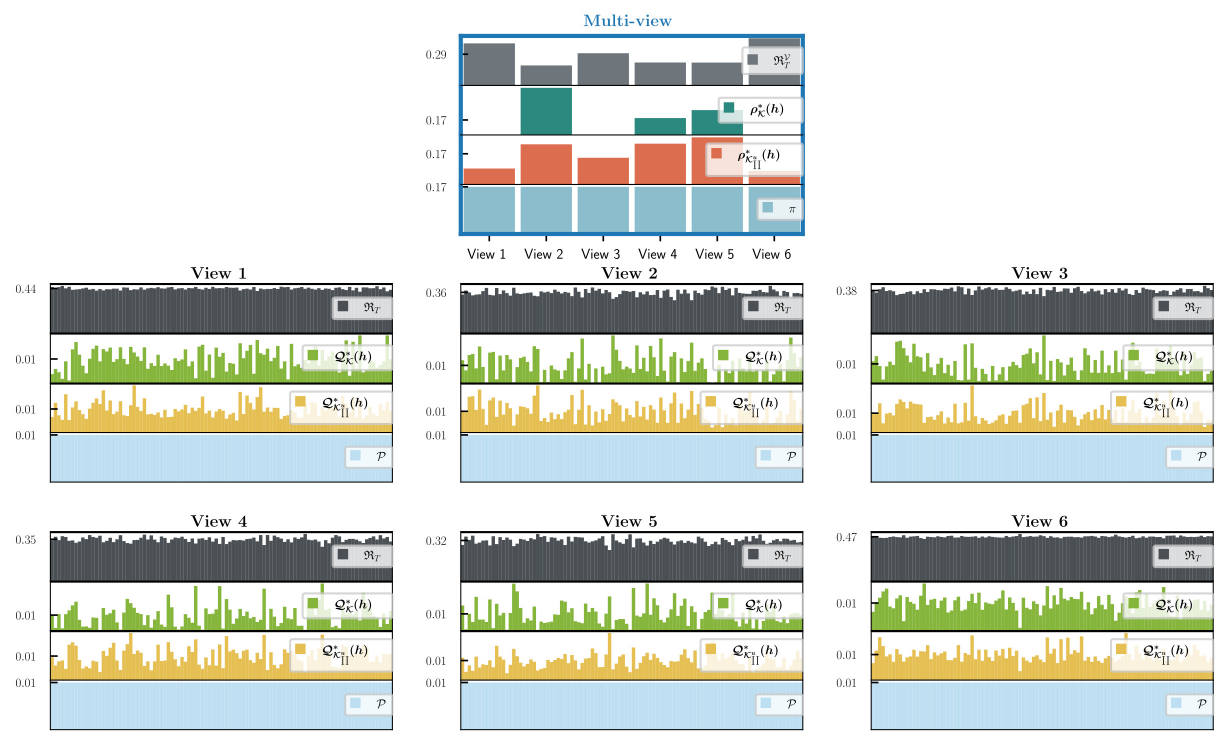}} \vspace{-1pt} \\
\subfloat[After Poisoning]{\includegraphics[width=0.85\linewidth]{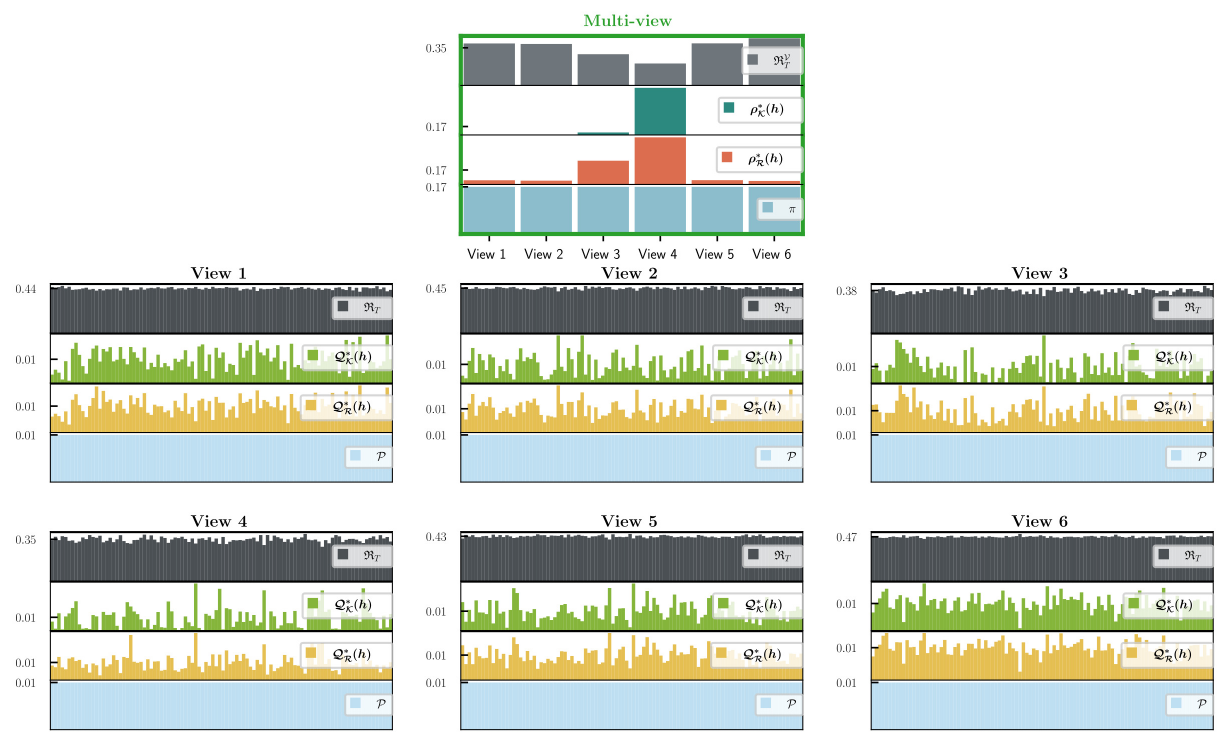}}

\caption{Comparison of optimized posteriors ($\mathcal{Q}_v^*$) and hyper-posterior ($\rho^*$) distributions from the top two performing algorithms (based on the bound value over views on the mfeat-large dataset 4vs9), before and after poisoning the most effective views (2 and 5) by adding Gaussian noise. The task is binary classification on the mfeat 4 vs 9 dataset. \textbf{(a)} Posterior distributions before data poisoning. \textbf{(b)} Posterior distributions after data poisoning. Parameters: $\alpha=1.1$, stump, $s\_labeled\_size=50\%$.}
\label{figure:poisoning-full}
\end{figure}

\end{document}